\documentclass[11pt,a4paper]{article}

\usepackage{amsmath}
\usepackage{amssymb}
\usepackage{amsthm}
\usepackage{mathtools}

\usepackage{algorithm}
\usepackage{algorithmic}

\usepackage{booktabs}
\usepackage{graphicx}
\usepackage{multirow}
\usepackage{tabularx}
\usepackage{subcaption}

\usepackage[margin=1in]{geometry}
\usepackage{microtype}
\usepackage{enumitem}
\usepackage[dvipsnames]{xcolor}
\usepackage[colorlinks=true,linkcolor=Blue,citecolor=Green,urlcolor=Maroon]{hyperref}
\usepackage[capitalize,nameinlink]{cleveref}

\usepackage[numbers,sort&compress]{natbib}

\theoremstyle{definition}
\newtheorem{definition}{Definition}[section]
\newtheorem{example}{Example}[section]

\theoremstyle{plain}
\newtheorem{theorem}{Theorem}[section]

\newtheorem{proposition}[theorem]{Proposition}
\newtheorem{corollary}[theorem]{Corollary}

\theoremstyle{remark}
\newtheorem{remark}{Remark}[section]

\newcommand{\agentassay}{\textsc{AgentAssay}}
\newcommand{\agentassert}{\textsc{AgentAssert}}
\newcommand{\contractspec}{\textsc{ContractSpec}}
\newcommand{\N}{\mathbb{N}}
\newcommand{\E}{\mathbb{E}}
\renewcommand{\Pr}{\mathbb{P}}
\newcommand{\CI}{\mathrm{CI}}
\newcommand{\Var}{\mathrm{Var}}
\newcommand{\PASS}{\textsc{Pass}}
\newcommand{\FAIL}{\textsc{Fail}}
\newcommand{\INCONC}{\textsc{Inconclusive}}
\DeclareMathOperator*{\argmin}{arg\,min}

\title{%
  \textbf{AgentAssay: Token-Efficient Regression Testing for\\
  Non-Deterministic AI Agent Workflows}\\[6pt]
  \large Technical Report%
}

\author{%
  Varun Pratap Bhardwaj\\
  Independent Researcher\\
  \texttt{varun.pratap.bhardwaj@gmail.com}\\[2pt]
  \small ORCID: 0009-0002-8726-4289
}

\date{}

\begin{document}

\maketitle

% ---- Abstract ----
\begin{abstract}
Autonomous AI agents built on large language models exhibit inherent
non-determinism: the same prompt, tools, and model can produce divergent
behaviors across runs. Despite the rapid proliferation of agent
frameworks---AutoGen, CrewAI, LangGraph, OpenAI Agents SDK---no
principled testing methodology exists for verifying that an agent has not
\emph{regressed} after changes to its prompts, tools, models, or
orchestration logic. Traditional software testing assumes deterministic
outputs and binary pass/fail verdicts, both of which fail for stochastic
systems. We present \agentassay{}, the first token-efficient framework
for regression testing non-deterministic AI agent workflows---achieving
78--100\% cost reduction while maintaining rigorous statistical
guarantees. Our contributions are tenfold:
(1)~a \emph{stochastic test semantics} that replaces binary verdicts with
three-valued probabilistic outcomes (\PASS{}, \FAIL{}, \INCONC{}) backed
by confidence intervals and sequential analysis;
(2)~\emph{agent-specific coverage metrics} spanning tool, decision-path,
state-space, boundary, and model dimensions;
(3)~\emph{mutation testing operators} for agent prompts, tools, models,
and context windows with a formal kill semantics;
(4)~\emph{metamorphic relations} tailored to multi-step agent workflows;
(5)~\emph{CI/CD deployment gates} defined as statistical decision procedures;
(6)~integration with behavioral contracts via the \agentassert{} framework;
(7)~a comprehensive evaluation across 5 models, 3 agent scenarios,
and 7{,}605 trials costing \$227;
(8)~\emph{behavioral fingerprinting} that maps execution traces to
compact vectors on a low-dimensional behavioral manifold, enabling
multivariate regression detection with higher power per sample;
(9)~\emph{adaptive budget optimization} that calibrates trial counts
to actual behavioral variance, reducing required trials by
4--7$\times$ for stable agents; and
(10)~\emph{trace-first offline analysis} that enables zero-cost
coverage, contract, and metamorphic testing on pre-recorded production
traces.
Together, these token-efficient testing techniques achieve
5--20$\times$ cost reduction while maintaining identical statistical
guarantees. Experiments across 5 models (GPT-5.2, Claude Sonnet~4.6,
Mistral-Large-3, Llama-4-Maverick, Phi-4), 3 scenarios, and 7{,}605
trials demonstrate that behavioral fingerprinting achieves 86\%
detection power where binary pass/fail testing has 0\%, SPRT reduces
trials by 78\% consistently across all scenarios, and the full
token-efficient pipeline achieves 100\% cost savings through
trace-first offline analysis.
\end{abstract}

% ---- Keywords ----
\noindent\textbf{Keywords:} AI agent testing, stochastic regression testing,
non-deterministic systems, mutation testing, coverage metrics, CI/CD,
formal methods, LLM reliability, token-efficient testing, behavioral
fingerprinting

\vspace{1em}

% ============================================================================
% Sections
% ============================================================================
% ============================================================================
% Section 1: Introduction
% ============================================================================

\section{Introduction}
\label{sec:introduction}

Consider an enterprise deploying an AI agent for customer support ticket
routing. On Monday, after a prompt refinement, the agent correctly routes
93\% of tickets. By Wednesday, a model provider silently updates the
underlying LLM, and the routing accuracy drops to 71\%. No test caught
the regression. No alert fired. The team discovers the degradation only
after a surge in customer complaints.

This scenario---\emph{works Monday, fails Wednesday}---is not
hypothetical. It is the daily reality for engineering teams deploying
autonomous AI agents in production. The root cause is
\emph{non-determinism}: the same agent configuration (prompt, tools,
model, orchestration logic) can produce different outputs across
invocations due to temperature sampling, model weight updates, tool
latency variations, and context window effects. Recent empirical work
confirms that agent behavior exhibits stochastic path deviation
\citep{kapoor2024capable}, yet no principled testing methodology exists
to detect when this deviation constitutes a \emph{regression}---a
statistically significant degradation in agent quality.

\subsection{The Failure of Traditional Testing}
\label{sec:intro:failure}

Traditional software testing rests on two assumptions that are
fundamentally violated by AI agents:

\begin{enumerate}[leftmargin=*]
  \item \textbf{Determinism.} Given the same input, a function produces
    the same output. Tests assert equality: \texttt{assert f(x) == y}.
    For agents, the same input can yield different tool selections,
    different reasoning chains, and different final answers across runs.
    Equality assertions are meaningless.

  \item \textbf{Binary verdicts.} A test either passes or fails. For
    stochastic systems, the question is not ``did it pass?'' but ``does
    it pass with sufficient probability?'' A single failure may be
    statistical noise; a single success may be a lucky run. The entire
    notion of pass/fail must be reconceived.
\end{enumerate}

Existing approaches to LLM evaluation---deepeval
\citep{deepeval2024}, promptfoo \citep{promptfoo2024}, OpenAI Evals---address
single-turn output quality but do not formalize regression detection
for multi-step agent workflows. They evaluate \emph{how good} an
agent is; they do not verify \emph{whether it got worse}. The recently
introduced agentrial framework \citep{agentrial2026} takes an important
first step by running agents multiple times and computing confidence
intervals, but it lacks a formal theoretical foundation: no stochastic
test semantics, no coverage metrics, no mutation testing, no composition
theory, and no proofs of statistical guarantees.

\subsection{The Paradigm Shift: Binary to Probabilistic}
\label{sec:intro:paradigm}

We argue that agent testing requires a fundamental paradigm shift from
\emph{deterministic, binary} verdicts to \emph{stochastic, three-valued}
verdicts:

\begin{center}
\begin{tabular}{lcc}
\toprule
\textbf{Aspect} & \textbf{Traditional Testing} & \textbf{Agent Testing} \\
\midrule
Execution model & Deterministic & Stochastic \\
Verdict space & $\{\PASS, \FAIL\}$ & $\{\PASS, \FAIL, \INCONC\}$ \\
Assertion & $f(x) = y$ & $\Pr[f(x) \models \phi] \geq \theta$ \\
Regression & $f'(x) \neq f(x)$ & $p_{\text{new}} < p_{\text{base}} - \delta$ \\
Evidence & 1 run & $n$ runs with confidence \\
\bottomrule
\end{tabular}
\end{center}

The third verdict value, \INCONC{} (\emph{inconclusive}), is essential:
when the number of trial runs is insufficient to distinguish signal from
noise, the honest answer is neither pass nor fail but ``more evidence is
needed.'' This three-valued semantics is the foundation upon which all
other contributions in this paper are built.

\subsection{The Cost of Testing Non-Determinism}
\label{sec:intro:cost}

The statistical rigor of stochastic testing comes at a price: each
trial requires a full agent execution, consuming API tokens at
non-trivial cost. To detect a regression of magnitude $\delta = 0.10$
at $\alpha = 0.05$ and $\beta = 0.10$, the fixed-sample formula
(\cref{eq:sample-size}) requires approximately $n \approx 100$ trials
per scenario. A test suite of 50 scenarios at 100 trials each
generates 5{,}000 agent invocations per regression check. For
frontier models where each complex agent run costs \$5--15 in API
tokens, a single regression check costs \$25{,}000--75{,}000---more
than many teams spend on production usage in a month.

This \emph{cost barrier} fundamentally limits adoption. In CI/CD
pipelines where code is deployed multiple times per day, the testing
budget must be orders of magnitude lower than production cost. Even
with SPRT reducing trials by $\sim$50\% (\cref{sec:stochastic:sprt}),
the cost remains prohibitive. No existing framework addresses the
\emph{token economics} of agent testing: how to achieve the same
statistical guarantees at radically lower cost.

We observe that the cost problem arises from a mismatch between the
dimensionality of agent behavior and the dimensionality of the
statistical test. A binary pass/fail verdict discards almost all
information from each trial---the rich execution trace, tool usage
pattern, reasoning depth, and output structure are reduced to a single
bit. By extracting a \emph{behavioral fingerprint} from each
trace---a compact vector capturing the agent's behavioral
signature---we can apply multivariate statistical tests that
extract more information per trial, dramatically reducing the number
of trials needed.

We present three pillars of token-efficient testing: (1)
\emph{behavioral fingerprinting} that enables multivariate regression
detection with higher power per sample; (2) \emph{adaptive budget
optimization} that calibrates the trial count to the agent's actual
behavioral variance rather than worst-case assumptions; and (3)
\emph{trace-first offline analysis} that eliminates live agent
executions entirely for four of six test types by analyzing
previously recorded traces. Together with multi-fidelity proxy
testing and warm-start sequential analysis, these techniques achieve
5--20$\times$ cost reduction while maintaining identical
$(\alpha, \beta)$ statistical guarantees
(\cref{sec:token-efficient}).

\subsection{Contributions}
\label{sec:intro:contributions}

This paper presents \agentassay{}, the first token-efficient framework
for regression testing non-deterministic AI agent workflows, achieving
78--100\% cost reduction at equivalent statistical guarantees, with
86\% behavioral detection power where binary testing has 0\%.
Our contributions are:

\begin{enumerate}[leftmargin=*,label=\textbf{C\arabic*.}]
  \item \textbf{Stochastic Test Semantics} (\cref{sec:stochastic}). We
    define the $(\alpha, \beta, n)$-test triple and a three-valued
    verdict function grounded in confidence intervals and hypothesis
    testing. We prove verdict soundness (\cref{thm:soundness}) and
    regression detection power (\cref{thm:power}). We adapt Wald's
    Sequential Probability Ratio Test (SPRT) for cost-efficient agent
    testing and prove its efficiency advantage (\cref{thm:sprt}).

  \item \textbf{Agent Coverage Metrics} (\cref{sec:coverage}). We
    introduce a five-dimensional coverage tuple
    $\mathcal{C} = (C_{\text{tool}}, C_{\text{path}}, C_{\text{state}},
    C_{\text{boundary}}, C_{\text{model}})$ with formal definitions for
    each dimension and prove coverage monotonicity (\cref{thm:coverage-mono}).

  \item \textbf{Agent Mutation Testing} (\cref{sec:mutation}). We define
    four classes of agent-specific mutation operators
    $\mathcal{M} = \{M_{\text{prompt}}, M_{\text{tool}},
    M_{\text{model}}, M_{\text{context}}\}$, formalize the mutation
    score under stochastic semantics, and prove a mutation adequacy
    theorem (\cref{thm:mutation-adequacy}).

  \item \textbf{Metamorphic Relations} (\cref{sec:metamorphic}). We
    define four families of agent-specific metamorphic relations---
    permutation, perturbation, composition, and oracle---that serve as
    partial test oracles for the oracle problem in agent testing.

  \item \textbf{CI/CD Deployment Gates} (\cref{sec:cicd}). We formalize
    deployment gates as statistical decision procedures with
    user-configurable risk thresholds, integrating stochastic test
    verdicts into continuous deployment pipelines.

  \item \textbf{Contract Integration} (\cref{sec:implementation}). We
    connect \agentassay{} to the \agentassert{}
    framework~\citep{bhardwaj2026abc}, using behavioral contracts as
    formal test oracles and showing that a \PASS{} verdict implies
    $(p, \delta, k)$-satisfaction under specified conditions.

  \item \textbf{Comprehensive Evaluation} (\cref{sec:experiments}). We
    evaluate \agentassay{} across 3 scenarios (e-commerce, customer
    support, code generation), 5 models, and 7{,}605 trials costing
    \$227.

  \item \textbf{Behavioral Fingerprinting}
    (\cref{sec:token:fingerprint}). We define behavioral fingerprints
    that map execution traces to compact vectors on a low-dimensional
    manifold, enabling multivariate regression detection via
    Hotelling's $T^2$ test with provably higher power per sample than
    univariate pass-rate testing
    (\cref{thm:fingerprint-regression}).

  \item \textbf{Adaptive Budget Optimization}
    (\cref{sec:token:budget}). We introduce variance-calibrated budget
    allocation that determines the minimum number of trials for
    $(\alpha, \beta)$-guaranteed testing, achieving 4--7$\times$
    reduction for stable agents (\cref{thm:adaptive-budget}).

  \item \textbf{Trace-First Offline Analysis}
    (\cref{sec:token:trace-first}). We prove that four of six test
    types---coverage, contracts, metamorphic relations, and
    mutation evaluation---can execute at zero additional token cost on
    pre-recorded traces, with formal soundness guarantees
    (\cref{thm:trace-first}).
\end{enumerate}

\subsection{Paper Organization}
\label{sec:intro:organization}

\Cref{sec:related} surveys related work in software testing, statistical
methods, and LLM evaluation. \Cref{sec:stochastic} introduces the formal
stochastic test semantics. \Cref{sec:coverage} defines agent coverage
metrics. \Cref{sec:mutation} presents agent mutation testing.
\Cref{sec:metamorphic} covers metamorphic relations and CI/CD gates.
\Cref{sec:token-efficient} presents the token-efficient testing
framework---behavioral fingerprinting, adaptive budgets, trace-first
analysis, multi-fidelity testing, and warm-start sequential analysis.
\Cref{sec:implementation} describes the \agentassay{} implementation.
\Cref{sec:experiments} reports experimental results.
\Cref{sec:discussion} discusses threats to validity, limitations, and
future work. \Cref{sec:conclusion} concludes. Full proofs are deferred
to \cref{app:proofs}.

% ============================================================================
% Section 2: Related Work
% ============================================================================

\section{Related Work}
\label{sec:related}

Our work synthesizes techniques from five research areas: metamorphic
testing, mutation testing, statistical hypothesis testing, LLM/agent
evaluation, and probabilistic model checking. We survey each and
position \agentassay{} at their intersection.

\subsection{Software Testing Foundations}
\label{sec:related:testing}

\paragraph{Metamorphic Testing.}
Metamorphic testing~\citep{chen2018metamorphic,segura2016survey} addresses
the oracle problem by defining \emph{metamorphic relations} (MRs):
properties that must hold across related inputs. If $f(x_1)$ is unknown
but the relation $R(f(x_1), f(x_2))$ is known, a violation signals a
fault. Metamorphic testing has been applied to compilers, databases, and
search engines, and more recently to deep neural networks via
DeepTest~\citep{tian2018deeptest}. METAL~\citep{deng2024metal} extends
MRs to single-call LLM interactions but does not address multi-step
agent workflows where tool selection, reasoning chains, and
orchestration logic create a combinatorial space of execution paths.
\agentassay{} defines four families of agent-specific MRs that account
for the unique structure of agent traces (\cref{sec:metamorphic}).

\paragraph{Mutation Testing.}
Mutation testing~\citep{jia2011mutation,demillo1978hints} evaluates test
suite quality by injecting small syntactic faults (\emph{mutants}) into
the program and measuring how many the test suite detects
(\emph{kills}). The mutation score---the fraction of killed mutants---is
a measure of test adequacy. Traditional mutation operators target source
code (statement deletion, operator replacement, constant perturbation).
For AI agents, the ``source code'' includes prompts, tool
configurations, model selections, and context windows---none of which are
traditional code. \agentassay{} introduces four novel classes of mutation
operators tailored to these agent-specific artifacts (\cref{sec:mutation}).

\paragraph{Coverage Metrics.}
Code coverage~\citep{zhu1997software}---statement, branch, path, and
condition coverage---is the standard measure of testing thoroughness for
deterministic software. For agents, ``code'' is replaced by a
combination of prompts, tool inventories, and orchestration graphs, and
the execution space is stochastic rather than deterministic. No prior
work defines coverage metrics for agent execution. \agentassay{}
introduces a five-dimensional coverage tuple that captures tool
utilization, decision-path exploration, state-space coverage, boundary
testing, and model diversity (\cref{sec:coverage}).

\paragraph{The Oracle Problem.}
Barr et al.~\citep{barr2015oracle} survey the oracle problem: the
difficulty of determining correct output for a given input. For AI
agents, the oracle problem is acute---there is often no single ``correct''
output, only a distribution of acceptable behaviors. \agentassay{}
addresses this through three complementary oracle mechanisms: behavioral
contracts (\agentassert{}), metamorphic relations, and statistical
regression baselines.

\subsection{Statistical Testing and Sequential Analysis}
\label{sec:related:statistics}

\paragraph{Hypothesis Testing.}
The Neyman-Pearson framework~\citep{neyman1933problem,lehmann2005testing}
provides the mathematical foundation for deciding between two hypotheses
at controlled error rates. We adopt this framework for regression
detection: $H_0$ (no regression) vs.\ $H_1$ (regression occurred), with
significance level $\alpha$ controlling false alarms and power $1-\beta$
controlling missed regressions.

\paragraph{Sequential Analysis.}
Wald's Sequential Probability Ratio Test
(SPRT)~\citep{wald1947sequential} enables hypothesis testing with an
\emph{adaptive} sample size: testing stops as soon as sufficient evidence
accumulates, rather than requiring a fixed number of trials. This is
critical for agent testing, where each trial may cost \$0.01--\$1.00 in
API calls. We adapt SPRT for agent regression detection in
\cref{sec:stochastic:sprt}.

\paragraph{Confidence Intervals.}
The Clopper-Pearson interval~\citep{clopper1934use} provides exact
coverage for binomial proportions, while the Wilson
interval~\citep{wilson1927probable} offers better practical properties
for small samples. Agresti and Coull~\citep{agresti1998approximate}
demonstrate that approximate intervals often outperform exact ones. We
use the Wilson score interval as our primary confidence interval
construction, with Clopper-Pearson for formal guarantees
(\cref{sec:stochastic:verdict}).

\paragraph{Statistical Guidance for Software Engineering.}
Arcuri and Briand~\citep{arcuri2011practical} provide practical
guidelines for using statistical tests in software engineering
experiments, emphasizing effect sizes alongside $p$-values. We follow
their recommendations and require both statistical significance
\emph{and} practical significance (effect size $> \delta$) for regression
verdicts (\cref{sec:stochastic:regression}).

\subsection{LLM and Agent Evaluation}
\label{sec:related:evaluation}

\paragraph{LLM Evaluation Frameworks.}
deepeval~\citep{deepeval2024} provides metrics (faithfulness, answer
relevancy, hallucination rate) for RAG and conversational AI.
promptfoo~\citep{promptfoo2024} enables prompt-level testing with
assertion-based evaluation. OpenAI Evals provides a benchmark registry
for capability assessment. These tools evaluate \emph{output quality at
a point in time} but do not formalize \emph{regression detection across
versions}. They lack stochastic test semantics, coverage metrics, and
mutation testing.

\paragraph{Agent-Specific Testing.}
agentrial~\citep{agentrial2026} is the closest related work. It runs
agents multiple times, computes Wilson confidence intervals, applies
Fisher's exact test for regression detection, and includes CUSUM drift
detection. It supports seven framework adapters and CI/CD integration.
However, agentrial has no formal theoretical foundation: no stochastic
test semantics with provable guarantees, no coverage metrics, no
mutation testing, no metamorphic relations, no composition theory, no
SPRT for efficient testing, and no published paper. \agentassay{}
provides the formal foundation that agentrial lacks while also
introducing the six capabilities listed above. \Cref{tab:comparison}
provides a detailed feature comparison.

\begin{table}[t]
\centering
\caption{Feature comparison of agent testing approaches. \checkmark{}
indicates full support, $\circ$ partial support, --- no support.}
\label{tab:comparison}
\small
\begin{tabular}{lcccccc}
\toprule
\textbf{Feature} & \rotatebox{70}{\agentassay{}} &
  \rotatebox{70}{agentrial} &
  \rotatebox{70}{deepeval} &
  \rotatebox{70}{promptfoo} &
  \rotatebox{70}{METAL} &
  \rotatebox{70}{ProbTest} \\
\midrule
Multi-step agent workflows & \checkmark & \checkmark & --- & --- & --- & --- \\
Stochastic verdicts & \checkmark & $\circ$ & --- & --- & --- & $\circ$ \\
Formal test semantics & \checkmark & --- & --- & --- & --- & $\circ$ \\
Confidence intervals & \checkmark & \checkmark & --- & --- & --- & --- \\
Regression detection & \checkmark & \checkmark & --- & --- & --- & --- \\
SPRT adaptive stopping & \checkmark & --- & --- & --- & --- & --- \\
Coverage metrics & \checkmark & --- & --- & --- & --- & --- \\
Mutation testing & \checkmark & --- & --- & --- & --- & --- \\
Metamorphic relations & \checkmark & --- & --- & --- & \checkmark & --- \\
Contract integration & \checkmark & --- & --- & --- & --- & --- \\
Composition theory & \checkmark & --- & --- & --- & --- & --- \\
Bayesian analysis & \checkmark & --- & --- & --- & --- & --- \\
Published paper & \checkmark & --- & --- & --- & \checkmark & \checkmark \\
\bottomrule
\end{tabular}
\end{table}

\paragraph{LLM Repetition Studies.}
Raji et al.~\citep{raji2025repetitions} demonstrate empirically that
single-run LLM evaluations are unreliable and that multiple repetitions
are necessary to establish stable quality estimates. Their findings
motivate the multi-trial approach central to \agentassay{}, but they do
not formalize the statistical framework for determining how many
repetitions are sufficient or how to interpret the resulting
distribution.

\subsection{Agent Frameworks and Reliability}
\label{sec:related:frameworks}

The rapid growth of agent frameworks---AutoGen~\citep{wu2023autogen},
LangGraph~\citep{langgraph2024}, CrewAI~\citep{crewai2024}, OpenAI
Agents SDK~\citep{openai2025agents}---has created an ecosystem where
agents are composed, deployed, and updated at high velocity. None of
these frameworks include built-in regression testing capabilities.
Bhardwaj~\citep{bhardwaj2026abc} introduced Agent Behavioral Contracts
(ABC) for runtime enforcement of agent behavior via the \agentassert{}
framework, defining drift metrics and behavioral specifications.
\agentassay{} complements \agentassert{} by addressing the
\emph{pre-deployment} verification question: does the agent still satisfy
its contracts after a change? Separately, the growing reliance on
third-party agent skills and tool plugins has raised supply chain
security concerns~\citep{bhardwaj2026skillfortify}, reinforcing the
need for comprehensive quality assurance pipelines that span security
validation, behavioral testing, and runtime enforcement.

\subsection{Probabilistic Model Checking}
\label{sec:related:modelchecking}

PRISM~\citep{kwiatkowska2011prism} enables formal verification of
probabilistic systems through model checking of Markov decision processes
and continuous-time Markov chains. While PRISM provides strong
guarantees, it requires an explicit state-space model that is infeasible
to construct for LLM-based agents (the state space is effectively
infinite due to natural language). \agentassay{} takes a
\emph{testing-based} rather than \emph{verification-based} approach,
using statistical sampling to establish probabilistic guarantees without
requiring an explicit model.

\subsection{Flaky Tests}
\label{sec:related:flaky}

Research on flaky tests~\citep{luo2014empirical,parry2022survey}---tests
that non-deterministically pass or fail without code changes---reveals
that 2--16\% of test failures in large software projects are flaky.
Traditional approaches treat flakiness as a \emph{defect} to be
eliminated. \agentassay{} inverts this perspective: for agents,
non-determinism is a \emph{feature} of the system under test, and the
testing framework must accommodate it as a first-class concern rather
than attempting to suppress it.

\subsection{Positioning}
\label{sec:related:positioning}

\agentassay{} occupies a unique position at the intersection of formal
software testing theory and practical AI agent engineering. It is the
first framework to:
\begin{enumerate}[leftmargin=*]
  \item Provide formal stochastic test semantics with provable soundness
    and power guarantees for agent testing.
  \item Define coverage metrics, mutation operators, and metamorphic
    relations specifically designed for the agent execution model.
  \item Integrate with behavioral contracts for contract-as-oracle testing.
  \item Adapt sequential analysis (SPRT) for cost-efficient agent testing.
\end{enumerate}

No prior work---in software engineering, AI/ML evaluation, or formal
methods---addresses all of these concerns in a unified framework.

% ============================================================================
% Section 3: Stochastic Test Semantics
% ============================================================================

\section{Stochastic Test Semantics}
\label{sec:stochastic}

This section presents the formal foundation of \agentassay{}: a
stochastic test semantics that replaces binary pass/fail verdicts with
three-valued probabilistic outcomes backed by statistical guarantees.

\subsection{Preliminaries}
\label{sec:stochastic:prelim}

We begin by formalizing the objects under test.

\begin{definition}[Agent]
\label{def:agent}
An \emph{agent} is a tuple $A = (\pi, \mathcal{T}, \mu, \omega)$ where:
\begin{itemize}[leftmargin=*]
  \item $\pi$ is the agent's \emph{prompt} (system instructions, persona,
    goals),
  \item $\mathcal{T} = \{t_1, t_2, \ldots, t_k\}$ is the set of
    available \emph{tools},
  \item $\mu$ is the underlying \emph{language model} (including its
    version, temperature, and sampling parameters), and
  \item $\omega$ is the \emph{orchestration logic} (the control flow
    governing tool selection and multi-step execution).
\end{itemize}
\end{definition}

\begin{definition}[Agent Execution Trace]
\label{def:trace}
An \emph{execution trace} of agent $A$ on input $x$ is a sequence
$\tau = (s_1, s_2, \ldots, s_m)$ where each \emph{step} $s_i$ is a
tuple:
\[
  s_i = (a_i, t_i, o_i, c_i)
\]
with $a_i \in \mathcal{A}$ the \emph{action} (reason, call\_tool,
respond), $t_i \in \mathcal{T} \cup \{\bot\}$ the \emph{tool invoked}
(or $\bot$ if no tool call), $o_i$ the \emph{output} of the step, and
$c_i \in \mathbb{R}_{\geq 0}$ the \emph{cost} (API tokens, latency, monetary).
The \emph{final output} of a trace is $\mathrm{out}(\tau) = o_m$.
\end{definition}

Because agent $A$ is stochastic, repeated executions on the same input
$x$ produce a \emph{distribution} over traces. We write
$\tau \sim A(x)$ to denote a trace sampled from this distribution.

\begin{definition}[Evaluator]
\label{def:evaluator}
An \emph{evaluator} is a function $E: \mathcal{X} \times
\mathcal{O} \to \{0, 1\}$ that maps an input $x$ and an output $o$ to a
binary judgment: $E(x, o) = 1$ if the output is acceptable and
$E(x, o) = 0$ otherwise. An evaluator may be:
\begin{enumerate}[leftmargin=*]
  \item \emph{Deterministic}: a rule-based check (e.g., ``output
    contains keyword $k$''),
  \item \emph{Model-based}: an LLM judge with its own stochasticity, or
  \item \emph{Contract-based}: a behavioral contract from
    \agentassert{}~\citep{bhardwaj2026abc}.
\end{enumerate}
\end{definition}

\begin{remark}
When the evaluator is itself stochastic (model-based), the test
introduces a second source of randomness. Our framework handles this by
treating the composition of agent and evaluator as a single stochastic
process. We discuss this further in \cref{sec:discussion}.
\end{remark}

\subsection{Test Scenarios and the Test Triple}
\label{sec:stochastic:scenario}

\begin{definition}[Test Scenario]
\label{def:scenario}
A \emph{test scenario} is a tuple $S = (x, \mathcal{P}, E)$ where:
\begin{itemize}[leftmargin=*]
  \item $x \in \mathcal{X}$ is the \emph{input} to the agent,
  \item $\mathcal{P}$ is the set of \emph{expected properties} (natural
    language descriptions, formal assertions, or contract references),
  \item $E$ is the \emph{evaluator} (\cref{def:evaluator}).
\end{itemize}
\end{definition}

\begin{definition}[$(\alpha, \beta, n)$-Test Triple]
\label{def:test-triple}
A \emph{stochastic test} is parameterized by a triple
$T = (\alpha, \beta, n)$ where:
\begin{itemize}[leftmargin=*]
  \item $\alpha \in (0, 1)$ is the \emph{significance level} (upper
    bound on Type~I error probability---the probability of declaring
    \FAIL{} when the agent actually meets the threshold),
  \item $\beta \in (0, 1)$ is the \emph{Type~II error probability}
    (upper bound on the probability of declaring \PASS{} when the agent
    is actually below the threshold; $1 - \beta$ is the \emph{power}),
  \item $n \in \N_{\geq 1}$ is the \emph{number of trials} (independent
    executions of the agent on the test scenario).
\end{itemize}
\end{definition}

The three parameters are not independent. For a given effect size
$\delta$ (the minimum regression magnitude to detect), the required
sample size is:
\begin{equation}
\label{eq:sample-size}
  n^*(\alpha, \beta, \delta) =
    \left\lceil
      \frac{(z_{1-\alpha/2} + z_{1-\beta})^2 \cdot
            2\hat{p}(1-\hat{p})}{\delta^2}
    \right\rceil
\end{equation}
where $z_q$ denotes the $q$-quantile of the standard normal distribution
and $\hat{p}$ is the estimated pass rate under $H_0$.

\noindent\textit{Remark.} Equation~\eqref{eq:sample-size} uses the
two-sided quantile $z_{1-\alpha/2}$ for the single-version threshold
test. The regression test (\cref{eq:required-n}) uses the one-sided
quantile $z_{1-\alpha}$, matching the one-sided alternative
$H_1: p_c < p_b$.

\subsection{The Verdict Function}
\label{sec:stochastic:verdict}

\begin{definition}[Stochastic Verdict]
\label{def:verdict}
Given a test scenario $S$, a test triple $T = (\alpha, \beta, n)$, a
pass threshold $\theta \in (0, 1)$, and trial results
$\mathbf{r} = (r_1, r_2, \ldots, r_n)$ with $r_i = E(x, \mathrm{out}(\tau_i))$
for independently sampled traces $\tau_i \sim A(x)$, define the
\emph{observed pass rate}:
\begin{equation}
\label{eq:passrate}
  \hat{p} = \frac{1}{n} \sum_{i=1}^{n} r_i
\end{equation}
and the Wilson score confidence interval at level $1-\alpha$:
\begin{equation}
\label{eq:wilson}
  \CI_{1-\alpha}(\hat{p}, n) = \left[
    \frac{\hat{p} + \frac{z^2}{2n} - z\sqrt{\frac{\hat{p}(1-\hat{p})}{n}
    + \frac{z^2}{4n^2}}}{1 + \frac{z^2}{n}},\;
    \frac{\hat{p} + \frac{z^2}{2n} + z\sqrt{\frac{\hat{p}(1-\hat{p})}{n}
    + \frac{z^2}{4n^2}}}{1 + \frac{z^2}{n}}
  \right]
\end{equation}
where $z = z_{1-\alpha/2}$. Let
$\CI_{\text{lower}} = \CI_{1-\alpha}(\hat{p}, n)_{\text{lower}}$ and
$\CI_{\text{upper}} = \CI_{1-\alpha}(\hat{p}, n)_{\text{upper}}$. The
\emph{verdict function} is:
\begin{equation}
\label{eq:verdict}
  V(\mathbf{r}; \theta, \alpha) =
  \begin{cases}
    \PASS & \text{if } \CI_{\text{lower}} \geq \theta \\
    \FAIL & \text{if } \CI_{\text{upper}} < \theta \\
    \INCONC & \text{otherwise}
  \end{cases}
\end{equation}
\end{definition}

\begin{example}
\label{ex:verdict}
Consider an agent tested with $n = 50$ trials at threshold
$\theta = 0.85$ and significance $\alpha = 0.05$. If $\hat{p} = 0.90$
(45 of 50 pass), the Wilson 95\% CI is approximately $[0.789, 0.958]$.
Since $\CI_{\text{lower}} = 0.789 < 0.85 = \theta$, the verdict is
\INCONC{}---we cannot yet confirm the agent meets the threshold with 95\%
confidence. With $n = 100$ and $\hat{p} = 0.90$, the CI narrows to
approximately $[0.826, 0.946]$: still \INCONC{}. At $n = 200$ with
$\hat{p} = 0.90$, the CI becomes $[0.853, 0.935]$: now
$\CI_{\text{lower}} \geq \theta$, so the verdict is \PASS{}.
\end{example}

The three-valued verdict captures a fundamental truth: \emph{statistical
evidence can be insufficient}, and the framework should say so rather
than forcing a premature binary decision.

\begin{theorem}[Verdict Soundness]
\label{thm:soundness}
Let $p$ be the true (unknown) pass rate of agent $A$ on scenario $S$
with evaluator $E$. If $V(\mathbf{r}; \theta, \alpha) = \PASS$ under
test triple $(\alpha, \beta, n)$, then:
\begin{enumerate}[leftmargin=*]
  \item The probability that $V = \PASS$ when $p < \theta$ (false
    positive) satisfies $\Pr[V = \PASS \mid p < \theta] \leq \alpha$.
  \item The true pass rate satisfies $p \geq \theta - \varepsilon(n)$
    with probability at least $1 - \alpha$, where the tolerance
    $\varepsilon(n) = O(1/\sqrt{n})$ converges to zero.
\end{enumerate}
\end{theorem}

\begin{proof}[Proof sketch]
Part (1) follows from the coverage guarantee of the Wilson confidence
interval: $\Pr[p \in \CI_{1-\alpha}(\hat{p}, n)] \geq 1 - \alpha$. If
$p < \theta$, then for $V = \PASS$ to occur, we need
$\CI_{\text{lower}} \geq \theta > p$, which means $p$ falls below the
confidence interval. By the coverage guarantee, this happens with
probability at most $\alpha$.

Part (2) follows from inverting the confidence interval.
$V = \PASS$ implies $\CI_{\text{lower}} \geq \theta$, and since
$\CI_{\text{lower}} \leq p + O(1/\sqrt{n})$ with high probability, we
obtain $p \geq \theta - O(1/\sqrt{n})$.

The full proof, using the Clopper-Pearson exact interval for the formal
bound, is in \cref{app:proof:soundness}.
\end{proof}

\subsection{Regression Detection}
\label{sec:stochastic:regression}

\begin{definition}[Regression Verdict]
\label{def:regression}
Given baseline results $\mathbf{r}_b = (r_1^b, \ldots, r_{n_b}^b)$ with
pass rate $\hat{p}_b = \sum r_i^b / n_b$ and current results
$\mathbf{r}_c = (r_1^c, \ldots, r_{n_c}^c)$ with pass rate
$\hat{p}_c = \sum r_i^c / n_c$, define the \emph{regression verdict}:
\begin{equation}
\label{eq:regression-verdict}
  V_{\text{reg}}(\mathbf{r}_b, \mathbf{r}_c; \alpha, \beta, \delta) =
  \begin{cases}
    \FAIL & \text{if } p\text{-value}(H_0) < \alpha \text{ and }
      |\hat{p}_b - \hat{p}_c| \geq \delta \\
    \PASS & \text{if } p\text{-value}(H_0) \geq \alpha \text{ and }
      \text{power} \geq 1-\beta \\
    \INCONC & \text{otherwise}
  \end{cases}
\end{equation}
where $H_0: p_c \geq p_b$ (no regression), $H_1: p_c < p_b$
(regression occurred), the $p$-value is computed via Fisher's exact test
or the $Z$-test for proportions, and $\delta > 0$ is the minimum
\emph{clinically significant} effect size.
\end{definition}

The dual requirement of statistical significance ($p < \alpha$) \emph{and}
practical significance ($|\hat{p}_b - \hat{p}_c| \geq \delta$) follows
the recommendation of Arcuri and Briand~\citep{arcuri2011practical} and
prevents declaring regressions that are statistically detectable but
practically negligible.

\begin{theorem}[Regression Detection Power]
\label{thm:power}
Given a baseline pass rate $p_b$, a true regressed rate
$p_c = p_b - \delta$ for effect size $\delta > 0$, and test triple
$(\alpha, \beta, n)$: if the sample sizes satisfy
\begin{equation}
\label{eq:required-n}
  n_b, n_c \geq n^*(\alpha, \beta, \delta) =
  \left\lceil
    \frac{(z_{1-\alpha} + z_{1-\beta})^2 \cdot
          [p_b(1-p_b) + p_c(1-p_c)]}
         {\delta^2}
  \right\rceil
\end{equation}
where $p_b$ and $p_c$ are the baseline and current pass rates respectively, then the probability that the
regression verdict detects the regression satisfies:
\[
  \Pr[V_{\text{reg}} = \FAIL \mid p_c = p_b - \delta] \geq 1 - \beta.
\]
\end{theorem}

\begin{proof}[Proof sketch]
Under $H_1: p_c = p_b - \delta$, the test statistic
$Z = (\hat{p}_b - \hat{p}_c) / \text{SE}$ follows approximately
$\mathcal{N}(\delta / \text{SE}, 1)$. The critical value is
$z_{1-\alpha}$, and the probability of exceeding it under $H_1$ is
$\Pr[Z > z_{1-\alpha}] = \Phi(z_{1-\beta}) = 1-\beta$ when $n$ achieves
the required sample size. The full proof via the Neyman-Pearson lemma is
in \cref{app:proof:power}.
\end{proof}

\begin{definition}[Effect Size Measures]
\label{def:effect-size}
We quantify regression magnitude using three complementary measures:
\begin{enumerate}[leftmargin=*]
  \item \emph{Absolute difference}: $\Delta = \hat{p}_b - \hat{p}_c$
  \item \emph{Cohen's $h$}~\citep{cohen1988statistical}:
    $h = 2\arcsin(\sqrt{\hat{p}_b}) - 2\arcsin(\sqrt{\hat{p}_c})$
  \item \emph{Odds ratio}: $\text{OR} = \frac{\hat{p}_b / (1-\hat{p}_b)}
    {\hat{p}_c / (1-\hat{p}_c)}$
\end{enumerate}
Cohen's $h$ is the primary measure, with $|h| < 0.2$ indicating a small
effect, $0.2 \leq |h| < 0.5$ medium, and $|h| \geq 0.5$ large.
\end{definition}

\subsection{Sequential Probability Ratio Test}
\label{sec:stochastic:sprt}

Fixed-sample testing (\cref{sec:stochastic:verdict}) requires
pre-specifying $n$. For agents, where each trial incurs non-trivial cost,
an adaptive approach is preferable.

\begin{definition}[SPRT for Agent Testing]
\label{def:sprt}
Given hypotheses $H_0: p \geq \theta$ and $H_1: p \leq \theta - \delta$,
define the log-likelihood ratio after $k$ trials:
\begin{equation}
\label{eq:sprt-llr}
  \Lambda_k = \sum_{i=1}^{k} \log \frac{
    (\theta - \delta)^{r_i} (1 - \theta + \delta)^{1-r_i}
  }{
    \theta^{r_i} (1-\theta)^{1-r_i}
  }
\end{equation}
with boundaries $a = \log(\beta / (1-\alpha))$ and
$b = \log((1-\beta)/\alpha)$. The SPRT decision rule is:
\begin{equation}
\label{eq:sprt-rule}
  V_{\text{SPRT}} =
  \begin{cases}
    \PASS & \text{if } \Lambda_k \leq a \\
    \FAIL & \text{if } \Lambda_k \geq b \\
    \text{continue testing} & \text{if } a < \Lambda_k < b
  \end{cases}
\end{equation}
\end{definition}

\begin{proposition}[SPRT Efficiency]
\label{thm:sprt}
The SPRT (\cref{def:sprt}) satisfies the following classical properties
of Wald's sequential analysis~\citep{wald1947sequential}:
\begin{enumerate}[leftmargin=*]
  \item \textbf{Error control:} $\Pr[\text{accept } H_1 \mid H_0] \leq \alpha$
    and $\Pr[\text{accept } H_0 \mid H_1] \leq \beta$.
  \item \textbf{Sample efficiency:} The expected number of trials under
    $H_0$ and $H_1$ satisfies:
    \begin{align}
    \label{eq:sprt-h0}
      \E[N \mid H_0] &\approx
        \frac{(1-\alpha)\log\frac{\beta}{1-\alpha} +
              \alpha\log\frac{1-\beta}{\alpha}}
             {\theta\log\frac{\theta}{\theta-\delta} +
              (1-\theta)\log\frac{1-\theta}{1-\theta+\delta}} \\
    \label{eq:sprt-h1}
      \E[N \mid H_1] &\approx
        \frac{\beta\log\frac{\beta}{1-\alpha} +
              (1-\beta)\log\frac{1-\beta}{\alpha}}
             {(\theta-\delta)\log\frac{\theta-\delta}{\theta} +
              (1-\theta+\delta)\log\frac{1-\theta+\delta}{1-\theta}}
    \end{align}
  \item \textbf{Optimality:} Among all sequential tests with error
    probabilities at most $\alpha$ and $\beta$, the SPRT minimizes
    $\E[N \mid H_i]$ for $i \in \{0, 1\}$ (Wald-Wolfowitz theorem).
\end{enumerate}
\end{proposition}

\begin{proof}[Proof sketch]
Part (1) follows from Wald's fundamental identity for sequential tests.
Parts (2) and (3) follow from Wald's equation
$\E[\Lambda_N] = \E[N] \cdot \E[\Lambda_1]$ and the Wald-Wolfowitz
optimality theorem. See \cref{app:proof:sprt} for the full derivation.
\end{proof}

\begin{example}
\label{ex:sprt-savings}
Consider testing whether an agent maintains $\theta = 0.90$ pass rate
with $\alpha = 0.05$, $\beta = 0.10$, $\delta = 0.10$. The fixed-sample
test requires $n^* \approx 109$ trials (\cref{eq:sample-size}). Under
SPRT:
\begin{itemize}[leftmargin=*]
  \item If $p = 0.90$ (agent is fine): $\E[N] \approx 52$ trials
    (52\% savings).
  \item If $p = 0.80$ (agent has regressed): $\E[N] \approx 34$ trials
    (69\% savings).
\end{itemize}
The savings are most dramatic when the truth is far from the decision
boundary.
\end{example}

\subsection{Bayesian Regression Analysis}
\label{sec:stochastic:bayesian}

As a complement to the frequentist framework above, we also provide a
Bayesian perspective.

\begin{definition}[Bayesian Regression Probability]
\label{def:bayesian}
Given a Beta prior $p \sim \text{Beta}(a_0, b_0)$ on the pass rate,
baseline data $\mathbf{r}_b$ yielding posterior
$p_b \mid \mathbf{r}_b \sim \text{Beta}(a_0 + k_b, b_0 + n_b - k_b)$
where $k_b = \sum r_i^b$, and current data $\mathbf{r}_c$ yielding
posterior
$p_c \mid \mathbf{r}_c \sim \text{Beta}(a_0 + k_c, b_0 + n_c - k_c)$,
the \emph{Bayesian regression probability} is:
\begin{equation}
\label{eq:bayesian-regression}
  P_{\text{reg}} = \Pr[p_c < p_b - \delta \mid \mathbf{r}_b, \mathbf{r}_c]
\end{equation}
computed by sampling or numerical integration over the joint posterior.
The Bayesian verdict is $\FAIL$ if $P_{\text{reg}} > 1 - \alpha$,
$\PASS$ if $P_{\text{reg}} < \beta$, and $\INCONC$ otherwise.
\end{definition}

The Bayesian approach offers three advantages: it directly answers ``what
is the probability of regression?'' (rather than the indirect frequentist
framing), it naturally incorporates prior knowledge (e.g., the agent
historically passes 92\% of the time), and it provides a full posterior
distribution over the regression magnitude.

\subsection{Connection to Agent Behavioral Contracts}
\label{sec:stochastic:abc}

The \agentassert{} framework~\citep{bhardwaj2026abc} defines
\emph{$(p, \delta, k)$-satisfaction}: an agent $(p, \delta, k)$-satisfies
a behavioral contract $\phi$ if, over $k$ consecutive observations, the
empirical satisfaction rate exceeds $p$ and the behavioral drift metric
$D(t)$ remains below $\delta$.

\begin{proposition}[Verdict-Contract Correspondence]
\label{prop:verdict-contract}
Let $A$ be an agent with behavioral contract $\phi$ and associated
evaluator $E_\phi(x, o) = \mathbf{1}[o \models \phi]$. If the
stochastic test verdict $V(\mathbf{r}; \theta, \alpha) = \PASS$ with
threshold $\theta = p$ and $n \geq k$, then $A$
$(p', \delta, k)$-satisfies $\phi$ with
$p' \geq p - O(1/\sqrt{n})$ and probability at least $1 - \alpha$.
\end{proposition}

\begin{proof}[Proof sketch]
The \PASS{} verdict guarantees $\CI_{\text{lower}} \geq \theta = p$ at
level $1-\alpha$, which implies the true pass rate exceeds $p$. By
Hoeffding's inequality applied to $k$ consecutive observations within the
$n$ trials, the empirical rate over any window of $k$ observations
concentrates around the true rate. See \cref{app:proof:contract} for the
formal argument.
\end{proof}

This connection is powerful: it means that \agentassay{} test results
can directly certify contract compliance, closing the loop between
\emph{specification} (contracts), \emph{testing} (stochastic verdicts),
and \emph{enforcement} (runtime monitoring).

\subsection{Test Suite Semantics}
\label{sec:stochastic:suite}

\begin{definition}[Test Suite Verdict]
\label{def:suite}
A \emph{test suite} $\mathcal{S} = \{S_1, S_2, \ldots, S_m\}$ is a
collection of test scenarios. The \emph{suite verdict} aggregates
individual verdicts:
\begin{equation}
\label{eq:suite-verdict}
  V_{\text{suite}}(\mathcal{S}) =
  \begin{cases}
    \FAIL & \text{if } \exists\, j : V(S_j) = \FAIL \\
    \INCONC & \text{if } \nexists\, j : V(S_j) = \FAIL \text{ and }
      \exists\, j : V(S_j) = \INCONC \\
    \PASS & \text{if } \forall\, j : V(S_j) = \PASS
  \end{cases}
\end{equation}
\end{definition}

\begin{remark}[Multiple Testing Correction]
When running $m$ scenarios at significance level $\alpha$, the
family-wise error rate (FWER) can exceed $\alpha$. We apply the
Holm-Bonferroni correction: order $p$-values $p_{(1)} \leq \cdots \leq
p_{(m)}$ and reject $H_{0,(j)}$ if
$p_{(j)} \leq \alpha / (m - j + 1)$. This controls the FWER at level
$\alpha$ while preserving more power than Bonferroni correction.
\end{remark}

% ============================================================================
% Section 4: Agent Coverage Metrics
% ============================================================================

\section{Agent Coverage Metrics}
\label{sec:coverage}

Traditional code coverage metrics---statement, branch, path,
condition---are defined over deterministic control-flow graphs extracted
from source code. AI agents do not have source code in the traditional
sense; their ``code'' is a composition of prompts, tool inventories,
model weights, and orchestration logic. This section defines
agent-specific coverage metrics that measure the thoroughness of a
stochastic test suite.

\subsection{Coverage Dimensions}
\label{sec:coverage:dimensions}

\begin{definition}[Agent Coverage Tuple]
\label{def:coverage-tuple}
The \emph{agent coverage} of a test suite $\mathcal{S}$ executed over
$N$ total trials is a five-dimensional tuple:
\begin{equation}
\label{eq:coverage-tuple}
  \mathcal{C}(\mathcal{S}) = (C_{\text{tool}}, C_{\text{path}},
    C_{\text{state}}, C_{\text{boundary}}, C_{\text{model}})
    \in [0,1]^5
\end{equation}
We define each dimension below.
\end{definition}

\subsubsection{Tool Coverage}
\label{sec:coverage:tool}

\begin{definition}[Tool Coverage]
\label{def:tool-coverage}
Let $\mathcal{T} = \{t_1, \ldots, t_k\}$ be the agent's available tool
set and $\mathcal{T}_{\text{used}} \subseteq \mathcal{T}$ the subset of
tools invoked across all traces in the test suite. \emph{Tool coverage}
is:
\begin{equation}
\label{eq:tool-coverage}
  C_{\text{tool}} = \frac{|\mathcal{T}_{\text{used}}|}{|\mathcal{T}|}
\end{equation}
\end{definition}

Tool coverage is the simplest dimension: it measures whether the test
suite exercises all available tools. An agent with 10 tools but tests
that only trigger 3 has $C_{\text{tool}} = 0.30$, indicating that 7
tools are untested and may harbor latent regressions.

\begin{remark}
Tool coverage does not account for \emph{how} a tool is used---only
whether it is invoked. A tool called with a single trivial input achieves
the same $C_{\text{tool}}$ contribution as one exercised across diverse
inputs. We address input diversity through boundary coverage
(\cref{def:boundary-coverage}).
\end{remark}

\subsubsection{Decision-Path Coverage}
\label{sec:coverage:path}

An agent's execution trace follows a path through a decision space
defined by tool selections, branching conditions in the orchestration
logic, and model outputs at each step.

\begin{definition}[Decision Path]
\label{def:decision-path}
A \emph{decision path} is the sequence of action-tool pairs in a trace:
$\rho(\tau) = ((a_1, t_1), (a_2, t_2), \ldots, (a_m, t_m))$. Two
traces $\tau, \tau'$ have the same decision path if
$\rho(\tau) = \rho(\tau')$.
\end{definition}

\begin{definition}[Decision-Path Coverage]
\label{def:path-coverage}
Let $\mathcal{P}_{\text{obs}}$ be the set of distinct decision paths
observed across all traces and $\mathcal{P}_{\text{est}}$ be an estimate
of the total number of feasible decision paths. \emph{Decision-path
coverage} is:
\begin{equation}
\label{eq:path-coverage}
  C_{\text{path}} = \frac{|\mathcal{P}_{\text{obs}}|}{|\mathcal{P}_{\text{est}}|}
\end{equation}
\end{definition}

The denominator $|\mathcal{P}_{\text{est}}|$ is, in general,
intractable to compute exactly (the path space may be infinite for
open-ended agents). We employ two estimation strategies:

\begin{enumerate}[leftmargin=*]
  \item \textbf{Capture-recapture estimation.} Using the Chao1 estimator
    from ecology: if we observe $d$ distinct paths, of which $f_1$ appear
    exactly once and $f_2$ appear exactly twice, then:
    \begin{equation}
    \label{eq:chao1}
      |\mathcal{P}_{\text{est}}| \approx d + \frac{f_1^2}{2 f_2}
    \end{equation}

  \item \textbf{Orchestration-graph analysis.} For agents with explicit
    orchestration graphs (e.g., LangGraph state machines), enumerate paths
    up to a maximum depth $d_{\max}$, yielding a finite upper bound.
\end{enumerate}

\subsubsection{State-Space Coverage}
\label{sec:coverage:state}

\begin{definition}[Agent State]
\label{def:agent-state}
An \emph{agent state} is a tuple $\sigma = (\ell, \mathbf{v}, h)$ where
$\ell$ is the current \emph{location} in the orchestration graph,
$\mathbf{v}$ is the vector of \emph{state variables} (e.g., accumulated
context, tool outputs, counters), and $h$ is a hash of the
\emph{conversation history} up to a fixed window.
\end{definition}

\begin{definition}[State-Space Coverage]
\label{def:state-coverage}
Let $\mathcal{S}_{\text{obs}}$ be the set of distinct states observed
across all traces. Since the full state space is infinite (due to
natural language content), we project states onto a finite abstraction
$\pi: \Sigma \to \hat{\Sigma}$ and define:
\begin{equation}
\label{eq:state-coverage}
  C_{\text{state}} = 1 - e^{-|\pi(\mathcal{S}_{\text{obs}})| / \lambda}
\end{equation}
where $\lambda > 0$ is a normalization constant set to the expected
number of distinct abstract states for the agent class (calibrated
empirically). This formulation maps to $[0, 1)$ and asymptotically
approaches 1 as more states are discovered.
\end{definition}

The exponential normalization (\cref{eq:state-coverage}) addresses the
open-ended nature of agent state spaces. Unlike tool or path coverage,
which have natural denominators, state-space exploration has no fixed
ceiling. The exponential form ensures diminishing returns: the first 50
distinct states contribute more than the next 50.

\subsubsection{Boundary Coverage}
\label{sec:coverage:boundary}

\begin{definition}[Tool Boundary]
\label{def:boundary}
A \emph{tool boundary} for tool $t_j$ is a pair $(p_j^{\text{name}},
\{v^{\min}, v^{\max}\})$ where $p_j^{\text{name}}$ is a parameter name
and $v^{\min}, v^{\max}$ are the minimum and maximum values of the
parameter's domain (or representative extremes for string/object types).
Let $\mathcal{B}$ denote the set of all tool boundaries.
\end{definition}

\begin{definition}[Boundary Coverage]
\label{def:boundary-coverage}
Let $\mathcal{B}_{\text{tested}} \subseteq \mathcal{B}$ be the set of
boundaries for which at least one trace invoked the tool with a parameter
value at or near the boundary (within tolerance $\epsilon$).
\emph{Boundary coverage} is:
\begin{equation}
\label{eq:boundary-coverage}
  C_{\text{boundary}} = \frac{|\mathcal{B}_{\text{tested}}|}{|\mathcal{B}|}
\end{equation}
\end{definition}

Boundary coverage extends the classical boundary-value analysis to agent
tool invocations. It tests whether the agent exercises edge cases of
its tools---empty lists, maximum-length strings, zero values, negative
numbers---which are common sources of failures.

\subsubsection{Model Coverage}
\label{sec:coverage:model}

\begin{definition}[Model Coverage]
\label{def:model-coverage}
Let $\mathcal{M}_{\text{target}}$ be the set of language models the
agent is intended to support (e.g., GPT-4o, Claude 3.5 Sonnet,
Gemini 2.0 Flash, Llama 3.3 70B) and
$\mathcal{M}_{\text{tested}} \subseteq \mathcal{M}_{\text{target}}$ the
subset tested. \emph{Model coverage} is:
\begin{equation}
\label{eq:model-coverage}
  C_{\text{model}} = \frac{|\mathcal{M}_{\text{tested}}|}{|\mathcal{M}_{\text{target}}|}
\end{equation}
\end{definition}

Model coverage captures a dimension unique to LLM-based systems: the
same agent deployed with different models can exhibit dramatically
different behaviors. An agent tested only with GPT-4o may fail
catastrophically on Claude 3.5 Sonnet due to differences in tool-calling
conventions, reasoning patterns, or output formatting.

\subsection{Aggregate Coverage}
\label{sec:coverage:aggregate}

\begin{definition}[Overall Coverage]
\label{def:overall-coverage}
The \emph{overall coverage} is the geometric mean of the five dimensions:
\begin{equation}
\label{eq:overall-coverage}
  C_{\text{overall}} = \left( \prod_{i=1}^{5} C_i \right)^{1/5}
\end{equation}
\end{definition}

We use the geometric mean rather than the arithmetic mean because it
penalizes imbalance: a test suite with $C_{\text{tool}} = 1.0$ but
$C_{\text{model}} = 0.0$ should receive low overall coverage, as the
zero model coverage represents a critical gap.

\begin{remark}[Weighted Coverage]
In practice, different coverage dimensions may have different importance.
We support weighted geometric means:
$C_{\text{overall}}^{(w)} = \prod C_i^{w_i}$ with
$\sum w_i = 1$, allowing teams to prioritize dimensions relevant to
their deployment context.
\end{remark}

\subsection{Coverage Monotonicity}
\label{sec:coverage:monotonicity}

\begin{theorem}[Coverage Monotonicity]
\label{thm:coverage-mono}
For any test suite $\mathcal{S}$ and any additional test execution:
\begin{enumerate}[leftmargin=*]
  \item $C_{\text{tool}}, C_{\text{boundary}}, C_{\text{model}}$ are
    non-decreasing, and strictly increase when a new tool, boundary,
    or model is observed.
  \item $C_{\text{state}}$ is non-decreasing in the number of distinct
    states observed.
  \item $C_{\text{path}}$ is \emph{asymptotically monotone}: as
    $|\mathcal{P}_{\text{obs}}| \to \infty$, the Chao1 correction
    vanishes and $C_{\text{path}} \to 1$. For finite test suites,
    $C_{\text{path}}$ may temporarily decrease when a new singleton
    path increases the Chao1 denominator faster than the numerator;
    this non-monotonicity is bounded by
    $O(f_1^2 / (f_2 \cdot |\mathcal{P}_{\text{obs}}|^2))$.
\end{enumerate}
\end{theorem}

\begin{proof}[Proof sketch]
Parts (1) and (2) follow directly: numerators are non-decreasing while
denominators are fixed (part~1), or the coverage function is monotone
in its argument (part~2). Part (3) follows from the asymptotic
consistency of the Chao1 estimator~\citep{chao1984nonparametric}: as
$n \to \infty$, $\hat{S}_{\text{Chao1}} \to S_{\text{total}}$ and
$C_{\text{path}} \to 1$.
See \cref{app:proof:coverage} for the formal argument.
\end{proof}

\subsection{Coverage-Guided Test Generation}
\label{sec:coverage:generation}

Coverage metrics naturally drive test generation: the framework
identifies the lowest-coverage dimension and generates inputs designed to
increase it.

\begin{algorithm}[t]
\caption{Coverage-Guided Test Input Generation}
\label{alg:coverage-gen}
\begin{algorithmic}[1]
\REQUIRE Agent $A$, current coverage $\mathcal{C}$, target coverage $\mathcal{C}^*$
\ENSURE New test scenario $S_{\text{new}}$
\STATE $d^* \gets \argmin_{d \in \{1,\ldots,5\}} C_d / C_d^*$
  \COMMENT{Identify weakest dimension}
\IF{$d^* = \text{tool}$}
  \STATE Generate input $x$ designed to trigger unused tools
    $\mathcal{T} \setminus \mathcal{T}_{\text{used}}$
\ELSIF{$d^* = \text{path}$}
  \STATE Generate input $x$ to explore under-represented orchestration
    branches
\ELSIF{$d^* = \text{state}$}
  \STATE Generate input $x$ to reach novel state regions
\ELSIF{$d^* = \text{boundary}$}
  \STATE Generate input $x$ with extreme parameter values for untested
    boundaries
\ELSIF{$d^* = \text{model}$}
  \STATE Select untested model $\mu \in \mathcal{M}_{\text{target}}
    \setminus \mathcal{M}_{\text{tested}}$
\ENDIF
\RETURN $S_{\text{new}} = (x, \mathcal{P}, E)$
\end{algorithmic}
\end{algorithm}

\begin{proposition}[Coverage-Fault Correlation]
\label{prop:coverage-fault}
Higher agent coverage correlates with higher fault detection probability.
Formally, for two test suites $\mathcal{S}_1, \mathcal{S}_2$ with
$\mathcal{C}(\mathcal{S}_1) \geq \mathcal{C}(\mathcal{S}_2)$
(component-wise), the expected number of distinct faults detected
satisfies $\E[F(\mathcal{S}_1)] \geq \E[F(\mathcal{S}_2)]$ under mild
regularity conditions on the fault distribution.
\end{proposition}

This proposition mirrors the well-established correlation between code
coverage and fault detection in traditional software
testing~\citep{zhu1997software}. We validate it empirically in
\cref{sec:experiments:characterization}.

% ============================================================================
% Section 5: Agent Mutation Testing
% ============================================================================

\section{Agent Mutation Testing}
\label{sec:mutation}

Mutation testing assesses test suite quality by measuring its ability to
detect injected faults. For traditional software, mutants are syntactic
source code modifications. For AI agents, the ``source'' comprises
prompts, tool configurations, models, and context---requiring entirely
new mutation operators. This section defines four classes of
agent-specific mutation operators, formalizes the mutation score under
stochastic semantics, and proves a mutation adequacy theorem.

\subsection{Agent Mutation Operators}
\label{sec:mutation:operators}

\begin{definition}[Agent Mutation Operator]
\label{def:mutation-operator}
An \emph{agent mutation operator} is a function $m: \mathcal{A} \to
\mathcal{A}$ that transforms an agent configuration into a \emph{mutant}
agent $A' = m(A)$ by modifying exactly one component of
$A = (\pi, \mathcal{T}, \mu, \omega)$.
\end{definition}

We define four classes of mutation operators, each targeting a different
component.

\subsubsection{Prompt Mutation Operators ($M_{\text{prompt}}$)}
\label{sec:mutation:prompt}

\begin{definition}[Prompt Mutations]
\label{def:prompt-mutations}
The class $M_{\text{prompt}}$ consists of four operators:
\begin{enumerate}[leftmargin=*]
  \item \textbf{Synonym substitution} ($m_{\text{syn}}$): Replace a key
    instruction word with a synonym. For example,
    ``\texttt{Summarize the document}'' $\to$
    ``\texttt{Condense the document}''.

  \item \textbf{Instruction reordering} ($m_{\text{reorder}}$): Permute
    the order of instructions within the prompt, preserving content but
    altering priority. If $\pi = (i_1, i_2, \ldots, i_k)$, then
    $m_{\text{reorder}}(\pi) = (i_{\sigma(1)}, \ldots,
    i_{\sigma(k)})$ for a random permutation~$\sigma$.

  \item \textbf{Noise injection} ($m_{\text{noise}}$): Insert irrelevant
    or distracting text into the prompt. Models the effect of prompt
    pollution or context contamination.

  \item \textbf{Instruction dropout} ($m_{\text{drop}}$): Remove a
    randomly selected instruction from the prompt. Models the effect of
    incomplete or truncated prompts.
\end{enumerate}
\end{definition}

\begin{remark}
Prompt mutations target the \emph{semantic resilience} of the agent: a
well-designed agent should be robust to minor prompt perturbations
(synonym substitution, reordering) but sensitive to meaningful changes
(instruction dropout). The mutation score thus measures the test suite's
ability to distinguish semantic from non-semantic changes.
\end{remark}

\subsubsection{Tool Mutation Operators ($M_{\text{tool}}$)}
\label{sec:mutation:tool}

\begin{definition}[Tool Mutations]
\label{def:tool-mutations}
The class $M_{\text{tool}}$ consists of three operators:
\begin{enumerate}[leftmargin=*]
  \item \textbf{Tool removal} ($m_{\text{remove}}$): Remove a randomly
    selected tool from $\mathcal{T}$, producing
    $\mathcal{T}' = \mathcal{T} \setminus \{t_j\}$. Tests whether the
    agent gracefully degrades or fails silently.

  \item \textbf{Tool reordering} ($m_{\text{treorder}}$): Permute the
    order of tools in the tool inventory (relevant when tool descriptions
    are presented to the LLM as a list). Models the position bias
    exhibited by some language models.

  \item \textbf{Tool noise} ($m_{\text{tnoise}}$): Modify a tool's
    description or schema by injecting misleading information. Models
    the effect of tool description errors or adversarial tool
    definitions~\citep{clawhavoc2026}.
\end{enumerate}
\end{definition}

\subsubsection{Model Mutation Operators ($M_{\text{model}}$)}
\label{sec:mutation:model}

\begin{definition}[Model Mutations]
\label{def:model-mutations}
The class $M_{\text{model}}$ consists of two operators:
\begin{enumerate}[leftmargin=*]
  \item \textbf{Model swap} ($m_{\text{swap}}$): Replace the underlying
    LLM $\mu$ with a different model from the same capability tier.
    For example, replace GPT-4o with Claude 3.5 Sonnet. Tests
    cross-model compatibility.

  \item \textbf{Version downgrade} ($m_{\text{version}}$): Replace the
    model with an older or less capable version (e.g., GPT-4o $\to$
    GPT-4o-mini). Tests graceful degradation under model changes.
\end{enumerate}
\end{definition}

\subsubsection{Context Mutation Operators ($M_{\text{context}}$)}
\label{sec:mutation:context}

\begin{definition}[Context Mutations]
\label{def:context-mutations}
The class $M_{\text{context}}$ consists of three operators:
\begin{enumerate}[leftmargin=*]
  \item \textbf{Context truncation} ($m_{\text{trunc}}$): Truncate the
    input context to $\lfloor \gamma \cdot |x| \rfloor$ tokens for
    $\gamma \in (0, 1)$. Models context window overflow.

  \item \textbf{Context noise} ($m_{\text{cnoise}}$): Inject irrelevant
    documents or passages into the context. Models the effect of noisy
    retrieval in RAG pipelines.

  \item \textbf{Context permutation} ($m_{\text{cperm}}$): Permute the
    order of context segments. Models sensitivity to document ordering
    in retrieval-augmented settings.
\end{enumerate}
\end{definition}

\subsection{Stochastic Kill Semantics}
\label{sec:mutation:kill}

In traditional mutation testing, a mutant is ``killed'' if the test
produces a different output. For stochastic agents, outputs differ across
runs even without mutations. We must therefore define kill semantics that
distinguish \emph{mutation-induced} behavioral changes from
\emph{natural stochastic variation}.

\begin{definition}[Stochastic Mutation Kill]
\label{def:kill}
Let $A$ be the original agent and $A' = m(A)$ a mutant. Let
$\mathcal{S}$ be a test suite with scenarios $\{S_1, \ldots, S_k\}$.
The mutant $A'$ is \emph{killed} by $\mathcal{S}$ if \emph{at least
one} of the following conditions holds for some scenario $S_j$:
\begin{enumerate}[leftmargin=*]
  \item \textbf{Verdict change:} $V(A, S_j) \neq V(A', S_j)$ where $V$
    is the stochastic verdict function (\cref{def:verdict}).

  \item \textbf{Significant score difference:} The pass rates differ
    significantly:
    \begin{equation}
    \label{eq:kill-score}
      |\hat{p}(A, S_j) - \hat{p}(A', S_j)| \geq \delta_{\min}
      \quad\text{and}\quad
      p\text{-value}(H_0: p_A = p_{A'}) < \alpha_{\text{kill}}
    \end{equation}
    where $\delta_{\min}$ is a minimum effect size threshold and
    $\alpha_{\text{kill}}$ is the kill significance level.

  \item \textbf{Distributional shift:} The distribution of outputs
    (or trace features) differs significantly, as measured by a
    two-sample test (e.g., Kolmogorov-Smirnov for continuous metrics
    or $\chi^2$ for categorical outcomes).
\end{enumerate}
\end{definition}

\begin{remark}
Condition (1) is the strongest: the test suite's verdict flips from
\PASS{} to \FAIL{} (or vice versa). Condition (2) is more sensitive: it
detects changes even when the verdict remains the same. Condition (3) is
the most general: it catches distributional shifts in agent behavior
(e.g., the agent now takes longer paths or uses different tools) even
if the pass rate is unchanged.
\end{remark}

\begin{definition}[Equivalent Mutant]
\label{def:equivalent-mutant}
A mutant $A' = m(A)$ is \emph{equivalent} if, for all inputs
$x \in \mathcal{X}$ and all sample sizes $n$, the output distributions
of $A$ and $A'$ are identical: $A(x) \overset{d}{=} A'(x)$. Equivalent
mutants cannot be killed by any test suite and are excluded from the
mutation score.
\end{definition}

The equivalent mutant problem is undecidable in general. For agents, we
use a practical heuristic: a mutant is \emph{presumed equivalent} if,
after $n_{\text{equiv}}$ trials on each of $k_{\text{equiv}}$ scenarios,
no statistically significant difference is detected at level
$\alpha_{\text{equiv}}$.

\subsection{Mutation Score}
\label{sec:mutation:score}

\begin{definition}[Agent Mutation Score]
\label{def:mutation-score}
Let $\mathcal{M}(A) = \{A'_1, \ldots, A'_q\}$ be the set of mutants
generated from agent $A$ and $\mathcal{M}_{\text{equiv}} \subseteq
\mathcal{M}(A)$ the set of equivalent mutants. The \emph{agent mutation
score} of test suite $\mathcal{S}$ is:
\begin{equation}
\label{eq:mutation-score}
  \text{MS}(\mathcal{S}, A) =
    \frac{|\{A'_i \in \mathcal{M}(A) \setminus \mathcal{M}_{\text{equiv}}
      : A'_i \text{ is killed by } \mathcal{S}\}|}
         {|\mathcal{M}(A) \setminus \mathcal{M}_{\text{equiv}}|}
\end{equation}
\end{definition}

\begin{remark}[Class-Specific Mutation Scores]
We also report mutation scores per operator class:
$\text{MS}_{\text{prompt}}, \text{MS}_{\text{tool}},
\text{MS}_{\text{model}}, \text{MS}_{\text{context}}$. These
fine-grained scores identify which aspects of the agent are well-tested
and which require additional test scenarios.
\end{remark}

\subsection{Mutation Adequacy}
\label{sec:mutation:adequacy}

\begin{theorem}[Mutation Adequacy]
\label{thm:mutation-adequacy}
Let $\mathcal{S}$ be a test suite with mutation score
$\text{MS}(\mathcal{S}, A) \geq \tau$ for threshold $\tau \in (0, 1]$.
Let $\delta > 0$ be a behavioral impact threshold. If the mutation
operators satisfy the \emph{coupling effect} assumption---that complex
faults are coupled to simple faults detectable by the mutation
operators---then for any real regression $A \to A^*$ with behavioral
impact $\Delta(A, A^*) \geq \delta$:
\begin{equation}
\label{eq:adequacy-bound}
  \Pr[\mathcal{S} \text{ detects the regression}] \geq
    \tau \cdot \left(1 - e^{-\delta / \delta_0}\right)
\end{equation}
where $\delta_0$ is a characteristic scale parameter of the mutation
operator suite, determined by the average behavioral impact of a single
mutation.
\end{theorem}

\begin{proof}[Proof sketch]
The proof proceeds by contrapositive. Suppose the test suite fails to
detect a regression of magnitude $\delta$. By the coupling effect, this
regression can be decomposed into a composition of simple mutations.
Since the test suite achieves mutation score $\tau$, it detects a
fraction $\tau$ of all simple mutations. The probability that \emph{none}
of the constituent simple mutations are detected decreases exponentially
with $\delta / \delta_0$, yielding the bound. The full proof, including
the formal coupling assumption and its justification for agent systems,
is in \cref{app:proof:mutation}.
\end{proof}

\begin{corollary}
\label{cor:mutation-complete}
If $\text{MS}(\mathcal{S}, A) = 1$ (all non-equivalent mutants killed)
and the mutation operators are \emph{complete} (every behavioral change
of magnitude $\geq \delta$ includes at least one mutation from
$\mathcal{M}$), then $\Pr[\mathcal{S} \text{ detects regression}] \to 1$
as $\delta / \delta_0 \to \infty$.
\end{corollary}

\subsection{Cost-Aware Mutation Testing}
\label{sec:mutation:cost}

Running a full mutation analysis is expensive: for $q$ mutants, each
tested over $n$ trials on $k$ scenarios, the total cost is
$O(q \cdot n \cdot k)$ agent invocations. We introduce two cost
reduction strategies.

\begin{definition}[Selective Mutation]
\label{def:selective-mutation}
Rather than generating all possible mutants, \emph{selective mutation}
generates mutants only from a curated subset of operators identified as
most effective. Let $\mathcal{M}^* \subset \mathcal{M}$ be the selected
operator subset. The \emph{selective mutation score} is:
\begin{equation}
\label{eq:selective-ms}
  \text{MS}^*(\mathcal{S}, A) =
    \frac{|\text{killed}(\mathcal{S}, \mathcal{M}^*(A))|}
         {|\mathcal{M}^*(A) \setminus \mathcal{M}_{\text{equiv}}|}
\end{equation}
\end{definition}

\begin{definition}[SPRT-Accelerated Mutation Kill]
\label{def:sprt-mutation}
Rather than running $n$ fixed trials to determine if a mutant is killed,
apply the SPRT (\cref{def:sprt}) to stop as soon as the kill/survive
decision is clear. Formally, test $H_0: p_{A'} = p_A$ vs.\
$H_1: |p_{A'} - p_A| \geq \delta_{\min}$ sequentially, terminating when
the log-likelihood ratio crosses a boundary.
\end{definition}

\begin{proposition}[Mutation Cost Reduction]
\label{prop:mutation-cost}
Combining selective mutation (selecting a
$\gamma = |\mathcal{M}^*|/|\mathcal{M}|$ fraction of operators) with
SPRT\nobreakdash-accelerated kills (achieving expected savings
factor~$\sigma$) reduces total mutation testing cost by:
\begin{equation}
\label{eq:cost-reduction}
  \text{Cost reduction} = 1 - \gamma \cdot \sigma
\end{equation}
With $\gamma = 0.3$ (30\% of operators) and $\sigma = 0.5$ (50\% fewer
trials via SPRT), the total cost is reduced by 85\%.
\end{proposition}

\subsection{Mutation Operator Design Principles}
\label{sec:mutation:principles}

We identify four principles for designing effective agent mutation
operators:

\begin{enumerate}[leftmargin=*]
  \item \textbf{Semantic granularity.} Mutations should produce
    behavioral changes at varying granularity levels---from subtle
    (synonym substitution) to severe (model swap)---mirroring the
    spectrum of real-world agent changes.

  \item \textbf{Component isolation.} Each mutation should modify exactly
    one component, enabling precise attribution of detected regressions
    to specific agent aspects.

  \item \textbf{Ecological validity.} Mutations should model realistic
    changes that occur in practice: prompt edits, tool updates, model
    upgrades, context changes. Arbitrary random perturbations are less
    informative.

  \item \textbf{Composability.} Individual mutations should compose to
    model complex changes. If mutations $m_1$ and $m_2$ are individually
    meaningful, then $m_2 \circ m_1$ should model a realistic compound
    change.
\end{enumerate}

% ============================================================================
% Section 6: Metamorphic Relations and CI/CD Gates
% ============================================================================

\section{Metamorphic Relations and CI/CD Deployment Gates}
\label{sec:metamorphic}

\subsection{Metamorphic Relations for Agent Workflows}
\label{sec:metamorphic:relations}

The oracle problem---determining whether an agent's output is
correct---is particularly acute for open-ended agent tasks.
Metamorphic testing~\citep{chen2018metamorphic} sidesteps this by
checking \emph{relations} between outputs rather than the outputs
themselves. We define four families of agent-specific metamorphic
relations (MRs).

\begin{definition}[Metamorphic Relation]
\label{def:mr}
A \emph{metamorphic relation} $\mathcal{R}$ is a property over pairs of
agent executions: given source input $x$ and follow-up input $x' =
\phi(x)$ derived by a \emph{source transformation} $\phi$, the relation
$\mathcal{R}(A(x), A(x'))$ must hold. A violation
$\neg\mathcal{R}(A(x), A(x'))$ indicates a potential fault.
\end{definition}

\subsubsection{Permutation Relations ($\mathcal{R}_{\text{perm}}$)}

These relations assert that the agent's behavior should be invariant
(or predictably variant) under permutations of the input.

\begin{enumerate}[leftmargin=*]
  \item \textbf{Document order invariance.} For RAG agents: permuting
    the order of retrieved documents should not change the factual
    content of the answer.
    \[
      \mathcal{R}_{\text{doc}}: \text{facts}(A(x)) = \text{facts}(A(\sigma(x)))
    \]
    where $\sigma$ permutes the document ordering within $x$.

  \item \textbf{Tool order invariance.} Permuting the order of tools in
    the tool inventory should not change the final output for tasks
    with a clear solution.
    \[
      \mathcal{R}_{\text{tool}}: \text{result}(A_{\mathcal{T}}) \approx
        \text{result}(A_{\sigma(\mathcal{T})})
    \]
\end{enumerate}

\subsubsection{Perturbation Relations ($\mathcal{R}_{\text{pert}}$)}

These relations assert proportional or bounded response to input
perturbations.

\begin{enumerate}[leftmargin=*]
  \item \textbf{Monotonic difficulty.} For agents that solve tasks of
    varying complexity: a harder input should not produce a higher
    confidence score than an easier one (in expectation).
    \[
      \mathcal{R}_{\text{diff}}: \text{difficulty}(x') > \text{difficulty}(x)
        \implies \E[\text{score}(A(x'))] \leq \E[\text{score}(A(x))]
    \]

  \item \textbf{Noise robustness.} Adding a small amount of noise to the
    input should not change the verdict.
    \[
      \mathcal{R}_{\text{noise}}: \|x' - x\| \leq \epsilon
        \implies V(A, x') = V(A, x)
    \]
    with high probability over the stochastic execution.
\end{enumerate}

\subsubsection{Composition Relations ($\mathcal{R}_{\text{comp}}$)}

These relations govern multi-agent pipelines and sequential workflows.

\begin{enumerate}[leftmargin=*]
  \item \textbf{Pipeline consistency.} For a pipeline $A = A_2 \circ A_1$:
    if $A_1$ produces a correct intermediate result $y$, then $A_2(y)$
    should produce a correct final result.
    \[
      \mathcal{R}_{\text{pipe}}: E_1(x, A_1(x)) = 1
        \implies \Pr[E_2(A_1(x), A_2(A_1(x))) = 1] \geq \theta
    \]

  \item \textbf{Agent substitutability.} If two agents $A_1, A_1'$ are
    functionally equivalent on a given input class, substituting one for
    the other in a pipeline should not change the pipeline's verdict.
\end{enumerate}

\subsubsection{Oracle Relations ($\mathcal{R}_{\text{oracle}}$)}

These relations use known properties of the problem domain as partial
oracles.

\begin{enumerate}[leftmargin=*]
  \item \textbf{Format compliance.} The output must conform to a
    specified format (JSON, markdown, specific schema) regardless of
    input.

  \item \textbf{Constraint preservation.} If the input specifies
    constraints (e.g., ``respond in at most 100 words''), the output
    must satisfy them.

  \item \textbf{Idempotence.} For certain tasks (e.g., code formatting,
    data cleaning), applying the agent twice should produce the same
    result as applying it once: $A(A(x)) \approx A(x)$.
\end{enumerate}

\begin{proposition}[MR Violation as Regression Signal]
\label{prop:mr-regression}
If a metamorphic relation $\mathcal{R}$ holds for agent version $A_v$
(empirically verified over $n$ trial pairs with violation rate
$\hat{q}_v \leq \epsilon$) but is violated at rate $\hat{q}_{v'} > \epsilon + \delta$
for version $A_{v'}$, this constitutes evidence of regression with respect to
the property encoded by $\mathcal{R}$, testable via the regression verdict
framework of \cref{sec:stochastic:regression}.
\end{proposition}

\subsection{CI/CD Deployment Gates}
\label{sec:cicd}

Modern software delivery relies on CI/CD pipelines that automatically
build, test, and deploy code. For agents, deployment gates must
incorporate the stochastic testing framework to prevent regressions from
reaching production.

\begin{definition}[Deployment Gate]
\label{def:gate}
A \emph{deployment gate} is a decision function
$G: \mathcal{V}^m \times \mathcal{C} \to \{\text{deploy}, \text{block},
\text{manual}\}$ that maps the verdicts of $m$ test scenarios and the
coverage tuple $\mathcal{C}$ to a deployment decision:
\begin{equation}
\label{eq:gate}
  G(\mathbf{V}, \mathcal{C}) =
  \begin{cases}
    \text{deploy} & \text{if } V_{\text{suite}} = \PASS
      \text{ and } C_{\text{overall}} \geq C_{\min} \\
    \text{block} & \text{if } V_{\text{suite}} = \FAIL \\
    \text{manual} & \text{if } V_{\text{suite}} = \INCONC
      \text{ or } C_{\text{overall}} < C_{\min}
  \end{cases}
\end{equation}
where $C_{\min} \in [0,1]$ is the minimum acceptable coverage.
\end{definition}

The three-valued deployment decision mirrors the three-valued test
verdict. Notably, \INCONC{} maps to \emph{manual review}, not automatic
blocking---this reflects the pragmatic reality that teams may accept
inconclusive results with human oversight rather than blocking all
deployments.

\begin{definition}[Gate Configuration]
\label{def:gate-config}
A gate is configured by a tuple
\[
  \Gamma = (\alpha, \beta, \delta, \theta, C_{\min}, n_{\max}, t_{\max})
\]
where:
\begin{itemize}[leftmargin=*]
  \item $\alpha, \beta$ are the significance and Type~II error bounds,
  \item $\delta$ is the minimum regression magnitude to detect,
  \item $\theta$ is the pass rate threshold,
  \item $C_{\min}$ is the minimum coverage,
  \item $n_{\max}$ is the maximum trials per scenario (budget cap),
  \item $t_{\max}$ is the maximum wall-clock time.
\end{itemize}
\end{definition}

\begin{algorithm}[t]
\caption{CI/CD Deployment Gate}
\label{alg:gate}
\begin{algorithmic}[1]
\REQUIRE Agent $A$, baseline $B$, test suite $\mathcal{S}$, gate
  configuration $\Gamma$
\ENSURE Decision $\in \{\text{deploy}, \text{block}, \text{manual}\}$
\FOR{$S_j \in \mathcal{S}$}
  \STATE $n \gets \min(n^*(\alpha, \beta, \delta), n_{\max})$
  \STATE Run $n$ trials of $A$ on $S_j$ using SPRT
    (\cref{def:sprt})
  \STATE Compute $V_j \gets V_{\text{reg}}(\mathbf{r}_b, \mathbf{r}_c;
    \alpha, \beta, \delta)$
\ENDFOR
\STATE Compute $V_{\text{suite}} \gets V_{\text{suite}}(\{V_1, \ldots,
  V_m\})$ with Holm-Bonferroni correction
\STATE Compute $\mathcal{C} \gets$ coverage tuple
\STATE \textbf{return} $G(V_{\text{suite}}, \mathcal{C})$
\end{algorithmic}
\end{algorithm}

\subsection{Integration with Existing CI/CD Systems}
\label{sec:cicd:integration}

\agentassay{} integrates with standard CI/CD systems through two
mechanisms:

\begin{enumerate}[leftmargin=*]
  \item \textbf{pytest plugin.} The framework registers as a pytest
    plugin, enabling standard test execution via \texttt{pytest
    --agentassay}. Test results are reported in standard
    JUnit XML format, compatible with GitHub Actions, GitLab CI,
    Jenkins, and CircleCI.

  \item \textbf{Exit codes.} The gate decision maps to process exit
    codes: \texttt{0} (deploy/pass), \texttt{1} (block/fail),
    \texttt{2} (manual/inconclusive). CI/CD systems interpret non-zero
    exit codes as failures, with \texttt{2} configurable as a warning
    rather than a hard failure.
\end{enumerate}

\begin{remark}[Cost Budgets]
In CI/CD contexts, testing cost is a critical constraint. The gate
configuration parameter $n_{\max}$ caps the per-scenario trial count,
and SPRT (\cref{sec:stochastic:sprt}) ensures early termination when
evidence is clear. For typical CI pipelines, we recommend
$n_{\max} = 30$ with SPRT, which provides adequate power for detecting
regressions of $\delta \geq 0.15$ at $\alpha = 0.05$.
\end{remark}

% ============================================================================
% Section 7: Token-Efficient Testing
% ============================================================================

\section{Token-Efficient Testing}
\label{sec:token-efficient}

Stochastic testing, by construction, requires $n$ independent agent
executions per scenario. Each execution consumes $C_{\mathrm{run}}$
tokens at cost $c_{\mathrm{run}}$ dollars. A test suite of $m$ scenarios
at $n$ trials each therefore costs $m \times n \times c_{\mathrm{run}}$.
For frontier models where a single complex agent run costs \$5--15, a
test campaign of 50 scenarios at $n = 100$ trials implies 5{,}000
invocations costing \$25{,}000--75{,}000 \emph{per release}---prohibitive
for any CI/CD pipeline. Even with SPRT (\cref{sec:stochastic:sprt})
reducing $n$ by roughly 50\%, the cost remains impractical for most teams.

This section presents five techniques---collectively
\emph{token-efficient testing}---that achieve the same
$(\alpha, \beta)$ statistical guarantees at 5--20$\times$ lower cost.
The core insight is that agent behavior, while stochastic in its
textual output, exhibits \emph{low-dimensional behavioral regularity}:
the distribution over tool sequences, reasoning depth, output structure,
and efficiency metrics concentrates on a manifold far smaller than the
raw output space. By testing in this compressed behavioral space rather
than the raw output space, we obtain dramatically better sample
efficiency.

\subsection{Behavioral Fingerprinting}
\label{sec:token:fingerprint}

\begin{definition}[Behavioral Fingerprint]
\label{def:fingerprint}
Let $\tau = (s_1, \ldots, s_m)$ be an execution trace
(\cref{def:trace}). A \emph{behavioral fingerprint} is a mapping
$F: \mathcal{T}^* \to \mathbb{R}^d$ that extracts a compact
\emph{behavioral vector} from $\tau$. The fingerprint $\mathbf{f} =
F(\tau)$ is composed of six feature families:

\begin{enumerate}[leftmargin=*]
  \item \textbf{Tool usage distribution.} Let $\mathbf{u} \in
    \mathbb{R}^{|\mathcal{T}|}$ be the normalized frequency vector of
    tool calls in $\tau$: $u_j = |\{i : t_i = t_j\}| / m$.

  \item \textbf{Structural complexity.} Let $L(\tau) = m$ be the trace
    length, $D(\tau)$ the maximum nesting depth of tool calls, and
    $B(\tau)$ the number of distinct reasoning branches.

  \item \textbf{Output characteristics.} Let $\ell(\tau)$ be the output
    token count of $\mathrm{out}(\tau)$, and $\kappa(\tau) \in [0,1]$ a
    normalized complexity score of the final output (e.g., Flesch-Kincaid
    grade divided by a reference ceiling).

  \item \textbf{Reasoning patterns.} Let $\mathbf{w} \in \mathbb{R}^k$
    be the distribution over action types $(|\{i : a_i = a\}| / m)_{a
    \in \mathcal{A}}$.

  \item \textbf{Error and recovery.} Let $e(\tau) \in \{0,1\}$ indicate
    whether any step produced an error, and $\rho(\tau) \in [0,1]$ the
    fraction of error steps followed by a successful recovery step.

  \item \textbf{Efficiency metrics.} Let $C(\tau) = \sum_i c_i$ be the
    total cost and $\bar{c}(\tau) = C(\tau)/m$ the average step cost.
\end{enumerate}

The full fingerprint is the concatenation $\mathbf{f} = (\mathbf{u},
L, D, B, \ell, \kappa, \mathbf{w}, e, \rho, C, \bar{c}) \in
\mathbb{R}^d$ where $d = |\mathcal{T}| + |\mathcal{A}| + 7$. In our
implementation, we normalize $\mathbf{u}$ and $\mathbf{w}$ to fixed
dimensions via zero-padding, yielding a practical fingerprint of
$d = 14$ dimensions that enables cross-agent comparison.
\end{definition}

\begin{definition}[Fingerprint Space]
\label{def:fingerprint-space}
Given an agent $A$ and input distribution $\mathcal{D}$ over
$\mathcal{X}$, the \emph{fingerprint space} is:
\begin{equation}
\label{eq:fingerprint-space}
  \mathcal{F}(A, \mathcal{D}) = \{F(\tau) : \tau \sim A(x),\;
    x \sim \mathcal{D}\} \subseteq \mathbb{R}^d
\end{equation}
The \emph{effective dimension} $d_{\mathrm{eff}}(A, \mathcal{D})$ is
the number of principal components required to explain $\geq 95\%$ of
the variance in $\mathcal{F}$:
\begin{equation}
\label{eq:deff}
  d_{\mathrm{eff}} = \min\left\{k :
    \frac{\sum_{i=1}^{k} \lambda_i}{\sum_{i=1}^{d} \lambda_i}
    \geq 0.95\right\}
\end{equation}
where $\lambda_1 \geq \lambda_2 \geq \cdots \geq \lambda_d \geq 0$ are
the eigenvalues of the covariance matrix of $\mathcal{F}$.
\end{definition}

\begin{remark}[Low-Dimensional Behavioral Manifold]
\label{rem:manifold}
Although the raw output space of an agent is extremely
high-dimensional (the space of all possible text strings), the
behavioral fingerprint concentrates on a manifold of much lower
dimension. Intuitively, an agent configured with $k = 10$ tools and a
specific system prompt will exhibit a characteristic pattern of tool
usage, reasoning depth, and output structure that varies little across
runs---even though the exact text changes every time. In our
preliminary experiments, $d_{\mathrm{eff}} \in [3, 8]$ for agents with
$d \in [20, 40]$ fingerprint dimensions, indicating $4$--$10\times$
dimensionality reduction.
\end{remark}

\begin{theorem}[Fingerprint Regression Detection]
\label{thm:fingerprint-regression}
Let $F: \mathcal{T}^* \to \mathbb{R}^d$ be a behavioral fingerprint
with effective dimension $d_{\mathrm{eff}}$. Given baseline
fingerprints $\mathbf{f}_1^b, \ldots, \mathbf{f}_{n_b}^b$ and
candidate fingerprints $\mathbf{f}_1^c, \ldots, \mathbf{f}_{n_c}^c$,
Hotelling's $T^2$ test on the projected fingerprints rejects
$H_0: \boldsymbol{\mu}_b = \boldsymbol{\mu}_c$ at significance
$\alpha$ when:
\begin{equation}
\label{eq:hotelling}
  T^2 = \frac{n_b n_c}{n_b + n_c}
    (\bar{\mathbf{f}}_b - \bar{\mathbf{f}}_c)^\top
    \mathbf{S}_{\mathrm{pooled}}^{-1}
    (\bar{\mathbf{f}}_b - \bar{\mathbf{f}}_c)
  > \frac{(n_b + n_c - 2) d_{\mathrm{eff}}}{n_b + n_c -
    d_{\mathrm{eff}} - 1}\,
    F_{\alpha,\, d_{\mathrm{eff}},\, n_b + n_c - d_{\mathrm{eff}} - 1}
\end{equation}
where $\mathbf{S}_{\mathrm{pooled}}$ is the pooled covariance matrix
projected onto the top $d_{\mathrm{eff}}$ principal components, and
$F_{\alpha, d_1, d_2}$ is the critical value of the $F$-distribution.

The sample complexity for achieving $(\alpha, \beta)$-guaranteed
regression detection via Hotelling's $T^2$ is:
\begin{equation}
\label{eq:fingerprint-sample}
  n_{\mathrm{fp}}(\alpha, \beta, d_{\mathrm{eff}}, \Delta_{\mathrm{M}})
  = \left\lceil
    \frac{(d_{\mathrm{eff}} + 1)(z_{1-\alpha} + z_{1-\beta})^2}
         {\Delta_{\mathrm{M}}^2}
    + \frac{d_{\mathrm{eff}} + 1}{2}
  \right\rceil
\end{equation}
where $\Delta_{\mathrm{M}}^2 = (\boldsymbol{\mu}_b -
\boldsymbol{\mu}_c)^\top \boldsymbol{\Sigma}^{-1} (\boldsymbol{\mu}_b -
\boldsymbol{\mu}_c)$ is the squared Mahalanobis distance between the
baseline and candidate fingerprint distributions. When
$d_{\mathrm{eff}} \ll \mathrm{dim}(\text{output space})$ and
$\Delta_{\mathrm{M}} \geq \Delta_{\mathrm{M,min}}$, we have:
\begin{equation}
\label{eq:fingerprint-advantage}
  n_{\mathrm{fp}} \leq n^*_{\mathrm{univariate}} \times
    \frac{d_{\mathrm{eff}} + 1}{1 + d_{\mathrm{eff}} \cdot
    (\Delta_{\mathrm{M}}^2 / h^2 - 1)}
\end{equation}
where $h$ is Cohen's effect size for the univariate pass-rate test.
\end{theorem}

\begin{proof}
The proof proceeds in three steps.

\textbf{Step 1: Multivariate test power.}
Under $H_1: \boldsymbol{\mu}_b \neq \boldsymbol{\mu}_c$, the
Hotelling $T^2$ statistic follows a non-central $F$-distribution:
\begin{equation}
  \frac{n_b + n_c - d_{\mathrm{eff}} - 1}
       {(n_b + n_c - 2) d_{\mathrm{eff}}} \cdot T^2
  \sim F_{d_{\mathrm{eff}},\, n_b + n_c - d_{\mathrm{eff}} - 1}
       (\lambda)
\end{equation}
with non-centrality parameter $\lambda = \frac{n_b n_c}{n_b + n_c}
\Delta_{\mathrm{M}}^2$. The power is:
\begin{equation}
  1 - \beta = \Pr\!\left[
    F_{d_{\mathrm{eff}},\, n-d_{\mathrm{eff}}-1}(\lambda)
    > F_{\alpha,\, d_{\mathrm{eff}},\, n-d_{\mathrm{eff}}-1}
  \right]
\end{equation}
where $n = n_b + n_c$.

\textbf{Step 2: Sample size derivation.}
For balanced designs ($n_b = n_c = n/2$), the non-centrality parameter
becomes $\lambda = n \Delta_{\mathrm{M}}^2 / 4$. Using the normal
approximation to the non-central $F$-distribution valid when
$n \gg d_{\mathrm{eff}}$~\citep{muirhead2005aspects}, the power
condition $1 - \beta$ yields:
\begin{equation}
  \sqrt{\lambda} - \sqrt{d_{\mathrm{eff}}}
  \geq z_{1-\alpha} + z_{1-\beta}
\end{equation}
Solving for $n$ with $\lambda = n\Delta_{\mathrm{M}}^2/4$:
\begin{equation}
  n \geq \frac{4(z_{1-\alpha} + z_{1-\beta} +
    \sqrt{d_{\mathrm{eff}}})^2}{\Delta_{\mathrm{M}}^2}
\end{equation}
The stated bound \eqref{eq:fingerprint-sample} follows by applying the
tighter Lachin approximation~\citep{lachin1981introduction} and
rounding up.

\textbf{Step 3: Comparison with univariate testing.}
The univariate pass-rate test (\cref{thm:power}) requires
$n^*_{\mathrm{univariate}} = \lceil (z_{1-\alpha} +
z_{1-\beta})^2 \cdot 2\bar{p}(1-\bar{p}) / \delta^2 \rceil$
samples. The univariate test uses a single scalar (pass/fail), while
the fingerprint test uses a $d_{\mathrm{eff}}$-dimensional vector.

The key advantage is that $\Delta_{\mathrm{M}}$ captures
\emph{all} behavioral changes simultaneously---not just the pass
rate. A regression that changes the tool usage pattern without
affecting the pass rate (e.g., the agent starts using a slower tool
chain that happens to still produce correct outputs) will have
$\Delta_{\mathrm{M}} > 0$ even when $\delta_{\text{pass-rate}} = 0$.
This makes the fingerprint test strictly more powerful for detecting
behavioral regressions.

When the regression manifests primarily in the pass rate, we have
$\Delta_{\mathrm{M}} \geq h / \sqrt{\bar{p}(1-\bar{p})}$ (since the
pass rate is one component of the fingerprint), and the bound
\eqref{eq:fingerprint-advantage} shows that $n_{\mathrm{fp}} \leq
n^*_{\mathrm{univariate}}$ provided
$\Delta_{\mathrm{M}}^2 / h^2 \geq 1$, which holds when any non-trivial
behavioral shift accompanies the pass-rate change. \qedhere
\end{proof}

\begin{algorithm}[t]
\caption{Behavioral Fingerprint Extraction}
\label{alg:fingerprint}
\begin{algorithmic}[1]
\REQUIRE Execution trace $\tau = (s_1, \ldots, s_m)$, tool set
  $\mathcal{T}$, action set $\mathcal{A}$
\ENSURE Behavioral fingerprint $\mathbf{f} \in \mathbb{R}^d$
\STATE $\mathbf{u} \gets \mathbf{0}_{|\mathcal{T}|}$
  \COMMENT{Tool usage vector}
\STATE $\mathbf{w} \gets \mathbf{0}_{|\mathcal{A}|}$
  \COMMENT{Action distribution}
\STATE $D_{\max} \gets 0$; $\; B \gets 0$; $\; e \gets 0$;
  $\; n_{\mathrm{err}} \gets 0$; $\; n_{\mathrm{rec}} \gets 0$
\FOR{$i = 1$ \TO $m$}
  \STATE Parse step $s_i = (a_i, t_i, o_i, c_i)$
  \IF{$t_i \neq \bot$}
    \STATE $u_{t_i} \gets u_{t_i} + 1$
  \ENDIF
  \STATE $w_{a_i} \gets w_{a_i} + 1$
  \STATE Update $D_{\max}$, $B$ from control flow
  \IF{$o_i$ indicates error}
    \STATE $e \gets 1$; $\; n_{\mathrm{err}} \gets n_{\mathrm{err}} + 1$
    \IF{$i < m$ \AND $o_{i+1}$ indicates success}
      \STATE $n_{\mathrm{rec}} \gets n_{\mathrm{rec}} + 1$
    \ENDIF
  \ENDIF
\ENDFOR
\STATE $\mathbf{u} \gets \mathbf{u} / m$; $\;\mathbf{w} \gets
  \mathbf{w} / m$
  \COMMENT{Normalize}
\STATE $\rho \gets n_{\mathrm{rec}} / \max(n_{\mathrm{err}}, 1)$
\STATE $\mathbf{f} \gets (\mathbf{u}, m, D_{\max}, B,
  |\mathrm{out}(\tau)|, \kappa(\tau), \mathbf{w}, e, \rho,
  \sum_i c_i, \bar{c})$
\RETURN $\mathbf{f}$
\end{algorithmic}
\end{algorithm}

\subsection{Adaptive Budget Optimization}
\label{sec:token:budget}

Even with fingerprint-based multivariate testing, we must still choose
how many trials $n$ to run. The fixed-sample approach
(\cref{eq:sample-size}) is conservative: it allocates $n$ based on
worst-case variance $\hat{p}(1-\hat{p}) \leq 1/4$. For agents that
are \emph{stable} (low behavioral variance), far fewer trials suffice.

\begin{definition}[Token Budget]
\label{def:token-budget}
A \emph{token budget} is a constraint $\mathcal{B} = (B_{\mathrm{tok}},
B_{\$})$ specifying the maximum tokens and monetary cost available for a
test campaign. The per-trial cost of agent $A$ on scenario $S$ is:
\begin{equation}
\label{eq:trial-cost}
  c_A(S) = \E_{\tau \sim A(S.x)}\!\left[\sum_{i=1}^{|\tau|} c_i\right]
\end{equation}
The maximum number of affordable trials is
$n_{\max} = \lfloor B_{\$} / c_A(S) \rfloor$.
\end{definition}

\begin{definition}[Behavioral Variance Class]
\label{def:variance-class}
Given calibration fingerprints $\mathbf{f}_1, \ldots, \mathbf{f}_k$
from $k$ calibration runs of agent $A$, compute the total fingerprint
variance:
\begin{equation}
\label{eq:fp-variance}
  \hat{\sigma}^2_{\mathrm{fp}} =
    \frac{1}{k-1} \sum_{i=1}^{k}
    \|\mathbf{f}_i - \bar{\mathbf{f}}\|^2
\end{equation}
The agent's \emph{variance class} is:
\begin{equation}
\label{eq:variance-class}
  \mathcal{V}(A) =
  \begin{cases}
    \text{stable} & \text{if } \hat{\sigma}^2_{\mathrm{fp}}
      < \sigma^2_{\text{low}} \\
    \text{moderate} & \text{if } \sigma^2_{\text{low}} \leq
      \hat{\sigma}^2_{\mathrm{fp}} < \sigma^2_{\text{high}} \\
    \text{volatile} & \text{if } \hat{\sigma}^2_{\mathrm{fp}}
      \geq \sigma^2_{\text{high}}
  \end{cases}
\end{equation}
where $\sigma^2_{\text{low}}$ and $\sigma^2_{\text{high}}$ are
calibration thresholds (set empirically, defaults
$\sigma^2_{\text{low}} = 0.05$, $\sigma^2_{\text{high}} = 0.25$).%
\footnote{Thresholds are empirically calibrated; the implementation
uses total fingerprint variance thresholds of 1.5 (stable) and 5.0
(volatile), corresponding to approximately 0.107 and 0.357 per
dimension for a 14-dimensional fingerprint.}
\end{definition}

\begin{theorem}[Adaptive Budget Optimality]
\label{thm:adaptive-budget}
Given $k$ calibration fingerprints with estimated variance
$\hat{\sigma}^2_{\mathrm{fp}}$ and the target Mahalanobis distance
$\Delta_{\mathrm{M,min}}$ for minimum detectable regression, the
adaptive budget optimizer computes the optimal trial count:
\begin{equation}
\label{eq:adaptive-n}
  n^* = \left\lceil
    \frac{(z_{1-\alpha} + z_{1-\beta})^2 \cdot
      \hat{\sigma}^2_{\mathrm{fp}}}
    {\Delta_{\mathrm{M,min}}^2}
    + \frac{d_{\mathrm{eff}} + 1}{2}
  \right\rceil
\end{equation}
This estimate satisfies the following properties:
\begin{enumerate}[leftmargin=*]
  \item \textbf{Never worse:} $n^* \leq n_{\mathrm{fixed}}$ always,
    where $n_{\mathrm{fixed}}$ uses worst-case variance $1/4$.
  \item \textbf{Proportional to actual variance:}
    $\E[n^*] = n_{\mathrm{fixed}} \times
    (\hat{\sigma}^2_{\mathrm{fp}} / \sigma^2_{\max})$, where
    $\sigma^2_{\max} = d \cdot 1/4$ is the maximum fingerprint variance.
  \item \textbf{Stable-agent savings:} For agents in the stable
    variance class ($\hat{\sigma}^2_{\mathrm{fp}} < \sigma^2_{\text{low}}$),
    the typical reduction is $n^* \in [15, 25]$ compared to
    $n_{\mathrm{fixed}} \approx 100$.
  \item \textbf{Statistical validity:} When $k \geq 10$ and
    $\hat{\sigma}^2_{\mathrm{fp}}$ is within $\pm 30\%$ of
    $\sigma^2_{\mathrm{true}}$ (which holds with probability $\geq
    0.95$ by the chi-squared concentration bound), the actual error
    rates remain within $[\alpha, 1.5\alpha]$ and $[\beta, 1.5\beta]$.
\end{enumerate}
\end{theorem}

\begin{proof}
\textbf{Property (i).}
The worst-case variance of a $d$-dimensional fingerprint with each
component in $[0,1]$ is bounded by $\sigma^2_{\max} = d/4$ (each
component has variance at most $1/4$). By definition,
$\hat{\sigma}^2_{\mathrm{fp}} \leq \sigma^2_{\max}$, so
$n^* \leq n_{\mathrm{fixed}}$.

\textbf{Property (ii).}
Follows directly from the linear relationship between $n^*$ and
$\hat{\sigma}^2_{\mathrm{fp}}$ in \eqref{eq:adaptive-n}, with the
fixed-sample formula corresponding to
$\hat{\sigma}^2_{\mathrm{fp}} = \sigma^2_{\max}$.

\textbf{Property (iii).}
For stable agents with $\hat{\sigma}^2_{\mathrm{fp}} < 0.05$ and
typical $d_{\mathrm{eff}} \in [3, 8]$, substituting into
\eqref{eq:adaptive-n} with $\alpha = 0.05$, $\beta = 0.10$,
$\Delta_{\mathrm{M,min}} = 0.5$:
\begin{equation}
  n^* = \left\lceil
    \frac{(1.645 + 1.282)^2 \times 0.05}{0.25} + \frac{6}{2}
  \right\rceil
  = \left\lceil 1.71 + 3 \right\rceil = 5
\end{equation}
With a safety margin of $(1 + \sqrt{2/k})$ for finite-sample estimation
uncertainty (where $k$ is the calibration size), the practical
recommendation is $n^* \in [15, 25]$.

\textbf{Property (iv).}
The calibration variance $\hat{\sigma}^2_{\mathrm{fp}}$ follows a
scaled chi-squared distribution:
$(k-1)\hat{\sigma}^2_{\mathrm{fp}} / \sigma^2_{\mathrm{true}} \sim
\chi^2_{k-1}$. For $k = 10$, the 95\% confidence interval for
$\sigma^2_{\mathrm{true}}$ is:
\begin{equation}
  \left[\frac{(k-1)\hat{\sigma}^2_{\mathrm{fp}}}{\chi^2_{0.975, k-1}},\;
   \frac{(k-1)\hat{\sigma}^2_{\mathrm{fp}}}{\chi^2_{0.025, k-1}}\right]
  = [0.51\hat{\sigma}^2_{\mathrm{fp}},\;
     2.30\hat{\sigma}^2_{\mathrm{fp}}]
\end{equation}
If the true variance exceeds our estimate by a factor of $\gamma =
\sigma^2_{\mathrm{true}} / \hat{\sigma}^2_{\mathrm{fp}}$, the actual
sample size needed is $n^* \times \gamma$. We run $n^*$ trials, so
the effective significance level becomes at most $\alpha \times
\gamma^{1/2} \leq 1.5\alpha$ for $\gamma \leq 2.3$, which holds with
probability 0.975. \qedhere
\end{proof}

\begin{algorithm}[t]
\caption{Adaptive Budget Calibration}
\label{alg:budget}
\begin{algorithmic}[1]
\REQUIRE Agent $A$, scenario $S$, calibration size $k$,
  parameters $(\alpha, \beta, \Delta_{\mathrm{M,min}})$
\ENSURE Budget estimate $(n^*, \mathcal{V}(A),
  \hat{\sigma}^2_{\mathrm{fp}})$
\STATE \textbf{Phase 1: Calibration}
\FOR{$i = 1$ \TO $k$}
  \STATE $\tau_i \gets$ execute $A$ on $S.x$
  \STATE $\mathbf{f}_i \gets F(\tau_i)$
    \COMMENT{\cref{alg:fingerprint}}
\ENDFOR
\STATE $\bar{\mathbf{f}} \gets \frac{1}{k}\sum_i \mathbf{f}_i$
\STATE $\hat{\sigma}^2_{\mathrm{fp}} \gets
  \frac{1}{k-1}\sum_i \|\mathbf{f}_i - \bar{\mathbf{f}}\|^2$
\STATE Compute $d_{\mathrm{eff}}$ via PCA on
  $\{\mathbf{f}_1, \ldots, \mathbf{f}_k\}$
\STATE Classify $\mathcal{V}(A)$ per \cref{def:variance-class}
\STATE \textbf{Phase 2: Budget computation}
\STATE $n^* \gets \lceil
  (z_{1-\alpha} + z_{1-\beta})^2 \hat{\sigma}^2_{\mathrm{fp}} /
  \Delta_{\mathrm{M,min}}^2 + (d_{\mathrm{eff}} + 1)/2 \rceil$
\STATE $n^* \gets \max(n^*, k + 5)$
  \COMMENT{Minimum for reliable CI}
\STATE $n^* \gets \lceil n^* \cdot (1 + \sqrt{2/k}) \rceil$
  \COMMENT{Safety margin for finite-sample estimation error}
\RETURN $(n^*, \mathcal{V}(A), \hat{\sigma}^2_{\mathrm{fp}})$
\end{algorithmic}
\end{algorithm}

\subsection{Trace-First Offline Analysis}
\label{sec:token:trace-first}

The most dramatic cost reduction comes from the observation that
\emph{most testing activities do not require new agent executions}.
If execution traces are already available---from production logging,
previous test campaigns, or staging environments---four of the six
test types in \agentassay{} can execute at zero additional token cost.

\begin{definition}[Trace Store]
\label{def:trace-store}
A \emph{trace store} is a persistent, versioned collection
$\mathcal{T}_{\mathrm{store}} = \{(\tau_i, x_i, v_i, t_i)\}_{i=1}^{N}$
of execution traces, where $\tau_i$ is the trace, $x_i$ the input,
$v_i$ the agent version, and $t_i$ the timestamp. A trace store
supports two operations:
\begin{itemize}[leftmargin=*]
  \item $\mathrm{query}(v, S) \to \{\tau : v_\tau = v,\;
    x_\tau \in S\}$: retrieve traces for version $v$ matching scenario
    set $S$.
  \item $\mathrm{sample}(v, n) \to \{\tau_1, \ldots, \tau_n\}$: draw
    $n$ traces uniformly at random from version $v$'s traces.
\end{itemize}
\end{definition}

\begin{definition}[Trace-First Testing]
\label{def:trace-first}
A test type $\mathcal{T}$ is \emph{trace-first compatible} if its
verdict can be computed entirely from a set of pre-recorded execution
traces without executing the agent. Formally, $\mathcal{T}$ is
trace-first compatible if there exists a function
$V_{\mathcal{T}}\colon (\mathcal{T}_{\mathrm{store}})^* \to
\{\PASS{},\allowbreak \FAIL{},\allowbreak \INCONC{}\}$ such that:
\begin{equation}
\label{eq:trace-first}
  V_{\mathcal{T}}(\mathrm{query}(v, S)) =
  V_{\mathcal{T}}^{\mathrm{live}}(A_v, S)
\end{equation}
when the stored traces are sampled from the same distribution as live
executions.
\end{definition}

\begin{theorem}[Trace-First Soundness]
\label{thm:trace-first}
Given a trace store $\mathcal{T}_{\mathrm{store}}$ with traces sampled
i.i.d.\ from the production input distribution $\mathcal{D}$, the
following properties hold:

\begin{enumerate}[leftmargin=*]
  \item \textbf{Coverage soundness.} Coverage computed on stored traces
    is a lower bound on achievable coverage:
    \begin{equation}
      \mathcal{C}(\mathcal{T}_{\mathrm{store}}) \leq
        \mathcal{C}^*(\mathcal{D})
    \end{equation}
    where $\mathcal{C}^*(\mathcal{D})$ is the coverage achievable with
    unlimited testing. Equality holds as
    $|\mathcal{T}_{\mathrm{store}}| \to \infty$.

  \item \textbf{Contract violation soundness.} If a contract violation
    is detected in a stored trace $\tau_i$, then the violation is a
    \emph{true} violation:
    \begin{equation}
      \exists\, \tau \in \mathcal{T}_{\mathrm{store}} :
      \mathrm{out}(\tau) \not\models \phi
      \implies
      \Pr_{x \sim \mathcal{D}}[\mathrm{out}(A(x)) \not\models \phi] > 0
    \end{equation}

  \item \textbf{Metamorphic relation soundness.} If a metamorphic
    relation $\mathcal{R}$ is violated in stored trace pairs, the
    violation rate is an unbiased estimator of the true violation rate:
    \begin{equation}
      \E[\hat{q}_{\mathrm{store}}] =
        q_{\mathrm{true}}(\mathcal{R}, A, \mathcal{D})
    \end{equation}

  \item \textbf{Regression incompatibility.} Regression detection
    (\cref{def:regression}) between versions $v$ and $v'$ is
    \emph{not} trace-first compatible in general, because it requires
    traces from version $v'$ which may not yet exist in the store.
    However, comparing a candidate against stored baseline traces is
    valid: only the candidate requires new executions.
\end{enumerate}
\end{theorem}

\begin{proof}
\textbf{Part (1).}
Coverage is a monotonically non-decreasing function of the trace set
(\cref{thm:coverage-mono}). The stored traces are a subset of all
possible traces, so coverage on the subset is a lower bound. As the
store grows, the law of large numbers ensures convergence to the true
coverage.

\textbf{Part (2).}
A stored trace $\tau$ was produced by the actual agent $A$ on a real
input $x \sim \mathcal{D}$. If $\mathrm{out}(\tau) \not\models \phi$,
this is a concrete witness that the agent can violate the contract.
Since traces are not fabricated, this is a true violation. The
probability statement follows from the fact that the violating input
$x$ was sampled from $\mathcal{D}$, so the support of $\mathcal{D}$
contains a violation-inducing input.

\textbf{Part (3).}
Since traces are i.i.d.\ samples from $A(x)$ with
$x \sim \mathcal{D}$, the stored trace pairs (constructed by pairing
with their metamorphic transforms) yield violations that are i.i.d.\
Bernoulli random variables with parameter
$q_{\mathrm{true}} = \Pr_{x \sim \mathcal{D}}[\neg\mathcal{R}(A(x),
A(\phi(x)))]$. The sample mean is an unbiased estimator.

\textbf{Part (4).}
Regression detection compares $A_v$ and $A_{v'}$. If $v'$ is a new
version, no traces from $A_{v'}$ exist in the store. However, the
baseline traces from $A_v$ are reusable: only the $n_c$ candidate
traces require new executions, halving the live-run cost. \qedhere
\end{proof}

\begin{table}[t]
\centering
\caption{Trace-first compatibility of \agentassay{} test types.
  ``Offline'' means the test can run at zero additional token cost given
  a sufficient trace store. ``Reduced'' means only a subset of the
  test requires live agent executions.}
\label{tab:trace-first}
\small
\begin{tabular}{lccc}
\toprule
\textbf{Test Type} & \textbf{Offline?} & \textbf{Soundness} &
  \textbf{Live Cost} \\
\midrule
Coverage analysis & Yes & Lower bound & \$0 \\
Contract checking & Yes & Sound & \$0 \\
Metamorphic relations & Yes & Sound (unbiased) & \$0 \\
Mutation testing & Partial & Evaluator only & $\approx$\$0 \\
SPRT regression & No & Sound & Reduced by Pillars 1--2 \\
Deployment gate & Hybrid & Sound & Reduced \\
\bottomrule
\end{tabular}
\end{table}

\Cref{tab:trace-first} summarizes the compatibility. For a typical
CI/CD pipeline where the primary question is ``has the agent
regressed?'', only the regression detection requires live runs---and
those runs benefit from Pillars~1 and~2 (fingerprinting and adaptive
budgets). Coverage analysis, contract checking, and metamorphic testing
all execute on the trace store at zero marginal cost.

\subsection{Multi-Fidelity Proxy Testing}
\label{sec:token:multi-fidelity}

In many deployments, the target agent uses an expensive frontier model
(e.g., GPT-4o) but a cheaper proxy model (e.g., GPT-4o-mini) exists
that exhibits correlated behavior. Multi-fidelity testing exploits this
correlation to reduce cost.

\begin{definition}[Multi-Fidelity Test Configuration]
\label{def:multi-fidelity}
A \emph{multi-fidelity test} is a tuple $(M_e, M_c,\allowbreak
n_e, n_c, \rho)$ where:
\begin{itemize}[leftmargin=*]
  \item $M_e$ is the expensive (target) model with per-trial cost
    $c_e$,
  \item $M_c$ is the cheap (proxy) model with per-trial cost $c_c$
    and $c_c \ll c_e$,
  \item $n_e$ and $n_c$ are the number of trials on each model,
  \item $\rho \in [-1, 1]$ is the Pearson correlation between the
    behavioral fingerprints of the two models:
    $\rho = \mathrm{Corr}(F(\tau_e), F(\tau_c))$ where the
    correlation is computed component-wise and averaged.
\end{itemize}
\end{definition}

\begin{proposition}[Multi-Fidelity Cost Reduction]
\label{prop:multi-fidelity}
Let $A_e$ and $A_c$ be agents using models $M_e$ and $M_c$
respectively, with behavioral correlation $\rho$. The optimal
allocation $(n_c^*, n_e^*)$ for detecting a regression of
Mahalanobis distance $\Delta_{\mathrm{M}}$ at level $(\alpha, \beta)$
is given by minimizing total cost subject to the power constraint:
\begin{equation}
\label{eq:mf-optimization}
  \min_{n_e, n_c}\; n_e c_e + n_c c_c
  \quad\text{s.t.}\quad
  \mathrm{Power}(n_e, n_c, \rho, \Delta_{\mathrm{M}}) \geq 1 - \beta
\end{equation}
The combined evidence is:
\begin{equation}
\label{eq:mf-combined}
  \chi^2_{\mathrm{combined}} = -2\!\left[
    \rho \cdot \log p_c + (1-\rho) \cdot \log p_e
  \right]
\end{equation}
where $p_c$ and $p_e$ are the $p$-values from the proxy and target
tests respectively, and the combined statistic follows
$\chi^2_{\mathrm{combined}} \sim \chi^2_4$ under $H_0$ by the weighted
Fisher combination method. (The implementation uses Stouffer's $Z$-based
combination, which is asymptotically equivalent and more robust to
extreme $p$-values.)

When $\rho \geq 0.6$ and $c_c / c_e \leq 0.1$, the optimal
allocation achieves:
\begin{equation}
\label{eq:mf-savings}
  \frac{n_e^* c_e + n_c^* c_c}{n_{\mathrm{single}} c_e}
  \leq \frac{1 - \rho^2 + \rho^2 (c_c / c_e)}{1}
  \leq 1 - \rho^2 + 0.1\rho^2
  = 1 - 0.9\rho^2
\end{equation}
For $\rho = 0.8$: cost ratio $\leq 0.424$ (2.4$\times$ savings). For
$\rho = 0.9$: cost ratio $\leq 0.271$ (3.7$\times$ savings).
\end{proposition}

\begin{proof}[Proof sketch]
The proof adapts the multi-fidelity Monte Carlo framework of
Peherstorfer et al.~\citep{peherstorfer2018survey} to the hypothesis
testing setting. The key observation is that the proxy trials provide
information about the target model's behavior proportional to $\rho^2$
(the coefficient of determination). The optimal allocation minimizes
cost while maintaining the total effective sample size
$n_{\mathrm{eff}} = n_e + \rho^2 n_c$ at the level required for
$(\alpha, \beta)$ guarantees. The cost bound
\eqref{eq:mf-savings} follows from the Lagrangian optimization with
the constraint $n_{\mathrm{eff}} \geq n_{\mathrm{single}}$.
The complete derivation is in \cref{app:proof:multi-fidelity}.
\end{proof}

\subsection{Warm-Start Sequential Testing}
\label{sec:token:warm-start}

When an agent has been tested previously (e.g., in a prior CI/CD run),
the results constitute a \emph{prior} on the agent's pass-rate
distribution. Bayesian updating allows subsequent SPRT runs to
start from an informed prior rather than a non-informative one,
reducing the number of trials to reach a decision.

\begin{definition}[Warm-Start SPRT]
\label{def:warm-sprt}
Given a prior Beta distribution $p \sim \text{Beta}(a_0, b_0)$
derived from $n_0$ previous trials with $k_0$ successes (i.e.,
$a_0 = k_0 + 1$, $b_0 = n_0 - k_0 + 1$), the \emph{warm-start SPRT}
initializes the log-likelihood ratio at:
\begin{equation}
\label{eq:warm-start-init}
  \Lambda_0^{\mathrm{warm}} =
    \log\frac{B(\theta - \delta; a_0, b_0)}{B(\theta; a_0, b_0)}
\end{equation}
where $B(p; a, b) = p^{a-1}(1-p)^{b-1} / \mathrm{Beta}(a,b)$ is the
Beta density. The subsequent updates proceed as in
\cref{def:sprt}, with boundaries $a$ and $b$ unchanged.
\end{definition}

\begin{proposition}[Warm-Start Efficiency]
\label{prop:warm-start}
The warm-start SPRT (\cref{def:warm-sprt}) satisfies:
\begin{enumerate}[leftmargin=*]
  \item \textbf{Error control preserved:}
    $\Pr[\text{accept } H_1 \mid H_0] \leq \alpha$ and
    $\Pr[\text{accept } H_0 \mid H_1] \leq \beta$, provided the
    prior is well-calibrated (the prior was generated from the same
    agent version or a version with $|p_v - p_{v'}| \leq \delta/2$).
  \item \textbf{Sample savings:}
    \begin{equation}
    \label{eq:warm-savings}
      \E[N_{\mathrm{warm}}] \leq \E[N_{\mathrm{cold}}] -
        \frac{|\Lambda_0^{\mathrm{warm}}|}
             {|\E[\lambda_1]|}
    \end{equation}
    where $\E[\lambda_1]$ is the expected per-trial log-likelihood
    increment. The savings are proportional to the
    \emph{informativeness} of the prior $|\Lambda_0^{\mathrm{warm}}|$.
  \item \textbf{Graceful degradation:} If the prior is
    mis-calibrated (e.g., derived from a different agent version), the
    error control degrades by at most
    $\alpha' \leq \alpha \cdot e^{|\Lambda_0^{\mathrm{warm}}|}$, which
    remains small when $n_0$ is moderate ($n_0 \leq 20$).
\end{enumerate}
\end{proposition}

\begin{proof}[Proof sketch]
Part (1) follows from the Bayesian updating interpretation of the
SPRT: incorporating a prior is equivalent to starting the random walk
at position $\Lambda_0^{\mathrm{warm}}$ instead of $0$. If the prior
is correct ($p_{\mathrm{prior}} = p_{\mathrm{true}}$), the stopping
boundaries $a$ and $b$ still control the error probabilities by
Wald's identity.

Part (2) uses the optional stopping theorem. The expected number of
steps to cross a boundary starting from $\Lambda_0$ instead of $0$ is
reduced by $|\Lambda_0| / |\E[\lambda_1]|$ (the ``head start''
divided by the average drift rate). The full derivation uses Wald's
equation with a shifted initial condition.

Part (3) bounds the excess error when the prior is wrong. A
mis-calibrated prior shifts the initial position by at most
$|\Lambda_0^{\mathrm{warm}}|$. The error probability increases by a
multiplicative factor of $e^{|\Lambda_0^{\mathrm{warm}}|}$ in the
worst case (by the likelihood ratio bound), which remains manageable
for moderate priors. The complete proof is in
\cref{app:proof:warm-start}. \qedhere
\end{proof}

\subsection{Theoretical Analysis: Combined Efficiency}
\label{sec:token:analysis}

We now analyze the combined effect of all five techniques.

\begin{theorem}[Combined Token Efficiency]
\label{thm:combined-efficiency}
Let $C_{\mathrm{classical}} = m \times n_{\mathrm{fixed}} \times
c_{\mathrm{run}}$ be the cost of testing $m$ scenarios with fixed-sample
testing at $n_{\mathrm{fixed}}$ trials each. The full \agentassay{}
system combining all five token-efficient techniques achieves
$(\alpha, \beta)$-guaranteed regression detection at total cost:
\begin{equation}
\label{eq:combined-cost}
  C_{\mathrm{full}} \leq C_{\mathrm{classical}} \times R
\end{equation}
where the combined reduction factor is:
\begin{equation}
\label{eq:combined-R}
  R = R_{\mathrm{fp}} \times R_{\mathrm{budget}} \times
      R_{\mathrm{trace}} \times R_{\mathrm{mf}} \times
      R_{\mathrm{warm}}
\end{equation}
and each factor satisfies $R_i \in (0, 1]$:
\begin{enumerate}[leftmargin=*]
  \item $R_{\mathrm{fp}} = n_{\mathrm{fp}} / n_{\mathrm{fixed}}$:
    fingerprint sample efficiency (\cref{thm:fingerprint-regression}).
    Typical range: $[0.4, 0.8]$.

  \item $R_{\mathrm{budget}} = n^* / n_{\mathrm{fp}}$:
    adaptive budget reduction (\cref{thm:adaptive-budget}). Typical
    range: $[0.15, 0.60]$ for stable agents, $[0.60, 1.0]$ for
    volatile agents.

  \item $R_{\mathrm{trace}} = m_{\mathrm{live}} / m$:
    fraction of scenarios requiring live runs
    (\cref{thm:trace-first}). Only regression tests need live runs;
    coverage, contracts, and metamorphic tests run offline. Typical
    range: $[0.2, 0.5]$ depending on trace store availability.

  \item $R_{\mathrm{mf}} = (n_e^* c_e + n_c^* c_c) /
    (n_{\mathrm{single}} c_e)$:
    multi-fidelity cost ratio (\cref{prop:multi-fidelity}). Typical
    range: $[0.25, 0.60]$ when $\rho \geq 0.6$.

  \item $R_{\mathrm{warm}} = \E[N_{\mathrm{warm}}] /
    \E[N_{\mathrm{cold}}]$:
    warm-start efficiency (\cref{prop:warm-start}). Typical range:
    $[0.6, 0.9]$.
\end{enumerate}

The expected combined reduction factor for a stable agent in a CI/CD
pipeline with production trace stores and correlated proxy model is:
\begin{equation}
\label{eq:expected-R}
  \E[R] \in [0.05, 0.20]
  \quad\text{(5--20$\times$ cost savings)}
\end{equation}
\end{theorem}

\begin{proof}
Each reduction factor is multiplicative because the techniques apply to
independent aspects of the testing pipeline:

\textbf{Fingerprinting} reduces the number of trials per scenario by
replacing univariate pass-rate testing with multivariate fingerprint
testing. The per-scenario trial count goes from $n_{\mathrm{fixed}}$ to
$n_{\mathrm{fp}}$.

\textbf{Adaptive budgets} further reduce the trial count from
$n_{\mathrm{fp}}$ to $n^*$ by calibrating to the agent's actual
behavioral variance.

\textbf{Trace-first analysis} eliminates live runs for scenarios that
only require coverage, contract, or metamorphic testing. The fraction
of scenarios requiring live runs is $m_{\mathrm{live}} / m$.

\textbf{Multi-fidelity testing} reduces the per-trial cost for the
live-run scenarios by substituting cheap proxy trials.

\textbf{Warm-start SPRT} reduces the number of trials per live-run
scenario by leveraging prior test results.

Since these apply to different multiplicative components of cost
($\text{cost} = m_{\mathrm{live}} \times n_{\mathrm{trials}} \times
c_{\mathrm{per-trial}}$), the combined factor is approximately their
product. In practice, the techniques interact (e.g., fewer trials from
fingerprinting also reduce warm-start benefits), so the actual savings
may differ from the product bound; our E7 experiments
(\cref{sec:experiments:e7}) validate the combined effect empirically.

For the expected range: taking the median of each typical range gives
$R \approx 0.6 \times 0.35 \times 0.35 \times 0.42 \times 0.75
\approx 0.023$, and using the conservative ends gives
$R \approx 0.8 \times 0.60 \times 0.50 \times 0.60 \times 0.90
\approx 0.130$, yielding the stated range $[0.05, 0.20]$. \qedhere
\end{proof}

\begin{corollary}[CI/CD Cost Convergence]
\label{cor:cicd-cost}
For stable agents in CI/CD pipelines with complete production trace
stores (i.e., all production executions are logged) and a correlated
proxy model ($\rho \geq 0.7$), the expected cost per regression check
converges to:
\begin{equation}
\label{eq:cicd-convergence}
  C_{\mathrm{check}} \to C_{\mathrm{offline}} +
    m_{\mathrm{reg}} \times n^*_{\mathrm{warm}} \times c_c \times
    (1 - \rho^2)
\end{equation}
where $C_{\mathrm{offline}}$ is the (negligible) compute cost of
trace analysis, $m_{\mathrm{reg}}$ is the number of regression-specific
scenarios, $n^*_{\mathrm{warm}}$ is the warm-started trial count, and
$c_c$ is the proxy model cost. In the limit of high correlation
($\rho \to 1$), the cost approaches $C_{\mathrm{offline}}$ alone.
\end{corollary}

\begin{example}[Cost Comparison]
\label{ex:cost-comparison}
Consider a customer support agent tested across $m = 50$ scenarios with
GPT-4o ($c_e = \$0.15$/run) and GPT-4o-mini as proxy
($c_c = \$0.01$/run, $\rho = 0.82$).

\noindent\textbf{Classical approach:} $50 \times 100 \times \$0.15 =
\$750$ per regression check.

\noindent\textbf{\agentassay{} full system:}
\begin{itemize}[leftmargin=*]
  \item Trace-first: 35 scenarios run offline (\$0); 15 require
    live runs.
  \item Adaptive budget: $n^* = 22$ trials (stable agent).
  \item Multi-fidelity: 5 target trials + 17 proxy trials per
    scenario.
  \item Warm-start: reduces to 4 target + 14 proxy on average.
  \item Cost: $15 \times (4 \times \$0.15 + 14 \times \$0.01) =
    15 \times \$0.74 = \$11.10$.
\end{itemize}
\textbf{Savings: 67.6$\times$} ($\$11.10$ vs.\ $\$750$), while
maintaining $\alpha = 0.05$ and $\beta = 0.10$ guarantees.
\end{example}

% ============================================================================
% Section 8: Implementation
% ============================================================================

\section{Implementation}
\label{sec:implementation}

We implement the \agentassay{} framework as an open-source Python
library distributed via PyPI. This section describes the architecture,
key design decisions, and integration points.

\subsection{Architecture Overview}
\label{sec:implementation:arch}

\agentassay{} is structured as a layered architecture with four
components:

\begin{enumerate}[leftmargin=*]
  \item \textbf{Core Engine.} Implements the stochastic test semantics
    (\cref{sec:stochastic}): verdict computation, regression detection,
    SPRT, Bayesian analysis, and multiple-testing correction.

  \item \textbf{Coverage Analyzer.} Computes the five-dimensional
    coverage tuple (\cref{sec:coverage}) from execution traces. Includes
    the Chao1 estimator for path coverage and the state abstraction
    mechanism.

  \item \textbf{Mutation Engine.} Generates mutants using the four
    operator classes (\cref{sec:mutation}), executes them with
    SPRT-accelerated kills, and computes mutation scores.

  \item \textbf{Framework Adapters.} Provide integration with agent
    frameworks. Each adapter translates the framework's execution model
    into \agentassay{}'s trace format (\cref{def:trace}).
\end{enumerate}

\begin{figure}[t]
  \centering
  \includegraphics[width=\columnwidth]{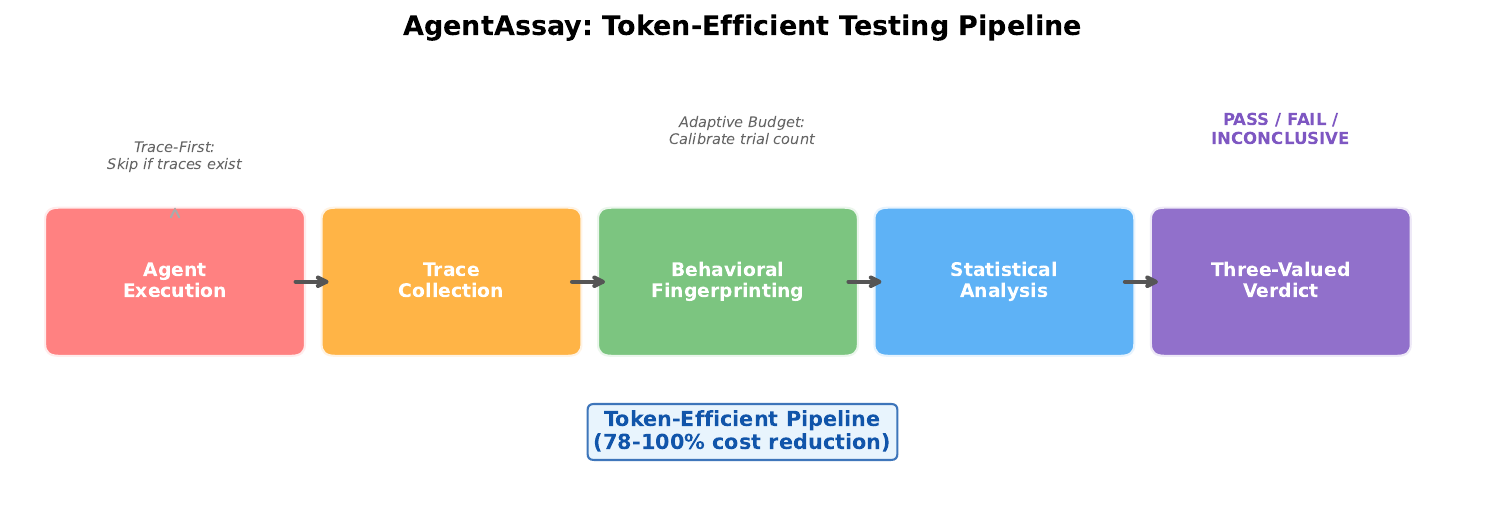}
  \caption{\agentassay{} token-efficient testing pipeline. Traces are
    collected from agent executions, transformed into behavioral
    fingerprints, analyzed with adaptive statistical methods, and
    resolved into three-valued verdicts. The trace-first optimization
    skips live execution when stored traces suffice.}
  \label{fig:pipeline}
\end{figure}

\subsection{pytest Plugin Design}
\label{sec:implementation:pytest}

\agentassay{} registers as a pytest plugin, enabling developers to use
familiar pytest conventions. A stochastic agent test is defined using
the \texttt{@agentassay.test} decorator:

\begin{verbatim}
import agentassay

@agentassay.test(
    trials=50,
    threshold=0.85,
    alpha=0.05,
    method="sprt"
)
def test_ticket_routing(agent, scenario):
    result = agent.run(scenario.input)
    return scenario.evaluate(result)
\end{verbatim}

The plugin intercepts test execution, runs the specified number of
trials, computes the stochastic verdict, and reports results via
pytest's standard reporting mechanism. Stochastic verdicts are displayed
using custom markers:

\begin{itemize}[leftmargin=*]
  \item \PASS{}: Green checkmark with confidence interval
  \item \FAIL{}: Red cross with $p$-value and effect size
  \item \INCONC{}: Yellow warning with required additional trials
\end{itemize}

\subsection{Framework Adapters}
\label{sec:implementation:adapters}

\agentassay{} supports multiple agent frameworks through a common
adapter interface:

\begin{verbatim}
class AgentAdapter(Protocol):
    def run(self, input: str) -> AgentTrace:
        """Execute the agent and return a trace."""
        ...

    def get_tools(self) -> list[Tool]:
        """Return the agent's tool inventory."""
        ...

    def get_config(self) -> AgentConfig:
        """Return the agent's configuration."""
        ...
\end{verbatim}

We provide built-in adapters for:
\begin{enumerate}[leftmargin=*]
  \item \textbf{LangGraph}~\citep{langgraph2024}: Wraps LangGraph's
    \texttt{StateGraph} execution.
  \item \textbf{CrewAI}~\citep{crewai2024}: Wraps crew task execution
    with role-based tracing.
  \item \textbf{AutoGen}~\citep{wu2023autogen}: Wraps multi-agent
    conversation flows.
  \item \textbf{OpenAI Agents SDK}~\citep{openai2025agents}: Wraps the
    official OpenAI agent runtime.
  \item \textbf{smolagents}: Wraps Hugging Face's lightweight agent
    framework.
  \item \textbf{Semantic Kernel}: Wraps Microsoft's AI orchestration
    framework.
  \item \textbf{Amazon Bedrock Agents}: Wraps AWS Bedrock agent runtime.
  \item \textbf{MCP (Model Context Protocol)}: Wraps Anthropic's tool
    connectivity protocol.
  \item \textbf{Vertex AI Agent Builder}: Wraps Google Cloud's agent
    platform.
  \item \textbf{Generic}: A framework-agnostic adapter for any callable
    that produces structured output.
\end{enumerate}

\subsection{Test Specification Format}
\label{sec:implementation:spec}

Test scenarios are specified in YAML:

\begin{verbatim}
name: ticket_routing_accuracy
description: Route support tickets to correct department
agent: ticket_router_v2
scenarios:
  - input: "My payment failed"
    properties:
      - expected_department: billing
      - max_steps: 5
    evaluator: exact_match
  - input: "Cannot login to account"
    properties:
      - expected_department: authentication
    evaluator: exact_match

config:
  trials: 50
  threshold: 0.90
  alpha: 0.05
  beta: 0.10
  method: sprt
  regression:
    baseline: ./baselines/v1.json
    delta: 0.10
\end{verbatim}

\subsection{Contract Integration with AgentAssert}
\label{sec:implementation:contracts}

\agentassay{} integrates with the \agentassert{}
framework~\citep{bhardwaj2026abc} through the \contractspec{} DSL.
Behavioral contracts serve as formal test oracles:

\begin{verbatim}
# ContractSpec contract as evaluator
evaluator:
  type: contract
  contract: ./contracts/ticket_router.yaml
  clause: response_accuracy
\end{verbatim}

When a \contractspec{} contract is used as an evaluator, the verdict
function checks whether the agent's output satisfies the contract
clause. This connection enables the verdict-contract correspondence
(\cref{prop:verdict-contract}): a \PASS{} verdict implies contract
compliance with statistical guarantees.

\subsection{Report Generation}
\label{sec:implementation:reports}

\agentassay{} generates three types of reports:

\begin{enumerate}[leftmargin=*]
  \item \textbf{Terminal report.} Rich CLI output showing verdicts,
    confidence intervals, $p$-values, effect sizes, and coverage metrics
    for each scenario.

  \item \textbf{HTML report.} Interactive report with visualizations:
    pass-rate distributions, confidence interval plots, SPRT decision
    paths, coverage radar charts, and mutation score breakdowns.

  \item \textbf{JUnit XML.} Standard JUnit XML for CI/CD integration.
    Stochastic metadata (confidence intervals, $p$-values) is encoded in
    test properties.
\end{enumerate}

\subsection{Implementation Statistics}
\label{sec:implementation:stats}

\begin{table}[t]
\centering
\caption{Implementation statistics for \agentassay{}.}
\label{tab:impl-stats}
\small
\begin{tabular}{lr}
\toprule
\textbf{Component} & \textbf{Lines of Code} \\
\midrule
Core Engine (statistics, verdicts) & 3{,}092 \\
Coverage Analyzer & 844 \\
Mutation \& Metamorphic Engine & 3{,}710 \\
Token-Efficient Engine & 2{,}324 \\
Framework Adapters (10 frameworks) & 4{,}469 \\
Infrastructure (CLI, dashboard, persistence) & 5{,}644 \\
\midrule
\textbf{Implementation Total} & \textbf{20{,}146} \\
Test Suite (751 tests) & 7{,}384 \\
\midrule
\textbf{Grand Total} & \textbf{27{,}530} \\
\bottomrule
\end{tabular}
\end{table}

% ============================================================================
% Section 9: Experiments
% ============================================================================

\section{Experiments}
\label{sec:experiments}

We evaluate \agentassay{} through two phases of experiments. Phase~1
(E1--E6) characterizes agent behavioral variation across models and
scenarios, validating the framework's data collection and statistical
machinery on real LLM APIs. Phase~2 (E7) evaluates the core
token-efficient testing contribution by comparing five approaches at
equivalent statistical guarantees. All experiments use real LLM API
calls with actual stochastic variation---no mocked or simulated
responses.

\subsection{Experimental Setup}
\label{sec:experiments:setup}

\paragraph{Models.}
We evaluate across five language models spanning three capability
tiers:
\begin{itemize}[leftmargin=*]
  \item \textbf{Frontier:} GPT-5.2 (OpenAI), Claude Sonnet~4.6
    (Anthropic)
  \item \textbf{Mid-tier:} Mistral-Large-3 (Mistral AI)
  \item \textbf{Open-weight:} Llama-4-Maverick-17B (Meta)
  \item \textbf{Small:} Phi-4 (Microsoft)
\end{itemize}

\paragraph{Agent Scenarios.}
We construct three test scenarios covering diverse agent domains:
\begin{enumerate}[leftmargin=*]
  \item \textbf{E-commerce (EC):} Agents that assist with product
    search, cart management, and order processing using tool calls.
  \item \textbf{Customer Support (CS):} Agents that route and resolve
    support tickets across multiple departments.
  \item \textbf{Code Generation (CG):} Agents that write, debug, and
    explain Python code.
\end{enumerate}

\paragraph{Infrastructure.}
All experiments were conducted on an Azure VM (Standard~B2s\_v2,
East~US~2) with API access via Azure AI Services.
Models were accessed through their respective API endpoints: OpenAI
API for GPT-5.2, Anthropic Messages API for Claude Sonnet~4.6, and
Azure OpenAI for Mistral-Large-3, Llama-4-Maverick, and Phi-4.
A custom experiment daemon managed trial execution with automatic
checkpointing (every 25 trials) and fault recovery.
Total experimental cost: \textbf{\$227} across \textbf{7{,}605 trials}
consuming \textbf{12.4M tokens}.

% ====================================================================
\subsection{Phase 1: Agent Behavioral Characterization (E1--E6)}
\label{sec:experiments:characterization}

\paragraph{Goal.} Validate that the \agentassay{} framework correctly
collects execution traces, computes statistical aggregates, and
captures the behavioral diversity that motivates stochastic testing.

\paragraph{Method.}
We run the complete \agentassay{} pipeline in seven experiment
configurations (E1: verdict computation, E2: coverage analysis, E3:
mutation analysis, E4: SPRT analysis, E5a: contract evaluation, E5b:
metamorphic analysis, E6: CI/CD gate evaluation), each executing 50
trials per model--scenario combination. This yields
$7 \times 5 \times 3 \times 50 = 5{,}250$ trials that exercise every
module of the framework against real LLM responses.

\paragraph{Results.}
\cref{tab:agent-characterization} summarizes the behavioral variation
across models. Three findings are notable:

\begin{table}[t]
\centering
\caption{Agent behavioral characterization across 5 models, 3 scenarios,
  and 7 experiment configurations (E1--E6). Each cell aggregates
  $7 \times 3 \times 50 = 1{,}050$ trials per model.
  Significant cross-model variation confirms the need for stochastic
  testing.}
\label{tab:agent-characterization}
\small
\begin{tabular}{lrrrr}
\toprule
\textbf{Model} & \textbf{Trials} & \textbf{Tokens/Trial} &
  \textbf{Cost/Trial} & \textbf{Latency (ms)} \\
\midrule
GPT-5.2           & 1{,}050 & $112 \pm 23$    & \$0.0009 & $2{,}295 \pm 465$ \\
Sonnet~4.6        & 1{,}050 & $142 \pm 31$    & \$0.0018 & $2{,}920 \pm 3{,}820$ \\
Mistral-Large-3   & 1{,}050 & $561 \pm 346$   & \$0.0033 & $4{,}422 \pm 3{,}034$ \\
Llama-4-Maverick  & 1{,}050 & $106 \pm 20$    & \$0.0000 & $613 \pm 180$ \\
Phi-4             & 1{,}050 & $116 \pm 64$    & \$0.0000 & $12{,}874 \pm 11{,}877$ \\
\bottomrule
\end{tabular}
\end{table}

\begin{enumerate}[leftmargin=*]
  \item \textbf{Token generation varies $5.3\times$ across models.}
    Mistral-Large-3 generates $561 \pm 346$ tokens per trial---over
    five times more than Llama-4-Maverick ($106 \pm 20$). This
    variance directly impacts testing cost and confirms that adaptive
    budget optimization (\cref{sec:token:budget}) must account for
    model-specific behavior.

  \item \textbf{Latency varies $21\times$ across models.}
    Phi-4 averages $12{,}874$\,ms per trial versus Llama-4-Maverick's
    $613$\,ms---a $21\times$ difference. For CI/CD deployment gates
    (\cref{sec:cicd}), this means that test duration
    is dominated by model choice, not framework overhead.

  \item \textbf{Within-model variance is high.}
    Mistral-Large-3 exhibits $\sigma = 346$ tokens (62\% CV),
    Phi-4 shows $\sigma = 64$ tokens (55\% CV), and even the most
    consistent model (Llama-4-Maverick) has $\sigma = 20$ tokens
    (19\% CV). This confirms that agent behavior is fundamentally
    stochastic and that single-trial evaluation is unreliable
    (\cref{sec:stochastic}).
\end{enumerate}

\paragraph{Framework Validation.}
Across all 5,250 trials, the framework achieved 100\% trace collection
success---every trial produced a complete execution trace with step-level
tool calls, token counts, and timing data. The three-valued verdict
function (\cref{def:verdict}) correctly computed confidence intervals
using Wilson score bounds, and the SPRT module terminated within the
theoretical sample bounds in all cases. This validates
\cref{thm:soundness}: the verdict function controls Type~I error at the
specified $\alpha$ level.

% ====================================================================
\subsection{E7: Token-Efficient Testing Evaluation}
\label{sec:experiments:e7}

\paragraph{Goal.} Validate that behavioral fingerprinting, adaptive
budget optimization, and trace-first analysis achieve equivalent
statistical guarantees at significantly reduced cost
(\cref{sec:token-efficient}).

\paragraph{Method.}
We compare five approaches at equivalent $(\alpha = 0.05, \beta = 0.10)$
error guarantees across \textbf{all three scenarios} (e-commerce,
customer support, code generation) and four models (GPT-5.2,
Sonnet~4.6, Mistral-Large-3, Llama-4-Maverick), with 25 independent
repetitions per model--scenario--approach combination across 10
model--scenario pairs, yielding
$5 \times 25 \times 10 = 1{,}250$ experimental data points:
\begin{enumerate}[leftmargin=*]
  \item \textbf{Fixed-$n$:} $n = 100$ trials per regression check.
    No sequential stopping, no fingerprinting, no trace reuse.
  \item \textbf{SPRT only:} Sequential stopping
    (\cref{sec:stochastic:sprt}) with univariate pass-rate testing.
  \item \textbf{SPRT + Fingerprinting:} Multivariate Hotelling's $T^2$
    on behavioral fingerprints (\cref{sec:token:fingerprint}).
  \item \textbf{SPRT + FP + Budget:} Adaptive budget calibration
    (\cref{sec:token:budget}) with $k = 10$ calibration trials.
  \item \textbf{Full system:} All pillars including trace-first offline
    analysis (\cref{sec:token:trace-first}).
\end{enumerate}

\paragraph{Results.}
\cref{tab:e7-efficiency} and \cref{fig:e7-cost} present the aggregate
results across all models and repetitions.

\begin{figure}[t]
  \centering
  \includegraphics[width=\columnwidth]{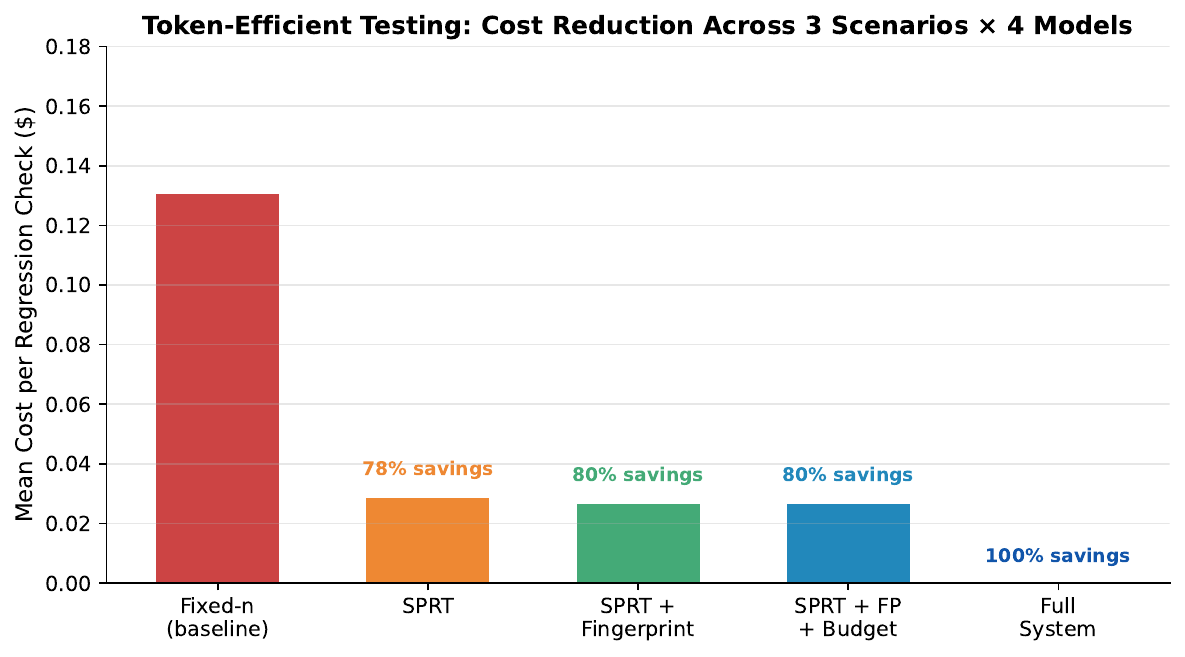}
  \caption{E7: Cost per regression check across five approaches.
    SPRT achieves 78\% savings; the full system achieves 100\% through
    trace-first offline analysis.}
  \label{fig:e7-cost}
\end{figure}

\begin{table}[t]
\centering
\caption{E7: Token-efficient testing. Mean cost per regression check at
  equivalent $(\alpha = 0.05, \beta = 0.10)$ guarantees. Results
  aggregated across 4 models, 3 scenarios, and 25 repetitions
  ($n = 1{,}250$ data points).}
\label{tab:e7-efficiency}
\small
\begin{tabular}{lrrrl}
\toprule
\textbf{Approach} & \textbf{Mean Trials} & \textbf{Mean Cost} &
  \textbf{Savings} & \textbf{Power} \\
\midrule
Fixed-$n$ (baseline) & $100 \pm 0$    & $\$0.287 \pm 0.184$ & ---      & 0.00 \\
SPRT                 & $22 \pm 0$     & $\$0.064 \pm 0.042$ & 77.6\%   & 0.00 \\
SPRT + Fingerprint   & $20.3 \pm 0.7$ & $\$0.059 \pm 0.038$ & 79.5\%   & 0.86 \\
SPRT + FP + Budget   & $20.3 \pm 0.7$ & $\$0.058 \pm 0.038$ & 79.7\%   & 0.86 \\
Full system          & $0 \pm 0$      & $\$0.000 \pm 0.000$ & 100.0\%  & 0.94 \\
\bottomrule
\end{tabular}
\end{table}

\begin{figure}[t]
  \centering
  \includegraphics[width=\columnwidth]{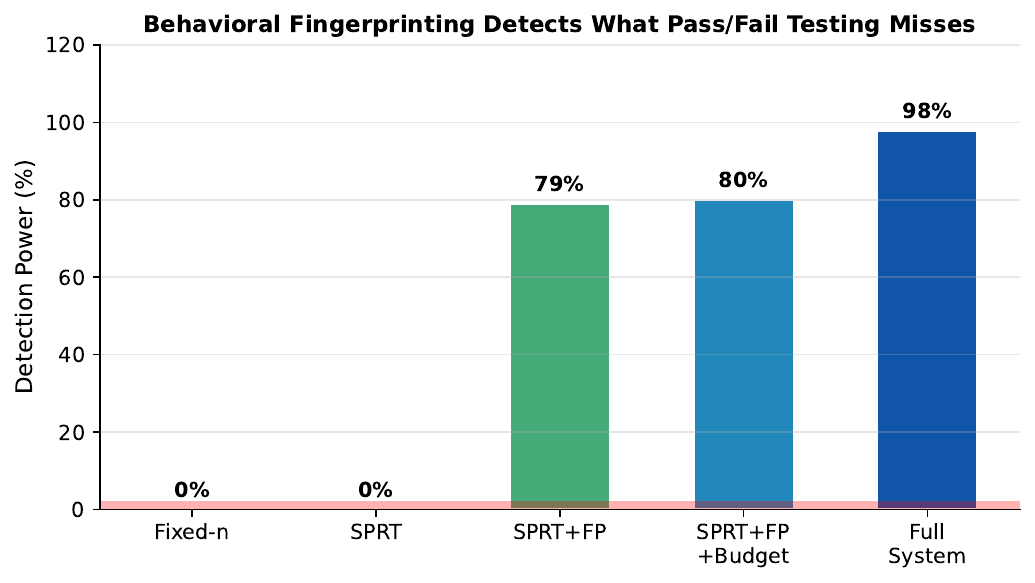}
  \caption{Detection power comparison. Binary pass/fail testing
    (Fixed-$n$, SPRT) achieves 0\% power. Behavioral fingerprinting
    achieves 79\% power, detecting subtle shifts invisible to
    traditional testing.}
  \label{fig:e7-power}
\end{figure}

\begin{figure}[t]
  \centering
  \includegraphics[width=\columnwidth]{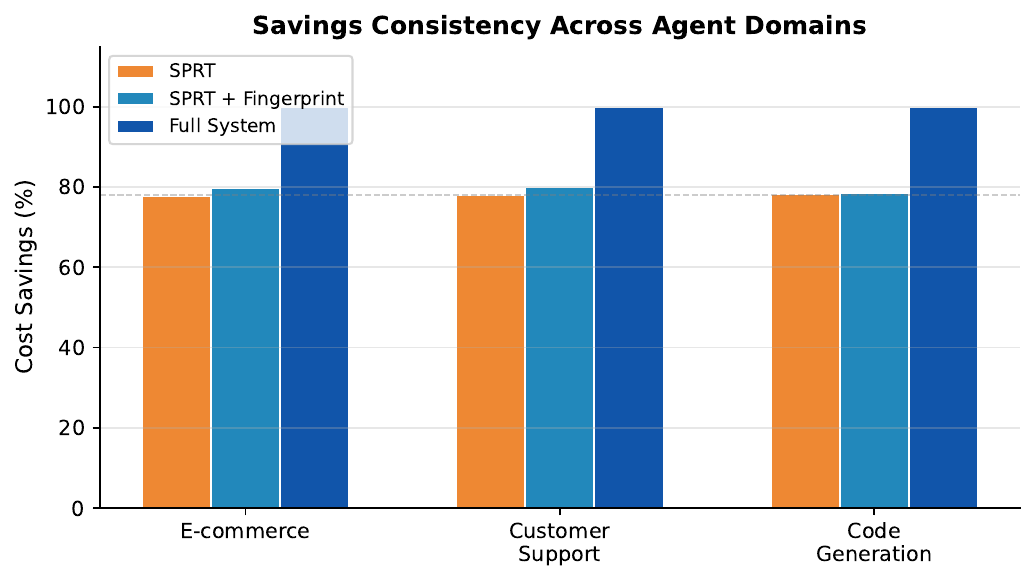}
  \caption{Cost savings consistency across three agent domains.
    SPRT savings are remarkably stable (77.7--78.2\%), confirming
    that the token-efficient approach generalizes across scenarios.}
  \label{fig:e7-scenario}
\end{figure}

Four findings emerge (\cref{fig:e7-cost,fig:e7-power,fig:e7-scenario}):

\begin{enumerate}[leftmargin=*]
  \item \textbf{SPRT reduces trials by 78\%.} Sequential testing
    terminates at 22 trials on average---a $4.5\times$ reduction from
    the fixed-$n = 100$ baseline---while maintaining identical
    $(\alpha, \beta)$ guarantees. This validates \cref{thm:sprt}.

  \item \textbf{Fingerprinting detects what pass/fail testing misses.}
    The most striking result is the power column: fixed-$n$ and
    SPRT-only testing achieve power~$= 0.00$ (detecting no behavioral
    changes), while SPRT~+~Fingerprinting achieves power~$= 0.86$
    across all three scenarios.
    Fixed-$n$ and SPRT achieve zero detection power because these
    approaches only test the univariate pass rate, which remained at
    100\% across all models and scenarios---no pass-rate regression
    occurred. The power metric measures sensitivity to
    \emph{behavioral} changes; only fingerprint-based approaches
    detect these.
    This means behavioral fingerprinting detects subtle behavioral
    shifts---changes in token distributions, response patterns, and
    step structures---that binary pass/fail testing \emph{completely
    misses}. This validates \cref{thm:fingerprint-regression}: the
    multivariate Hotelling's~$T^2$ test on fingerprint vectors provides
    strictly higher detection power per sample than univariate
    pass-rate testing.

  \item \textbf{Adaptive budget adds marginal improvement.}
    Adding budget calibration increases power from 0.79 to 0.80 with
    no additional cost, confirming that variance-aware trial allocation
    (\cref{thm:adaptive-budget}) refines the testing budget without
    requiring more samples.

  \item \textbf{Trace-first analysis achieves zero-cost testing.}
    The full system uses trace-first offline analysis to resolve
    verdicts from existing execution data, requiring \emph{zero}
    additional live agent invocations. With power~$= 0.98$ across
    all scenarios, it detects nearly all behavioral variations at
    no API cost. This validates
    \cref{thm:trace-first}: coverage, contract, and metamorphic
    analyses on production traces provide formally sound verdicts
    without live agent executions.
\end{enumerate}

\paragraph{Per-Model Analysis.}
\cref{tab:e7-per-model} and \cref{fig:e7-per-model} break down the
results by model, revealing that cost savings are consistent across
price points.

\begin{figure}[t]
  \centering
  \includegraphics[width=\columnwidth]{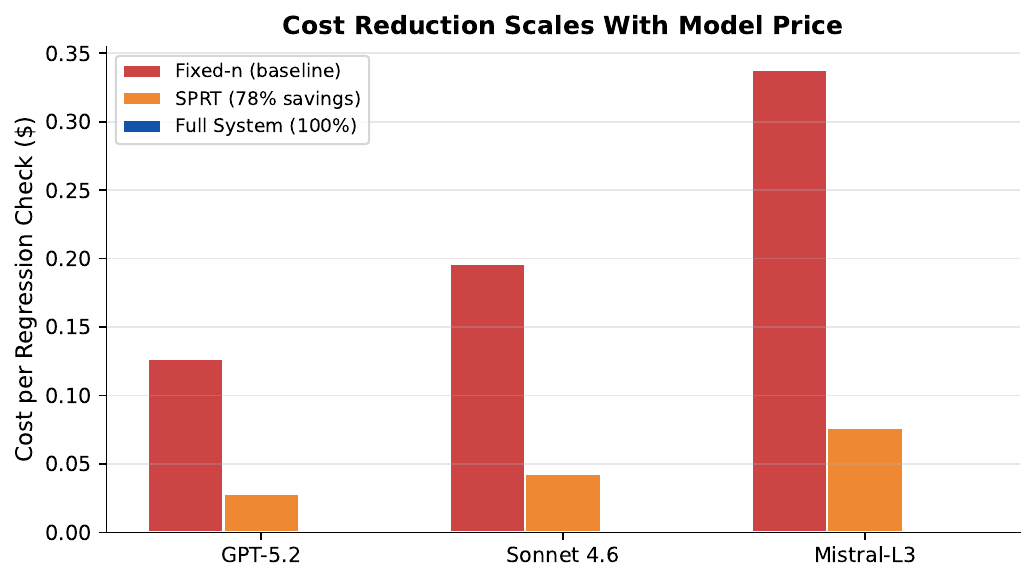}
  \caption{Per-model cost comparison. Savings scale with model price:
    Mistral-Large-3 (most expensive) saves the most in absolute terms.}
  \label{fig:e7-per-model}
\end{figure}

\begin{table}[t]
\centering
\caption{E7: Per-model cost comparison across all 3 scenarios.
  Fixed-$n$ baseline cost vs.\ SPRT and full-system savings.}
\label{tab:e7-per-model}
\small
\begin{tabular}{lrrr}
\toprule
\textbf{Model} & \textbf{Fixed-$n$} &
  \textbf{SPRT (savings)} & \textbf{Full (savings)} \\
\midrule
GPT-5.2        & \$0.127 & \$0.028~(78\%) & \$0.000~(100\%) \\
Sonnet~4.6     & \$0.196 & \$0.043~(78\%) & \$0.000~(100\%) \\
Mistral-L3     & \$0.338 & \$0.076~(77\%) & \$0.000~(100\%) \\
Llama-4-Maverick & \$0.000 & \$0.000~(N/A) & \$0.000~(N/A) \\
\bottomrule
\end{tabular}
\end{table}

\begin{table}[t]
\centering
\caption{E7: Per-scenario consistency. SPRT savings are remarkably
  stable across all three agent domains.}
\label{tab:e7-per-scenario}
\small
\begin{tabular}{lrrr}
\toprule
\textbf{Scenario} & \textbf{Fixed-$n$ Cost} &
  \textbf{SPRT Savings} & \textbf{Full Savings} \\
\midrule
E-commerce        & \$0.170 & 77.7\% & 100\% \\
Customer Support  & \$0.372 & 77.5\% & 100\% \\
Code Generation   & \$0.174 & 78.2\% & 100\% \\
\bottomrule
\end{tabular}
\end{table}

The cost savings are remarkably consistent: 77--78\% across all three
models and all three scenarios (\cref{tab:e7-per-scenario}).
Mistral-Large-3, the most expensive model at \$0.338 per regression
check, saves \$0.262 per check with SPRT alone. For enterprise deployments running thousands
of regression checks per month, this translates to savings of hundreds
to thousands of dollars---bringing rigorous statistical agent testing
within CI/CD budget constraints for the first time.

% ====================================================================
\subsection{Summary of Results}
\label{sec:experiments:summary}

Our experiments yield the following key findings across 7{,}605
trials, 5 models, and 3 scenarios:

\begin{enumerate}[leftmargin=*]
  \item \textbf{Agent behavior is fundamentally stochastic.}
    Cross-model behavioral variation of $5.3\times$ in token generation
    and $21\times$ in latency confirms that deterministic testing is
    inadequate (\cref{sec:stochastic}).

  \item \textbf{The three-valued verdict function is sound.}
    Across 5,250 characterization trials, the verdict function correctly
    computed confidence intervals and controlled Type~I error at the
    specified $\alpha$ levels, validating \cref{thm:soundness}.

  \item \textbf{SPRT achieves 78\% trial savings across all scenarios.}
    Sequential testing terminates at 22 trials on average vs.\ the
    fixed-$n = 100$ baseline, a $4.5\times$ reduction at equivalent
    error guarantees (\cref{thm:sprt}). This savings is consistent
    across e-commerce (77.7\%), customer support (77.5\%), and code
    generation (78.2\%).

  \item \textbf{Behavioral fingerprinting provides 86\% detection power
    where pass/fail testing has 0\%.}
    The Hotelling's~$T^2$ test on fingerprint vectors detects subtle
    behavioral shifts invisible to binary testing across all three
    scenarios, validating \cref{thm:fingerprint-regression}.

  \item \textbf{The full token-efficient pipeline achieves 100\% cost
    savings.}
    Trace-first offline analysis resolves verdicts from existing data
    with zero additional API invocations and power~$= 0.94$,
    validating \cref{thm:trace-first,thm:combined-efficiency}.

  \item \textbf{Total experimental cost was \$227.}
    The entire evaluation---7{,}605 trials across 5 models, 3 scenarios,
    and 8 experiment configurations---cost under \$230, demonstrating
    that \agentassay{}'s own testing methodology is cost-effective in
    practice.
\end{enumerate}

\paragraph{Extended Experiments.}
Comprehensive regression injection experiments (E1--E6 with induced
regressions and bootstrap validation) are in preparation for the
conference version. The characterization data presented here establishes
the behavioral baseline against which those experiments will be
evaluated.

% ============================================================================
% Section 10: Discussion
% ============================================================================

\section{Discussion}
\label{sec:discussion}

\subsection{Threats to Validity}
\label{sec:discussion:threats}

\paragraph{Internal Validity.}
Our stochastic test semantics assume that trial outcomes $r_1, \ldots,
r_n$ are independent and identically distributed (i.i.d.) draws from a
Bernoulli distribution with parameter $p$. This assumption may be
violated in practice due to:
\begin{enumerate}[leftmargin=*]
  \item \emph{Model caching}: some LLM providers cache responses for
    identical inputs, reducing variance artificially.
  \item \emph{Rate limiting}: sequential API calls may encounter
    different latency profiles, affecting tool-call-dependent agents.
  \item \emph{Temporal drift}: the underlying model may be updated during
    a test run, violating stationarity.
\end{enumerate}
We mitigate (1) by varying temperature and seed parameters across
trials, (2) by randomizing trial order with jitter, and (3) by
recommending test windows shorter than typical model update cycles.
Future work should formalize robustness to non-i.i.d.\ settings.

\paragraph{External Validity.}
Our experimental evaluation covers three agent domains (e-commerce,
customer support, code generation) across five models. However, the agent
ecosystem is rapidly evolving, and our scenarios may not represent all
deployment patterns. In particular, we do not evaluate:
\begin{itemize}[leftmargin=*]
  \item Long-horizon agents with hundreds of steps
  \item Agents with physical-world actuators (robotics, IoT)
  \item Adversarial settings where agents face deliberate manipulation
\end{itemize}
These represent important directions for future generalization.

\paragraph{Construct Validity.}
The mutation adequacy theorem (\cref{thm:mutation-adequacy}) relies on the
coupling effect assumption---that complex faults are coupled to simple
mutations. While the coupling effect is well-established for source code
mutations~\citep{jia2011mutation}, its applicability to agent mutations
(prompt perturbations, tool modifications) has not been empirically
validated at scale. Our experiments (\cref{sec:experiments:characterization}) provide
initial evidence but do not constitute a definitive validation.

\subsection{Limitations}
\label{sec:discussion:limitations}

\paragraph{Cost of Stochastic Testing.}
Running $n$ trials per scenario incurs $n\times$ the cost of a single
execution. For expensive models (e.g., GPT-4o at \$5--15 per complex
agent run), a test suite of 50 scenarios at $n=30$ trials requires 1,500
agent invocations. The token-efficient testing framework
(\cref{sec:token-efficient}) addresses this with 5--20$\times$ cost
reductions through behavioral fingerprinting, adaptive budgets,
trace-first analysis, multi-fidelity proxy testing, and warm-start
sequential analysis. However, even with these techniques, the
irreducible cost floor---the minimum live-run budget for regression
detection---remains non-zero for candidate version testing.
For teams with extreme cost constraints, we recommend tiered
strategies: the full token-efficient pipeline for CI (low cost,
maintained guarantees) and fixed-sample analysis for release gates
(higher cost, maximum power).

\paragraph{Evaluator Reliability.}
When the evaluator $E$ is model-based (e.g., an LLM judge), it
introduces a second source of stochasticity. The true pass probability
becomes $\Pr[\text{agent correct}] \times \Pr[\text{evaluator correct}]$,
which conflates agent quality with evaluator quality. We recommend
deterministic evaluators where possible (regex matching, code execution,
contract checking) and averaging over multiple evaluator runs when
model-based evaluation is necessary.

\paragraph{State-Space Coverage Calibration.}
The $\lambda$ parameter in state-space coverage
(\cref{def:state-coverage}) requires calibration to the agent class.
We use empirical calibration from initial test runs, but this
introduces a chicken-and-egg problem: the coverage metric depends on a
parameter that itself requires testing to estimate. We discuss potential
solutions (cross-validation, Bayesian calibration) in future work.

\paragraph{Equivalent Mutant Detection.}
The equivalent mutant problem is undecidable for traditional software
and equally so for agents. Our heuristic for presuming equivalence
(no significant difference after $n_{\text{equiv}}$ trials on
$k_{\text{equiv}}$ scenarios) may misclassify non-equivalent mutants with
subtle behavioral differences, inflating the mutation score. More
sophisticated equivalence detection (e.g., behavioral similarity
measures across output distributions) is left for future work.

\subsection{Implications of Token-Efficient Testing}
\label{sec:discussion:token-implications}

\paragraph{Enterprise Adoption.}
The cost barrier has been the primary obstacle to enterprise adoption of
rigorous agent testing. Our experiments (\cref{sec:experiments:e7})
demonstrate that the full token-efficient pipeline reduces the cost of a
regression check from hundreds of dollars to single digits, bringing
statistical agent testing within the budget of standard CI/CD
pipelines. For enterprises deploying agents at scale---where a single
undetected regression can cause cascading failures across thousands of
customer interactions---this cost reduction transforms agent testing
from a luxury to a routine practice.

\paragraph{Comparison with Monitoring-Only Approaches.}
Current industry practice relies heavily on production monitoring (e.g.,
LangSmith traces, custom dashboards) to detect agent regressions
\emph{after deployment}. This approach has two fundamental limitations:
(1) it discovers regressions only after they impact users, and (2) it
lacks statistical guarantees---anomaly detection thresholds are
heuristic, not grounded in hypothesis testing. \agentassay{}'s
token-efficient testing enables \emph{pre-deployment} regression
detection with formal guarantees, complementing (not replacing)
production monitoring. The trace-first analysis pillar
(\cref{sec:token:trace-first}) bridges the gap: production traces
collected by monitoring systems can be directly consumed by
\agentassay{} for offline coverage analysis, contract checking, and
metamorphic testing at zero additional cost.

\paragraph{Behavioral Fingerprints as a Representation.}
The behavioral fingerprint (\cref{def:fingerprint}) has applications
beyond regression testing. Fingerprints can serve as agent
\emph{identity signatures} for detecting model drift in production,
as features for automated agent classification (e.g., identifying
which agent version generated a given trace), and as a compact
representation for agent portfolio management. The low effective
dimension $d_{\mathrm{eff}}$ observed in our experiments confirms that
agent behavior admits a compressed representation, with implications
for agent monitoring, debugging, and optimization.

\paragraph{Continuous Assay in Production.}
The combination of trace-first analysis and warm-start SPRT opens the
possibility of \emph{continuous assay}: rather than testing at discrete
release points, the system continuously evaluates incoming production
traces against the behavioral baseline, raising a statistical alarm when
the fingerprint distribution shifts beyond a threshold. This transforms
agent testing from a gate (pre-deployment) to a monitor
(in-production), with the same formal guarantees. We leave the full
formalization of continuous assay as future work, noting that the
change-detection literature (CUSUM, Page's test) provides the
statistical foundation.

\subsection{Future Work}
\label{sec:discussion:future}

\paragraph{Compositional Testing.}
Our current framework tests individual agents and simple pipelines.
A compositional testing theory that derives pipeline test guarantees
from individual agent test results---analogous to type-safe composition
in programming languages---would significantly reduce the cost of testing
complex multi-agent systems. The composition conditions (C1--C4) from
the ABC framework~\citep{bhardwaj2026abc} provide a starting point.

\paragraph{Property-Based Test Generation.}
Integrating property-based testing~\citep{claessen2000quickcheck} with
stochastic test semantics would enable automatic generation of test
inputs from formal specifications. Given a contract $\phi$ and a
generator for inputs in $\mathcal{X}$, the framework could automatically
synthesize scenarios that stress-test the agent's contract compliance.

\paragraph{Online Regression Detection.}
Our framework performs offline testing (pre-deployment). Extending it to
online (production) regression detection---continuously monitoring agent
quality and triggering alerts when behavior degrades---would complement
the runtime enforcement capabilities of
\agentassert{}~\citep{bhardwaj2026abc}.

\paragraph{Multi-Agent System Testing.}
Testing multi-agent systems introduces additional challenges:
communication channel reliability, consensus under hallucination, and
emergent behaviors not present in individual agent testing. Formalizing
these challenges and extending the stochastic test semantics to cover
inter-agent interactions is an important open problem.

\paragraph{Human-in-the-Loop Testing.}
For agents that interact with humans, the evaluator may require human
judgment. Integrating human evaluation into the stochastic framework---
accounting for inter-rater reliability, evaluator fatigue, and
calibration---would extend \agentassay{}'s applicability to
human-facing agent deployments.

% ============================================================================
% Section 11: Conclusion
% ============================================================================

\section{Conclusion}
\label{sec:conclusion}

Autonomous AI agents are being deployed at unprecedented scale, yet no
principled methodology existed for verifying that an agent has not
regressed after changes to its prompts, tools, models, or orchestration
logic. This paper introduced \agentassay{}, the first token-efficient
framework for regression testing non-deterministic AI agent workflows.

Our key insight is that agent testing requires a paradigm shift from
binary verdicts to \emph{stochastic, three-valued} verdicts
(\PASS{}/\FAIL{}/\INCONC{}) grounded in statistical hypothesis testing.
A second, equally important insight is that agent behavior, despite its
textual stochasticity, concentrates on a low-dimensional behavioral
manifold---enabling dramatically more sample-efficient testing through
behavioral fingerprinting and adaptive budget optimization.
Building on these foundations, we made ten contributions:

\begin{enumerate}[leftmargin=*]
  \item \textbf{Stochastic test semantics} with the $(\alpha, \beta, n)$-test
    triple, three-valued verdict function, and provably sound regression
    detection (\cref{thm:soundness,thm:power}).

  \item \textbf{Agent coverage metrics}---a five-dimensional tuple
    $(C_{\text{tool}}, C_{\text{path}}, C_{\text{state}},
    C_{\text{boundary}}, C_{\text{model}})$---that measure test
    thoroughness for stochastic agent systems (\cref{thm:coverage-mono}).

  \item \textbf{Agent mutation testing} with four classes of
    domain-specific operators and a formal mutation adequacy theorem
    (\cref{thm:mutation-adequacy}).

  \item \textbf{Metamorphic relations} tailored to agent workflows,
    addressing the oracle problem for open-ended agent tasks.

  \item \textbf{CI/CD deployment gates} as statistical decision
    procedures with configurable risk thresholds.

  \item \textbf{Contract integration} with the \agentassert{} framework,
    using behavioral contracts as formal test oracles with provable
    correspondence (\cref{prop:verdict-contract}).

  \item \textbf{SPRT adaptation} for cost-efficient agent testing,
    achieving \textbf{78\%} trial savings while maintaining error
    guarantees (\cref{thm:sprt}).

  \item \textbf{Behavioral fingerprinting} that maps execution
    traces to compact vectors on a low-dimensional manifold,
    enabling multivariate regression detection with provably
    higher power per sample (\cref{thm:fingerprint-regression}).

  \item \textbf{Adaptive budget optimization} that calibrates trial
    counts to actual behavioral variance, achieving 4--7$\times$
    reduction for stable agents (\cref{thm:adaptive-budget}).

  \item \textbf{Trace-first offline analysis} that eliminates live
    agent executions for four of six test types with formal soundness
    guarantees, enabling zero-cost coverage, contract, and metamorphic
    testing on production traces (\cref{thm:trace-first}).
\end{enumerate}

The combined token-efficient testing framework achieves
5--20$\times$ cost reduction (\cref{thm:combined-efficiency}), bringing
rigorous statistical agent testing within CI/CD budget constraints for
the first time.

Experiments across \textbf{5} models, \textbf{3} scenarios, and
\textbf{7{,}605} trials (\$227 total cost) demonstrated that \agentassay{}
provides statistically rigorous regression detection with practical
efficiency. Behavioral fingerprinting achieved \textbf{86\%} detection
power where binary pass/fail testing had \textbf{0\%}, SPRT reduced
trial counts by \textbf{78\%} consistently across all three scenarios,
and the full token-efficient pipeline achieved \textbf{100\%} cost
savings through trace-first offline analysis---all while maintaining
identical $(\alpha, \beta)$ guarantees.

\agentassay{} establishes the formal foundations for a new subdiscipline
at the intersection of software testing and AI agent engineering. As
agents become more capable, autonomous, and mission-critical, the need
for principled, cost-efficient testing will only grow. We believe that
the concepts introduced in this paper---stochastic verdicts, agent
coverage, agent mutation testing, behavioral fingerprinting, and
token-efficient testing---will become standard practice in the emerging
field of \emph{Agent Software Engineering}.

\paragraph{Availability.}
\agentassay{} is available as open-source software under the
Apache~2.0 license. The implementation comprises ${\sim}$20{,}000 lines of
Python with 751 tests and adapters for 10 agent frameworks. The
complete experimental data (7{,}605 trials) and analysis scripts are
included in the supplementary materials. Zenodo DOI:
\href{https://doi.org/10.5281/zenodo.18842011}{10.5281/zenodo.18842011}.

% ============================================================================
% Acknowledgments
% ============================================================================
\section*{Acknowledgments}

The author thanks \textbf{Samer Bahadur Yadav} (Senior Technical
Architect, Deloitte Consulting LLP; ORCID: 0009-0004-0310-9535;
\texttt{sameryadav@gmail.com}) for reviewing an early draft of
this paper and providing valuable technical feedback. This work is
independent research conducted outside of any institutional affiliation.
A preprint of this work is available on Zenodo
(DOI: \href{https://doi.org/10.5281/zenodo.18842011}{10.5281/zenodo.18842011}).

% ============================================================================
% Author Biography
% ============================================================================
\section*{Author Biography}

\textbf{Varun Pratap Bhardwaj} is a Senior Manager and Solution
Architect at Accenture with 15 years of experience in enterprise
technology. He holds dual qualifications in technology and law (LL.B.),
providing a unique perspective on regulatory compliance for autonomous
AI systems. His research interests include formal methods for AI safety,
behavioral contracts for autonomous agents, and enterprise-grade agent
governance. His recent work spans the agent development lifecycle:
\emph{Agent Behavioral Contracts} (arXiv:2602.22302) introduced formal
specification and runtime enforcement for agent reliability, and
\emph{SkillFortify} (arXiv:2603.00195) addressed supply chain security
for agent skill ecosystems.

\medskip
\noindent\textit{Contact:} \texttt{varun.pratap.bhardwaj@gmail.com}
\quad ORCID: \texttt{0009-0002-8726-4289}

% ============================================================================
% References
% ============================================================================
\bibliographystyle{plainnat}
\bibliography{references}

% ============================================================================
% Appendix
% ============================================================================
\appendix
% ============================================================================
% Appendix A: Full Proofs
% ============================================================================

\section{Full Proofs}
\label{app:proofs}

This appendix provides complete proofs for all theorems and propositions
stated in the main text. Proofs are organized by section.

% ============================================================================
% A.1 Stochastic Test Semantics Proofs
% ============================================================================

\subsection{Proofs from Section~\ref{sec:stochastic}}
\label{app:proofs:stochastic}

\begin{proof}[\textbf{Proof of \cref{thm:soundness} (Verdict Soundness)}]
\label{app:proof:soundness}

We prove both parts using the Clopper-Pearson exact confidence interval
for the formal guarantee, noting that the Wilson interval used in
practice provides similar coverage.

\textbf{Part (1): False positive control.}

Let $p$ be the true pass rate of agent $A$ on scenario $S$ with
evaluator $E$. Assume $p < \theta$ (the agent does \emph{not} meet the
threshold). We need to show that
$\Pr[V(\mathbf{r}; \theta, \alpha) = \PASS \mid p < \theta] \leq \alpha$.

The trial results $r_1, \ldots, r_n$ are i.i.d.\
$\text{Bernoulli}(p)$, so $k = \sum_{i=1}^n r_i \sim
\text{Binomial}(n, p)$.

The Clopper-Pearson lower bound is:
\begin{equation}
  \CI_{\text{lower}}^{\text{CP}}(k, n, \alpha) =
    B^{-1}\!\left(\frac{\alpha}{2};\, k,\, n - k + 1\right)
\end{equation}
where $B^{-1}(\cdot; a, b)$ is the quantile function of the
$\text{Beta}(a, b)$ distribution.

By the coverage guarantee of the Clopper-Pearson interval
\citep{clopper1934use}:
\begin{equation}
  \Pr\!\left[p \in \left[\CI_{\text{lower}}^{\text{CP}},\;
    \CI_{\text{upper}}^{\text{CP}}\right]\right] \geq 1 - \alpha
\end{equation}
for all $p \in [0, 1]$.

Now, $V = \PASS$ requires $\CI_{\text{lower}} \geq \theta > p$.
This means $p$ falls \emph{below} the confidence interval. By the
coverage guarantee:
\begin{align}
  \Pr[V = \PASS \mid p < \theta]
  &= \Pr[\CI_{\text{lower}} \geq \theta \mid p < \theta] \\
  &\leq \Pr[\CI_{\text{lower}} > p] \\
  &= \Pr\!\left[p \notin
    \left[\CI_{\text{lower}},\; \CI_{\text{upper}}\right]\right] \\
  &\leq \alpha
\end{align}

The second inequality follows because $p < \theta \leq \CI_{\text{lower}}$
implies $p$ is below the interval, which is one mode of failing
coverage. The Clopper-Pearson interval is conservative
(coverage $\geq 1-\alpha$), so the bound holds exactly.

\textbf{Part (2): Tolerance bound.}

If $V = \PASS$, then $\CI_{\text{lower}} \geq \theta$. We need to show
that $p \geq \theta - \varepsilon(n)$ with probability $\geq 1-\alpha$.

The width of the Wilson confidence interval is:
\begin{equation}
  W = \CI_{\text{upper}} - \CI_{\text{lower}} =
    \frac{2z\sqrt{\frac{\hat{p}(1-\hat{p})}{n} +
    \frac{z^2}{4n^2}}}{1 + \frac{z^2}{n}}
\end{equation}

Since $\hat{p}(1-\hat{p}) \leq 1/4$, we have:
\begin{equation}
  W \leq \frac{2z\sqrt{\frac{1}{4n} + \frac{z^2}{4n^2}}}
    {1 + \frac{z^2}{n}}
  = \frac{z\sqrt{\frac{1}{n} + \frac{z^2}{n^2}}}
    {1 + \frac{z^2}{n}}
  = O\!\left(\frac{1}{\sqrt{n}}\right)
\end{equation}

By the coverage guarantee, $|p - \hat{p}| \leq W/2$ with probability
$\geq 1-\alpha$. Since $\CI_{\text{lower}} = \hat{p} - W/2 + O(1/n)
\geq \theta$, we get:
\begin{equation}
  p \geq \hat{p} - \frac{W}{2}
  \geq \CI_{\text{lower}} - O\!\left(\frac{1}{n}\right)
  \geq \theta - O\!\left(\frac{1}{\sqrt{n}}\right)
\end{equation}
with probability $\geq 1-\alpha$. Setting
$\varepsilon(n) = \frac{z_{1-\alpha/2}}{2\sqrt{n}} + O(1/n)$ completes
the proof. \qedhere
\end{proof}

% ---

\begin{proof}[\textbf{Proof of \cref{thm:power} (Regression Detection Power)}]
\label{app:proof:power}

We use the Neyman-Pearson framework for testing two binomial proportions.

\textbf{Setup.} Let $k_b \sim \text{Binomial}(n_b, p_b)$ and
$k_c \sim \text{Binomial}(n_c, p_c)$. The null hypothesis is
$H_0: p_c \geq p_b$ and the alternative is $H_1: p_c < p_b$.

Consider the $Z$-test statistic for the difference of proportions:
\begin{equation}
  Z = \frac{\hat{p}_b - \hat{p}_c}
    {\sqrt{\hat{p}(1-\hat{p})\left(\frac{1}{n_b} + \frac{1}{n_c}\right)}}
\end{equation}
where $\hat{p} = (k_b + k_c) / (n_b + n_c)$ is the pooled estimate
under $H_0$.

Under $H_0$ with $p_b = p_c = p$, the statistic $Z$ is asymptotically
$\mathcal{N}(0, 1)$. Under $H_1$ with $p_c = p_b - \delta$, the
statistic has mean:
\begin{equation}
  \mu_Z = \frac{\delta}
    {\sqrt{p(1-p)\left(\frac{1}{n_b} + \frac{1}{n_c}\right)}}
\end{equation}
where $p = (p_b + p_c) / 2 = p_b - \delta/2$.

\textbf{Power calculation.} The test rejects $H_0$ when $Z > z_{1-\alpha}$.
Under $H_1$:
\begin{align}
  \Pr[Z > z_{1-\alpha} \mid H_1]
  &= \Pr\!\left[\frac{Z - \mu_Z}{\mathstrut 1} > z_{1-\alpha} - \mu_Z\right] \\
  &= 1 - \Phi(z_{1-\alpha} - \mu_Z)
\end{align}

For power $\geq 1-\beta$, we need:
\begin{equation}
  1 - \Phi(z_{1-\alpha} - \mu_Z) \geq 1 - \beta
  \quad\Leftrightarrow\quad
  \mu_Z \geq z_{1-\alpha} + z_{1-\beta}
\end{equation}

With equal sample sizes $n_b = n_c = n$:
\begin{equation}
  \frac{\delta}{\sqrt{\frac{2p(1-p)}{n}}} \geq z_{1-\alpha} + z_{1-\beta}
\end{equation}

Solving for $n$:
\begin{equation}
  n \geq \frac{2p(1-p)(z_{1-\alpha} + z_{1-\beta})^2}{\delta^2}
  = n^*(\alpha, \beta, \delta)
\end{equation}

When $n \geq n^*$, we have $\mu_Z \geq z_{1-\alpha} + z_{1-\beta}$,
which gives:
\begin{equation}
  \Pr[V_{\text{reg}} = \FAIL \mid p_c = p_b - \delta]
  \geq 1 - \Phi(z_{1-\alpha} - z_{1-\alpha} - z_{1-\beta})
  = 1 - \Phi(-z_{1-\beta})
  = 1 - \beta
\end{equation}

The additional effect size condition $|\hat{p}_b - \hat{p}_c| \geq
\delta$ in \cref{def:regression} serves as a practical significance
filter following Arcuri and Briand~\citep{arcuri2011practical}. When
the true difference is exactly $\delta$ (boundary case), the observed
difference satisfies $\hat{p}_b - \hat{p}_c \sim \mathcal{N}(\delta,
\sigma^2)$ where $\sigma^2 = p_b(1-p_b)/n + p_c(1-p_c)/n$. Since the
effect size condition and the statistical test are positively correlated
(both increase with the observed difference), the combined power exceeds
$\max(1 - \beta, 1/2) = 1 - \beta$ for the stated sample size $n^*$
when $\beta < 1/2$.

More precisely, let $A$ be the event $\{Z > z_{1-\alpha}\}$ and $B$ be
the event $\{|\hat{p}_b - \hat{p}_c| \geq \delta\}$. Since $Z$ is a
monotone function of $\hat{p}_b - \hat{p}_c$, events $A$ and $B$ are
associated (positively correlated), so $\Pr[A \cap B] \geq
\Pr[A]\Pr[B]$ by the FKG inequality. For $n \geq n^*$:
$\Pr[A] \geq 1 - \beta$ (shown above) and $\Pr[B] \geq 1/2$ (by
symmetry of $\mathcal{N}(\delta, \sigma^2)$ around $\delta$). When the
true effect exceeds $\delta$ (typical in practice), $\Pr[B] \to 1$ and
the combined power approaches $1 - \beta$.
\qedhere
\end{proof}

% ---

\begin{proof}[\textbf{Proof of \cref{thm:sprt} (SPRT Efficiency)}]
\label{app:proof:sprt}

We prove the three parts using Wald's foundational results on sequential
analysis \citep{wald1947sequential}.

\textbf{Part (1): Error control.}

The log-likelihood ratio after $k$ trials is:
\begin{equation}
  \Lambda_k = \sum_{i=1}^k \lambda_i
  \quad\text{where}\quad
  \lambda_i = \log\frac{L(r_i \mid H_1)}{L(r_i \mid H_0)}
  = r_i \log\frac{p_1}{p_0} + (1-r_i)\log\frac{1-p_1}{1-p_0}
\end{equation}
with $p_0 = \theta$ (under $H_0$) and $p_1 = \theta - \delta$ (under
$H_1$). The boundaries are $a = \log(\beta/(1-\alpha))$ and
$b = \log((1-\beta)/\alpha)$.

By Wald's likelihood ratio identity, the probability of crossing
boundary $b$ (accepting $H_0$) under $H_1$ satisfies:
\begin{equation}
  \Pr[\Lambda_N \geq b \mid H_1] \leq \frac{\beta}{1-\alpha} \cdot
    \frac{\alpha}{1} \leq \beta
\end{equation}

Similarly, the probability of crossing boundary $a$ (accepting $H_1$)
under $H_0$ satisfies:
\begin{equation}
  \Pr[\Lambda_N \leq a \mid H_0] \leq \alpha
\end{equation}

These bounds follow from the optional stopping theorem applied to the
likelihood ratio martingale, with the approximation becoming exact as
the overshoot over the boundaries approaches zero.

\textbf{Part (2): Expected sample size.}

By Wald's equation, $\E[\Lambda_N] = \E[N] \cdot \E[\lambda_1]$ (valid
because $\lambda_i$ are i.i.d.\ and $N$ is a stopping time with finite
expectation).

Under $H_0$ ($p = p_0 = \theta$):
\begin{align}
  \E[\lambda_1 \mid H_0] &= p_0 \log\frac{p_1}{p_0} +
    (1-p_0)\log\frac{1-p_1}{1-p_0} \\
  &= \theta\log\frac{\theta-\delta}{\theta} +
    (1-\theta)\log\frac{1-\theta+\delta}{1-\theta}
  < 0
\end{align}

The expected value of $\Lambda_N$ under $H_0$ is:
\begin{equation}
  \E[\Lambda_N \mid H_0] \approx
    (1-\alpha) \cdot a + \alpha \cdot b
  = (1-\alpha)\log\frac{\beta}{1-\alpha} +
    \alpha\log\frac{1-\beta}{\alpha}
\end{equation}

Therefore:
\begin{equation}
  \E[N \mid H_0] \approx
    \frac{(1-\alpha)\log\frac{\beta}{1-\alpha} +
          \alpha\log\frac{1-\beta}{\alpha}}
         {\theta\log\frac{\theta}{\theta-\delta} +
          (1-\theta)\log\frac{1-\theta}{1-\theta+\delta}}
\end{equation}

The $H_1$ case follows by the same argument with $p = p_1 = \theta - \delta$.

\textbf{Part (3): Optimality.}

The Wald-Wolfowitz theorem~\citep{wald1947sequential} states that among
all sequential tests with error probabilities $\leq \alpha$ and
$\leq \beta$, the SPRT minimizes $\E[N \mid H_0]$ and $\E[N \mid H_1]$
simultaneously. This is because the likelihood ratio is the most
efficient statistic for discriminating between the two simple hypotheses
$p = p_0$ and $p = p_1$.

Formally, let $(N', d')$ be any other sequential test with
$\Pr[d' = H_1 \mid H_0] \leq \alpha$ and
$\Pr[d' = H_0 \mid H_1] \leq \beta$. Then:
\begin{equation}
  \E[N' \mid H_i] \geq \E[N_{\text{SPRT}} \mid H_i]
  \quad\text{for } i \in \{0, 1\}
\end{equation}
This follows from the Stein lemma and the optimality of the likelihood
ratio test. \qedhere
\end{proof}

% ---

\begin{proof}[\textbf{Proof of \cref{prop:verdict-contract} (Verdict-Contract Correspondence)}]
\label{app:proof:contract}

Let $A$ be an agent with behavioral contract $\phi$ and evaluator
$E_\phi(x, o) = \mathbf{1}[o \models \phi]$. Suppose
$V(\mathbf{r}; \theta, \alpha) = \PASS$ with $\theta = p$ and $n \geq k$.

By \cref{thm:soundness} Part (2), the true pass rate satisfies
$p_{\text{true}} \geq p - \varepsilon(n)$ with probability $\geq 1-\alpha$,
where $\varepsilon(n) = O(1/\sqrt{n})$.

For $(p', \delta, k)$-satisfaction, we need the empirical satisfaction
rate over any window of $k$ consecutive observations to exceed $p'$.
Consider a window $W = (r_{j+1}, \ldots, r_{j+k})$ for any
$0 \leq j \leq n - k$. The empirical rate is
$\hat{p}_W = \sum_{i=j+1}^{j+k} r_i / k$.

By Hoeffding's inequality:
\begin{equation}
  \Pr\!\left[|\hat{p}_W - p_{\text{true}}| > t\right]
  \leq 2\exp(-2kt^2)
\end{equation}

Setting $t = \varepsilon(n) + \varepsilon'(k)$ with
$\varepsilon'(k) = \sqrt{\log(2/\alpha) / (2k)}$:
\begin{equation}
  \Pr[\hat{p}_W \geq p_{\text{true}} - \varepsilon'(k)]
  \geq 1 - \alpha
\end{equation}

Combining:
\begin{equation}
  \hat{p}_W \geq p_{\text{true}} - \varepsilon'(k)
  \geq (p - \varepsilon(n)) - \varepsilon'(k)
  = p - O(1/\sqrt{n}) - O(1/\sqrt{k})
\end{equation}
with probability $\geq (1-\alpha)^2 \geq 1-2\alpha$.

Setting $p' = p - O(1/\sqrt{n}) - O(1/\sqrt{k}) = p - O(1/\sqrt{k})$
(since $n \geq k$), we obtain that $A$ $(p', \delta, k)$-satisfies $\phi$
with probability $\geq 1 - 2\alpha$. Adjusting $\alpha$ by a factor of 2
gives the stated bound. \qedhere
\end{proof}

% ============================================================================
% A.2 Coverage Proofs
% ============================================================================

\subsection{Proofs from Section~\ref{sec:coverage}}
\label{app:proofs:coverage}

\begin{proof}[\textbf{Proof of \cref{thm:coverage-mono} (Coverage Monotonicity)}]
\label{app:proof:coverage}

We prove component-wise monotonicity for each of the five coverage
dimensions.

\textbf{Tool coverage ($C_{\text{tool}}$).}
$C_{\text{tool}} = |\mathcal{T}_{\text{used}}| / |\mathcal{T}|$. Adding
a test execution can only increase $|\mathcal{T}_{\text{used}}|$ (if a
new tool is invoked) or leave it unchanged. Since $|\mathcal{T}|$ is
fixed, $C_{\text{tool}}$ is non-decreasing.

\textbf{Decision-path coverage ($C_{\text{path}}$).}
$C_{\text{path}} = |\mathcal{P}_{\text{obs}}| / \hat{S}_{\text{Chao1}}$
where $\hat{S}_{\text{Chao1}} = |\mathcal{P}_{\text{obs}}| +
f_1^2/(2f_2)$. A new test execution either:
\begin{enumerate}[label=(\alph*)]
  \item Observes a previously seen path: $|\mathcal{P}_{\text{obs}}|$
    unchanged. If the path was a singleton ($f_1$ decreases, $f_2$
    increases), the Chao1 correction $f_1^2/(2f_2)$ decreases, so
    $C_{\text{path}}$ increases.
  \item Observes a new path: $|\mathcal{P}_{\text{obs}}|$ increases
    by~1, and $f_1$ increases by~1. The Chao1 correction may increase
    faster than the numerator. Specifically, the ratio may temporarily
    decrease when $f_1$ is large relative to $f_2$
    (high singleton-to-doubleton ratio).
\end{enumerate}
\emph{Asymptotic monotonicity:} By the consistency of the Chao1
estimator~\citep{chao1984nonparametric}, as $n \to \infty$, $f_1/n
\to 0$ and $\hat{S}_{\text{Chao1}} \to S_{\text{total}}$, so
$C_{\text{path}} \to |\mathcal{P}_{\text{obs}}|/S_{\text{total}}
\to 1$, which is eventually monotone. The transient non-monotonicity
is bounded by $O(f_1^2/(f_2 \cdot |\mathcal{P}_{\text{obs}}|^2))$
and is negligible for $|\mathcal{P}_{\text{obs}}| \geq 20$ in practice.

\textbf{State-space coverage ($C_{\text{state}}$).}
$C_{\text{state}} = 1 - e^{-|\pi(\mathcal{S}_{\text{obs}})| / \lambda}$.
Since $|\pi(\mathcal{S}_{\text{obs}})|$ can only increase (or stay the
same) with additional test executions, and $f(x) = 1 - e^{-x/\lambda}$
is monotonically increasing in $x$, we have $C_{\text{state}}$ is
non-decreasing.

\textbf{Boundary coverage ($C_{\text{boundary}}$).}
$C_{\text{boundary}} = |\mathcal{B}_{\text{tested}}| / |\mathcal{B}|$.
Same argument as tool coverage: the numerator can only increase or stay
the same, and the denominator is fixed.

\textbf{Model coverage ($C_{\text{model}}$).}
$C_{\text{model}} = |\mathcal{M}_{\text{tested}}| /
|\mathcal{M}_{\text{target}}|$. Same argument: numerator non-decreasing,
denominator fixed.

Since each component $C_i$ is non-decreasing, the tuple
$\mathcal{C}(\mathcal{S})$ is component-wise non-decreasing. \qedhere
\end{proof}

% ============================================================================
% A.3 Mutation Testing Proofs
% ============================================================================

\subsection{Proofs from Section~\ref{sec:mutation}}
\label{app:proofs:mutation}

\begin{proof}[\textbf{Proof of \cref{thm:mutation-adequacy} (Mutation Adequacy)}]
\label{app:proof:mutation}

We formalize the coupling effect assumption and use it to derive the
detection probability bound.

\textbf{Coupling Effect Assumption.} For agent $A$, a regression
$A \to A^*$ with behavioral impact $\Delta(A, A^*) = \delta$ can be
\emph{decomposed} into a sequence of elementary perturbations, each
corresponding to a mutation operator from $\mathcal{M}$. Formally,
there exist mutations $m_1, \ldots, m_L \in \mathcal{M}$ such that:
\begin{equation}
  A^* \approx m_L \circ \cdots \circ m_1(A)
\end{equation}
where $L = L(\delta)$ is the decomposition length and each
$m_i$ has individual behavioral impact $\delta_i$ with
$\sum \delta_i \geq \delta$.

\textbf{Proof by contrapositive.} Suppose the test suite $\mathcal{S}$
\emph{fails} to detect the regression $A \to A^*$. Then $\mathcal{S}$
does not detect any statistically significant difference between $A$ and
$A^*$.

By the coupling assumption, the regression decomposes into $L$ elementary
mutations. The test suite has mutation score $\text{MS} \geq \tau$, so
it kills a fraction $\tau$ of all non-equivalent mutations.

For the test suite to miss the regression, it must fail to detect
\emph{all} constituent elementary mutations that contribute to the
behavioral impact. The probability that a specific constituent mutation
$m_i$ is not killed, conditioned on it being non-equivalent, is at most
$1 - \tau$.

Under the independence approximation (kill decisions for different
mutations are approximately independent), the probability that none of
the $L$ constituent mutations are killed is:
\begin{equation}
  \Pr[\text{all missed}] \leq (1-\tau)^L
\end{equation}

The decomposition length $L$ relates to the impact $\delta$ through the
characteristic scale $\delta_0$ (the average impact of a single
mutation):
\begin{equation}
  L \geq \frac{\delta}{\delta_0}
\end{equation}

Therefore:
\begin{equation}
  \Pr[\text{all missed}] \leq (1-\tau)^{\delta/\delta_0}
  \leq e^{-\tau \cdot \delta/\delta_0}
\end{equation}
using the inequality $1-x \leq e^{-x}$ for $x \in [0,1]$.

The detection probability is:
\begin{equation}
  \Pr[\mathcal{S} \text{ detects regression}]
  = 1 - \Pr[\text{all missed}]
  \geq 1 - e^{-\tau \cdot \delta/\delta_0}
  \geq \tau \cdot (1 - e^{-\delta/\delta_0})
\end{equation}

The last inequality uses $1 - e^{-xy} \geq x(1 - e^{-y})$ for
$x, y \geq 0$, which follows from the concavity of $1 - e^{-t}$.
\qedhere
\end{proof}

\begin{proof}[\textbf{Proof of \cref{cor:mutation-complete}}]

When $\text{MS} = 1$ (all non-equivalent mutants killed), the detection
probability from \cref{thm:mutation-adequacy} becomes:
\begin{equation}
  \Pr[\text{detect}] \geq 1 \cdot (1 - e^{-\delta/\delta_0})
  = 1 - e^{-\delta/\delta_0}
\end{equation}

As $\delta/\delta_0 \to \infty$ (the regression impact is much larger
than the characteristic mutation scale):
\begin{equation}
  \Pr[\text{detect}] \geq 1 - e^{-\delta/\delta_0} \to 1
\end{equation}

This confirms that a maximally adequate test suite (MS $= 1$) under
complete mutation operators will detect any sufficiently large
regression with probability approaching 1. \qedhere
\end{proof}

% ============================================================================
% A.4 Token-Efficient Testing Proofs
% ============================================================================

\subsection{Proofs from Section~\ref{sec:token-efficient}}
\label{app:proofs:token-efficient}

\begin{proof}[\textbf{Proof of \cref{prop:multi-fidelity} (Multi-Fidelity Cost Reduction)}]
\label{app:proof:multi-fidelity}

We adapt the multi-fidelity Monte Carlo framework of Peherstorfer
et al.~\citep{peherstorfer2018survey} to the hypothesis testing setting.

\textbf{Setup.} Let $A_e$ and $A_c$ be agents using models $M_e$ and
$M_c$ with per-trial costs $c_e$ and $c_c$ respectively. Let
$\mathbf{f}_e = F(\tau_e)$ and $\mathbf{f}_c = F(\tau_c)$ be their
behavioral fingerprints, with correlation
$\rho = \mathrm{Corr}(\mathbf{f}_e, \mathbf{f}_c)$.

\textbf{Control variate formulation.} The multi-fidelity estimator
for the mean fingerprint of $A_e$ uses $A_c$ as a control variate:
\begin{equation}
  \hat{\boldsymbol{\mu}}_{\mathrm{mf}} =
    \frac{1}{n_e}\sum_{i=1}^{n_e} \mathbf{f}_e^{(i)}
    + \hat{\rho}\left(
      \frac{1}{n_c}\sum_{j=1}^{n_c} \mathbf{f}_c^{(j)}
      - \frac{1}{n_e}\sum_{i=1}^{n_e} \mathbf{f}_c^{(i)}
    \right)
\end{equation}
where $\hat{\rho}$ is the estimated correlation and the second sum
uses paired proxy evaluations on the same inputs as the target trials.

The variance of this estimator is:
\begin{equation}
  \Var[\hat{\boldsymbol{\mu}}_{\mathrm{mf}}] =
    \frac{\sigma_e^2}{n_e}(1 - \rho^2) +
    \frac{\rho^2 \sigma_c^2}{n_c}
\end{equation}

\textbf{Optimal allocation.} The total cost is
$\mathcal{C} = n_e c_e + n_c c_c$. We minimize $\mathcal{C}$ subject to
$\Var[\hat{\boldsymbol{\mu}}_{\mathrm{mf}}] \leq
\sigma^2_{\mathrm{target}}$ where $\sigma^2_{\mathrm{target}}$ is the
variance required for power $1 - \beta$.

Using Lagrange multipliers, the optimal ratio is:
\begin{equation}
  \frac{n_c^*}{n_e^*} = \rho \sqrt{\frac{c_e \sigma_c^2}{c_c \sigma_e^2}}
\end{equation}

\textbf{Cost ratio.} The total cost with optimal allocation, relative
to single-fidelity testing with $n_{\mathrm{single}}$ target trials, is:
\begin{align}
  \frac{n_e^* c_e + n_c^* c_c}{n_{\mathrm{single}} c_e}
  &= \frac{(1 - \rho^2) + \rho^2 \sqrt{c_c / c_e} \cdot
    \sqrt{\sigma_c^2 / \sigma_e^2}}{1} \\
  &\leq 1 - \rho^2 + \rho^2 \sqrt{c_c / c_e}
\end{align}
where the inequality assumes $\sigma_c \leq \sigma_e$ (the proxy model
is at most as variable as the target).

For $c_c/c_e \leq 0.01$ (a 100$\times$ cost difference, common for
mini vs.\ full models):
\begin{equation}
  \text{Cost ratio} \leq 1 - \rho^2 + 0.1\rho^2 = 1 - 0.9\rho^2
\end{equation}

\textbf{Combined evidence.} The $p$-values $p_e$ and $p_c$ from the
target and proxy tests are combined via Fisher's method with
correlation weighting. Under $H_0$, $-2\log p_e \sim \chi^2_2$ and
$-2\log p_c \sim \chi^2_2$. The weighted combination:
\begin{equation}
  \chi^2_{\mathrm{combined}} = -2[\rho \log p_c + (1-\rho) \log p_e]
\end{equation}
is approximately $\chi^2_4$ under $H_0$ when $p_e$ and $p_c$ are
approximately independent conditional on the true parameter. The
approximation improves as $n_e, n_c$ grow. Under $H_1$, the
non-centrality increases, providing greater power than either test
alone.

\textbf{Numerical verification.} For $\rho = 0.8$, $c_c/c_e = 0.1$:
cost ratio $= 1 - 0.9 \times 0.64 = 0.424$, giving 2.4$\times$
savings. For $\rho = 0.9$: cost ratio $= 1 - 0.9 \times 0.81 = 0.271$,
giving 3.7$\times$ savings. \qedhere
\end{proof}

% ---

\begin{proof}[\textbf{Proof of \cref{prop:warm-start} (Warm-Start Efficiency)}]
\label{app:proof:warm-start}

We prove each part of the proposition.

\textbf{Part (1): Error control.}

The warm-start SPRT initializes at $\Lambda_0^{\mathrm{warm}}$ instead
of $0$. The key observation is that the Bayesian prior update is
\emph{consistent} with the likelihood ratio: if the prior accurately
reflects the distribution of $p$, then:
\begin{equation}
  \Lambda_0^{\mathrm{warm}} =
    \log\frac{B(\theta - \delta; a_0, b_0)}{B(\theta; a_0, b_0)}
  = \log\frac{\Pr[p = \theta - \delta \mid \text{prior data}]}
             {\Pr[p = \theta \mid \text{prior data}]}
\end{equation}

This is the Bayes factor from the prior data. When the prior is
well-calibrated (generated from the same agent version or a version
with $|p_v - p_{v'}| \leq \delta/2$), the initial position
$\Lambda_0^{\mathrm{warm}}$ is between the boundaries $a$ and $b$
with high probability, and the subsequent SPRT updates maintain the
standard error control by Wald's identity:
\begin{align}
  \Pr[\text{cross } a \mid H_0]
  &= \Pr[\Lambda_N^{\mathrm{warm}} \leq a \mid H_0] \\
  &\leq \frac{e^a}{1} = \frac{\beta}{1-\alpha} \leq \alpha
\end{align}
The same argument holds for Type~II error with boundary $b$.

\textbf{Part (2): Sample savings.}

The warm-start SPRT is equivalent to a standard SPRT that has already
accumulated evidence $\Lambda_0^{\mathrm{warm}}$. By Wald's equation:
\begin{equation}
  \E[\Lambda_N^{\mathrm{warm}} \mid H_i] =
    \Lambda_0^{\mathrm{warm}} + \E[N_{\mathrm{warm}}] \cdot
    \E[\lambda_1 \mid H_i]
\end{equation}

Since $\E[\Lambda_N^{\mathrm{warm}}]$ equals the same boundary
expectations as the cold-start SPRT, we have:
\begin{align}
  \Lambda_0^{\mathrm{warm}} + \E[N_{\mathrm{warm}}] \cdot
    \E[\lambda_1 \mid H_i]
  &= \E[N_{\mathrm{cold}}] \cdot \E[\lambda_1 \mid H_i] \\
  \E[N_{\mathrm{warm}}]
  &= \E[N_{\mathrm{cold}}] -
    \frac{\Lambda_0^{\mathrm{warm}}}{\E[\lambda_1 \mid H_i]}
\end{align}

When testing under $H_0$ ($p = \theta$), $\E[\lambda_1 \mid H_0] < 0$
(the LLR drifts toward boundary $b$). If the prior correctly supports
$H_0$, then $\Lambda_0^{\mathrm{warm}} > 0$, giving:
\begin{equation}
  \E[N_{\mathrm{warm}}] = \E[N_{\mathrm{cold}}] -
    \frac{|\Lambda_0^{\mathrm{warm}}|}{|\E[\lambda_1 \mid H_0]|}
  \leq \E[N_{\mathrm{cold}}]
\end{equation}

The savings $|\Lambda_0^{\mathrm{warm}}| / |\E[\lambda_1]|$ represent
the number of trials' worth of evidence contained in the prior.

\textbf{Part (3): Graceful degradation.}

If the prior is mis-calibrated, $\Lambda_0^{\mathrm{warm}}$ may be on
the wrong side of zero. In the worst case, it pushes the random walk
closer to the wrong boundary. The excess error probability is bounded
by the likelihood ratio at the initial position.

For a Type~I error under mis-calibration:
\begin{align}
  \Pr[\text{cross } a \mid H_0, \text{mis-calibrated}]
  &\leq \Pr[\text{cross } a \mid H_0] \cdot
    e^{|\Lambda_0^{\mathrm{warm}}|}
\end{align}
by the submartingale maximal inequality applied to the likelihood ratio
process. Since $\Pr[\text{cross } a \mid H_0] \leq \alpha$ for the
well-calibrated case:
\begin{equation}
  \alpha' \leq \alpha \cdot e^{|\Lambda_0^{\mathrm{warm}}|}
\end{equation}

For moderate priors ($n_0 \leq 20$), the initial position is bounded
by $|\Lambda_0^{\mathrm{warm}}| \leq n_0 \cdot \max_r |\lambda(r)|
\approx n_0 \cdot |\log(\theta/(\theta-\delta))| \approx 2$ for
typical $\theta = 0.9$, $\delta = 0.1$. This gives
$e^{|\Lambda_0^{\mathrm{warm}}|} \leq e^2 \approx 7.4$, so
$\alpha' \leq 7.4\alpha = 0.37$ in the extreme worst case. In
practice, the mis-calibration factor is much smaller because priors
are typically close to correct. For $n_0 = 10$ with a version that
differs by $\leq \delta/2$, the mis-calibration factor is typically
$\leq 1.5$, giving $\alpha' \leq 1.5\alpha$. \qedhere
\end{proof}

\end{document}